\documentclass{article}

\usepackage{microtype}
\usepackage{graphicx}
\usepackage{subfigure}
\usepackage{booktabs} 

\usepackage{hyperref}

\usepackage[accepted]{icml2019}

\icmltitlerunning{Provable Guarantees for Gradient-Based Meta-Learning}

\usepackage{amsmath,amssymb,amsthm}
\usepackage[ruled,vlined]{algorithm2e}
\usepackage[utf8]{inputenc}
\newtheorem{Def}{Definition}[section]
\newtheorem{Thm}{Theorem}[section]
\newtheorem{Lem}{Lemma}[section]
\newtheorem{Cor}{Corollary}[section]
\newtheorem{Clm}{Claim}[section]
\newtheorem{Prp}{Proposition}[section]

\newtheorem{Rem}{Remark}[section]
\newtheorem{Set}{Setting}[section]
\DeclareMathOperator*{\argmin}{arg\,min}

\DeclareMathOperator*{\E}{\mathbb{E}}
\newcommand{\R}{\operatorname{\bf R}}
\newcommand{\TAR}{\operatorname{\bf\bar R}}
\newcommand{\TASK}{\operatorname{TASK}}
\newcommand{\OGD}{\operatorname{OGD}}
\newcommand{\FTRL}{\operatorname{FTRL}}
\newcommand{\OMD}{\operatorname{OMD}}

\newcommand{\FAL}{\operatorname{FAL}}
\newcommand{\FLI}{\operatorname{FLI}}

\newcommand{\META}{\operatorname{META}}
\newcommand{\FTL}{\operatorname{FTL}}
\newcommand{\Conv}{\operatorname{Conv}}
\newcommand{\Proj}{\operatorname{Proj}}

\newcommand{\Breg}{\mathcal B}
\usepackage{scalerel,mathtools}
\let\svsqrt\sqrt
\newsavebox\Nsqrt
\def\sr#1{\ThisStyle{%
	\savebox\Nsqrt{\scalebox{.5}[1]{$\SavedStyle\svsqrt{\phantom{\cramped{#1#1}}}$}}%
	\ooalign{\usebox{\Nsqrt}\cr\kern.2pt\usebox{\Nsqrt}\cr\hfil$\SavedStyle\cramped{#1}$}}}
\def\pl{\texttt{+}}
\newcommand{\sd}{\scalebox{0.64}[1]{$\dots$}}

\newcommand{\Ephemeral}{}
\ifdefined\Ephemeral
	\newcommand{\Eph}{Ephemeral\xspace}
\else
	\newcommand{\Eph}{FMRL\xspace}
\fi

\usepackage{enumitem}

\begin{document}

\twocolumn[
\icmltitle{Provable Guarantees for Gradient-Based Meta-Learning}




\begin{icmlauthorlist}
\icmlauthor{Mikhail Khodak}{cmu}
\icmlauthor{\qquad\qquad\quad Maria-Florina Balcan}{cmu}
\icmlauthor{\qquad\qquad\quad Ameet Talwalkar}{cmu,dai}
\end{icmlauthorlist}

\icmlaffiliation{cmu}{Carnegie Mellon University}
\icmlaffiliation{dai}{Determined AI}
\icmlcorrespondingauthor{Mikhail Khodak}{khodak@cmu.edu}

\icmlkeywords{Machine Learning, ICML}

\vskip 0.3in
]



\printAffiliationsAndNotice{}  

\begin{abstract}
	We study the problem of meta-learning through the lens of online convex optimization, developing a meta-algorithm bridging the gap between popular gradient-based meta-learning and classical regularization-based multi-task transfer methods.
	Our method is the first to simultaneously satisfy good sample efficiency guarantees in the convex setting, with generalization bounds that improve with task-similarity, while also being computationally scalable to modern deep learning architectures and the many-task setting.
	Despite its simplicity, the algorithm matches, up to a constant factor, a lower bound on the performance of any such parameter-transfer method under natural task similarity assumptions
	.
	We use experiments in both convex and deep learning settings to verify and demonstrate the applicability of our theory.
\end{abstract}


\section{Introduction}\label{sec:intro}

The goal of {\em meta-learning} can be broadly defined as using the data of 
existing tasks to learn algorithms or representations that enable better or faster performance on unseen tasks.
As the modern iteration of learning-to-learn (LTL) \cite{thrun:98}, research on meta-learning has been largely focused on developing new tools that can exploit the power of the latest neural architectures.
Examples include the control of stochastic gradient descent (SGD) itself using a recurrent neural network \cite{ravi:17} and learning deep embeddings that allow simple classification methods to work well \cite{snell:17}.
A particularly simple but successful approach has been {\em parameter-transfer} via {\em gradient-based meta-learning}, which learns a {\em meta-initialization} $\phi$ for a class of parametrized functions $f_\theta:\mathcal X\mapsto\mathcal Y$ such that one or a few stochastic gradient steps on a few samples from a new task suffice to learn good task-specific model parameters $\hat\theta$ .
For example, when presented with examples $(x_i,y_i)\in\mathcal X\times\mathcal Y$ for an unseen task, the popular MAML algorithm \cite{finn:17} outputs
\begin{equation}\label{eq:maml}
\hat\theta=\phi-\eta\sum_i\nabla L(f_\phi(x_i),y_i)
\end{equation}
for loss function $L:\mathcal Y\times\mathcal Y\mapsto\mathbb R_+$ and learning rate $\eta>~0$;
$\hat\theta$ is then used for inference on the task.
Despite its simplicity, gradient-based meta-learning is a leading approach for LTL in numerous domains including vision \cite{li:17,nichol:18,kim:18}, robotics \cite{al-shedivat:18}, and federated learning \cite{chen:18}.


While meta-initialization is a more recent approach, methods for parameter-transfer have long been studied in the multi-task, transfer, and lifelong learning communities \cite{evgeniou:04,kuzborskij:13,pentina:14}.
A common classical alternative to~\eqref{eq:maml}, which in modern parlance may be called {\em meta-regularization}, is to learn a good bias $\phi$ for the following regularized empirical risk minimization (ERM) problem:
\begin{equation}\label{eq:reg}
\hat\theta=\argmin_\theta\frac{\|\theta-\phi\|_2^2}{2\eta}+\sum_iL(f_\theta(x_i),y_i)
\end{equation}
Although there exist statistical guarantees and poly-time algorithms for learning a  meta-regularization for simple models \cite{pentina:14,denevi:18b}, such methods are impractical and do not scale to modern settings with deep neural architectures and many tasks.
On the other hand, while the theoretically less-studied meta-initialization approach is often compared to meta-regularization \cite{finn:17}, their connection is not rigorously understood.

In this work, we formalize this connection using the theory of online convex optimization (OCO) \cite{zinkevich:03}, in which an intimate connection between initialization and regularization is well-understood due to the equivalence of online gradient descent (OGD) and follow-the-regularized-leader (FTRL) \cite{shalev-shwartz:11,hazan:15}.
In the lifelong setting of an agent solving a sequence of OCO tasks, we use this connection to analyze an algorithm that learns a $\phi$, which can be a meta-initialization for OGD or a meta-regularization for FTRL, such that the within-task regret of these algorithms improves with the similarity of the online tasks;
here the similarity is measured by the distance between the optimal actions $\theta^\ast$ of each task and is not known beforehand.
This algorithm, which we call Follow-the-Meta-Regularized-Leader (
\ifdefined\Ephemeral
	FMRL or {\bf\Eph}
\else
	{\bf\Eph}
\fi
), scales well in both computation and memory requirements, and in fact generalizes the gradient-based meta-learning algorithm Reptile \cite{nichol:18},
thus providing a convex-case theoretical justification for a leading method in practice.

More specifically, we make the following contributions:
\begin{itemize}[leftmargin=*]
	\vspace{-1.5mm}
	\item Our first result assumes a sequence of OCO tasks $t$ whose optimal actions $\theta_t^\ast$ are inside a small subset $\Theta^\ast$ of the action space.
	We show how \Eph can use these $\theta_t^\ast$ to make the average regret decrease in the diameter of $\Theta^\ast$ and do no worse on dissimilar tasks.
	Furthermore, we extend a lower bound of \citet{abernethy:08} to the multi-task setting to show that one can do no more than a small constant-factor better sans stronger assumptions.\vspace{-1.5mm}
	\item Under a realistic assumption on the loss functions, we show that \Eph also has low-regret guarantees in the practical setting where the optimal actions $\theta_t^\ast$ are difficult or impossible to compute and the algorithm only has access to a statistical or numerical approximation.
	In particular, we show high probability regret bounds in the case when the approximation uses the gradients observed during within-task training, as is done in practice by Reptile \cite{nichol:18}.\vspace{-1.5mm}
	\item We prove an online-to-batch conversion showing that the task parameters learned by a meta-algorithm with low task-averaged regret have low risk, connecting our guarantees to statistical LTL \cite{baxter:00,maurer:05}.\vspace{-1.5mm}
	\item We verify several assumptions and implications of our theory using a new meta-learning dataset we introduce consisting of text-classification tasks solvable using convex methods. We further study the empirical suggestions of our theory in the deep learning setting.
	\vspace{-1.5mm}
\end{itemize}

\subsection{Related Work}\label{subsec:related}
{\bf Gradient-Based Meta-Learning:} The model-agnostic meta-learning (MAML) algorithm of \citet{finn:17} pioneered this recent approach to LTL. A great deal of empirical work has studied and extended this approach \cite{li:17,grant:18,nichol:18,jerfel:18};
in particular, \citet{nichol:18} develop Reptile, a simple yet equally effective first-order simplification of MAML for which our analysis shows provable guarantees as a subcase.
Theoretically, \citet{franceschi:18} provide computational convergence guarantees for gradient-based meta-learning for strongly-convex functions, while \citet{finn:18} show that with infinite data MAML can approximate any function of task samples assuming a specific neural architecture as the model.
In contrast to both results, we show finite-sample learning-theoretic guarantees for convex functions under a natural task-similarity assumption.

{\bf Online LTL:}
Learning-to-learn and multi-task learning (MTL) have both been extensively studied in the online setting, although our setting differs significantly from the one usually studied in online MTL \cite{abernethy:07,dekel:07,cavallanti:10}.
There, in each round an agent is told which of a fixed set of tasks the current loss belongs to, whereas our analysis is in the lifelong setting, in which tasks arrive one at a time.
Here there are many theoretical results for learning useful data representations \cite{ruvolo:13,pentina:14,balcan:15,alquier:17};
the PAC-Bayesian result of \citet{pentina:14} can also be used for regularization-based parameter transfer, which we also consider.
Such methods are provable variants of practical shared-representation approaches, e.g. ProtoNets \cite{snell:17}, but unlike our algorithms they do not scale to deep neural networks.
Our work is especially related to \citet{alquier:17}, who also consider a many-task regret.
We achieve similar bounds with a significantly more practical algorithm, although within-task their results hold for any low-regret method whereas ours only hold for OCO.
Lastly, we note two concurrent works,  by \citet{denevi:19} and \citet{finn:19}, that address LTL via online learning, either directly or through online-to-batch conversion.


{\bf Statistical LTL:} While we focus on the online setting, our online-to-batch results also imply risk bounds for distributional meta-learning.
This setting was formalized by \citet{baxter:00};
\citet{maurer:05} further extended the hypothesis-space-learning framework to algorithm-learning.
Recently, \citet{amit:18} showed PAC-Bayesian generalization bounds for this setting, although without implying an efficient algorithm.
Also closely related are the regularization-based approaches of \citet{denevi:18a,denevi:18b}, which provide statistical learning guarantees for Ridge regression with a meta-learned kernel or bias.
\citet{denevi:18b} in particular focuses on usefulness relative to single-task learning, showing that their method is better than the $\ell_2$-regularized ERM, but neither addresses the connection between loss-regularization and gradient-descent-initialization.



\vspace{-1mm}
\section{Meta-Initialization \& Meta-Regularization}\label{sec:meta}

\begin{algorithm}[!t]
	\DontPrintSemicolon
	Pick a first meta-initialization $\phi_1$.\\
	\For{{\em task} $t\in[T]$}{
		Run a within-task online algorithm (e.g. OGD) on the losses of task $t$ using initialization $\phi_t$.\\
		Compute (exactly or approximately) the best fixed action in hindsight $\theta_t^\ast$ for task $t$.\\
		Update $\phi_t$ using a meta-update online algorithm (e.g. OGD) on the meta-loss $\ell_t(\phi)=\|\theta_t^\ast-\phi_t\|^2$.\\
	}
	\caption{\label{alg:simple}
		The generic online-within-online algorithm we study.
		First-order gradient-based meta-learning uses OGD in both the inner and outer loop.
	}
\end{algorithm}

We study simple methods of the form of Algorithm~\ref{alg:simple}, where we run a {\em within-task} online algorithm on each task and then update the initialization or regularization of this algorithm using a {\em meta-update} online algorithm.
\citet{alquier:17} study such a method where the meta-update is conducted using exponentially-weighted averaging.
Our use of OCO for the meta-update makes this class of algorithms much more practical;
for example, in the case of OGD for both the inner and outer loop we recover the Reptile algorithm of \citet{nichol:18}.
To analyze Algorithm~\ref{alg:simple}, we first discuss the OCO methods that make up both its inner and outer loop and the inherent connection they provide between initialization and regularization.
We then make this connection explicit by formalizing the notion of learning a meta-initialization or meta-regularization as learning a parameterized Bregman regularizer.
We conclude this section by proving convex-case upper and lower bounds on the task-averaged regret.

\subsection{Online Convex Optimization}\label{subsec:oco}

In the online learning setting, at each time $t=1,\dots,T$ an agent chooses action $\theta_t\in\Theta\subset\mathbb R^d$ and suffers loss $\ell_t(\theta_t)$ for some adversarially chosen function $\ell_t:\Theta\mapsto\mathbb R$ that subsumes the loss, model, and data in $L(f_\theta(x),y)$ into one function of $\theta$.
The goal is to minimize {\em regret} -- the difference between the total loss and that of the optimal fixed action:
\vspace{-2mm}
$$
\R_T=\sum_{t=1}^T\ell_t(\theta_t)-\min_{\theta\in\Theta}\sum_{t=1}^T\ell_t(\theta)
\vspace{-1mm}
$$
When $\R_T=o(T)$ then as $T\to\infty$ the average loss of the agent will approach that of an optimal fixed action.

For OCO, $\ell_t$ is assumed convex and Lipschitz for all $t$.
This setting provides many practically useful algorithms such as online gradient descent (OGD).
Parameterized by a starting point $\phi\in\Theta$ and learning rate $\eta>0$, OGD plays
\begin{equation}\label{eq:ogd}
\vspace{-1.5mm}
\theta_t=\Proj_\Theta\left(\phi-\eta\sum_{s<t}\nabla\ell_s(\theta_s)\right)
\vspace{-1.5mm}
\end{equation}

and achieves sublinear regret $\mathcal O(D\sqrt T)$ when $\eta\propto\frac{D}{\sqrt T}$, where $D$ is the diameter of the action space $\Theta$.

Note the similarity between OGD and the meta-initialization update in Equation~\ref{eq:maml}.
In fact another fundamental OCO algorithm, follow-the-regularized-leader (FTRL), is a direct analog for the meta-regularization algorithm in Equation~\ref{eq:reg}, with its action at each time being the output of $\ell_2$-regularized ERM over the previous data:
\begin{equation}\label{eq:ftrl}
\vspace{-1.5mm}
\theta_t=\argmin_{\theta\in\Theta}\frac1{2\eta}\|\theta-\phi\|_2^2+\sum_{s<t}\ell_s(\theta)
\vspace{-1.5mm}
\end{equation}
Note that most definitions set $\phi=0$.
A crucial connection here is that on linear functions $\ell_t(\cdot)=\langle\nabla_t,\cdot\rangle$, OGD initialized at $\phi=0$ plays the same actions $\theta_t\in\Theta~\forall~t\in[T]$ as FTRL.
Since linear losses are the hardest losses, in that low regret for them implies low regret for convex functions \cite{zinkevich:03}, in the online setting this equivalence suggests that meta-initialization is a reasonable surrogate for meta-regularization because it is solving the hardest version of the problem.
The OGD-FTRL equivalence can be extended to other geometries by replacing the squared-norm in~\eqref{eq:ftrl} by a strongly-convex function $R:\Theta\mapsto\mathbb R_+$:
\vspace{-2mm}
$$
\theta_t=\argmin_{\theta\in\Theta}\frac1\eta R(\theta)+\sum_{s<t}\ell_s(\theta)
\vspace{-1mm}
$$
In the case of linear losses this is the online mirror descent (OMD) generalization of OGD.
For $G$-Lipschitz losses, OMD and FTRL have the following well-known regret guarantee $\forall~\theta^\ast\in\Theta$ \citep[Theorem~2.11]{shalev-shwartz:11}:
\vspace{-2mm}
\begin{equation}\label{eq:regret}
\R_T\le\frac1\eta R(\theta^\ast)+\eta G^2T
\vspace{-1mm}
\end{equation}

\subsection{Task-Averaged Regret and Task Similarity}\label{subsec:tar}

We consider the {\em lifelong} extension of online learning, where $t=1,\dots,T$ now index a sequence of online learning problems, in each of which the agent must sequentially choose $m_t$ actions $\theta_{t,i}\in~\Theta$ and suffer loss $\ell_{t,i}:\Theta\mapsto\mathbb R$.
Since in meta-learning we are interested in doing well on individual tasks, we will aim to minimize a dynamic notion of regret in which the comparator changes with each task, so that the comparator corresponds to the best within-task parameter:
\begin{Def}\label{def:tar}
	The {\bf task-averaged regret} (TAR) of an online algorithm after $T$ tasks with $\{m_t\}_{t=1}^T$ steps is\vspace{-1mm}
	$$\TAR=\frac1T\sum_{t=1}^T\left(\sum_{i=1}^{m_t}\ell_{t,i}(\theta_{t,i})-\min_{\theta_t\in\Theta}\sum_{i=1}^{m_t}\ell_{t,i}(\theta_t)\right)$$
\end{Def}
\vspace{-2mm}
Note that, unlike in standard regret one cannot achieve TAR decreasing in $T$, the number of tasks, because the comparator is dynamic and so can force a constant loss at each task $t$.
Furthermore, the average is taken over $T$ and {\em not} the number of rounds per task $m_t$, so in our results we expect TAR to grow sub-linearly in $m_t$.
This corresponds to achieving sub-linear single-task regret on-average.

An alternative comparator that is seemingly natural in the study of gradient-based meta-learning is the best fixed initialization in hindsight;
however, this quantity overlooks the fact that meta-initialization is simply a tool to achieve what we actually care about, which is within-task performance.
If the difference between the task loss when starting from the best meta-initialization and that of the optimal within-task parameter is high, comparing to the best meta-initialization may not be very meaningful.
On the other hand, a low TAR ensures that the task loss of an algorithm compared to that of the optimal within-task parameter is low on average.


\begin{figure}
	\hspace{0.15\linewidth}
	\includegraphics[width=0.25\linewidth]{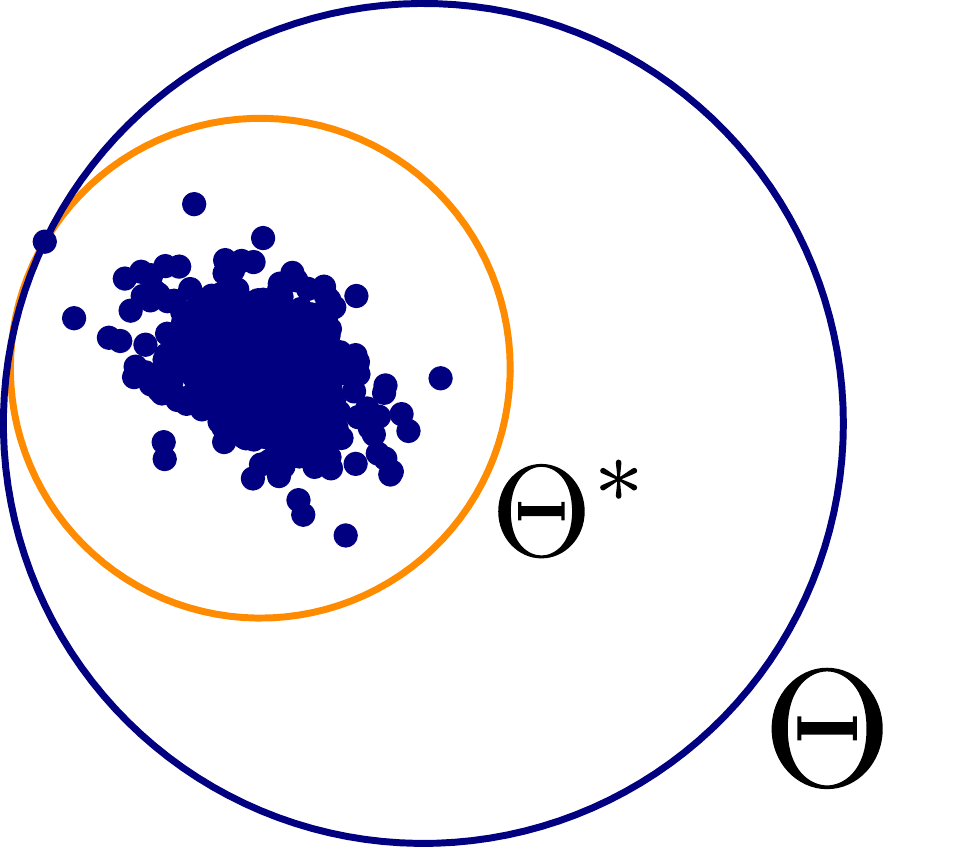}
	\hfill
	\includegraphics[width=0.25\linewidth]{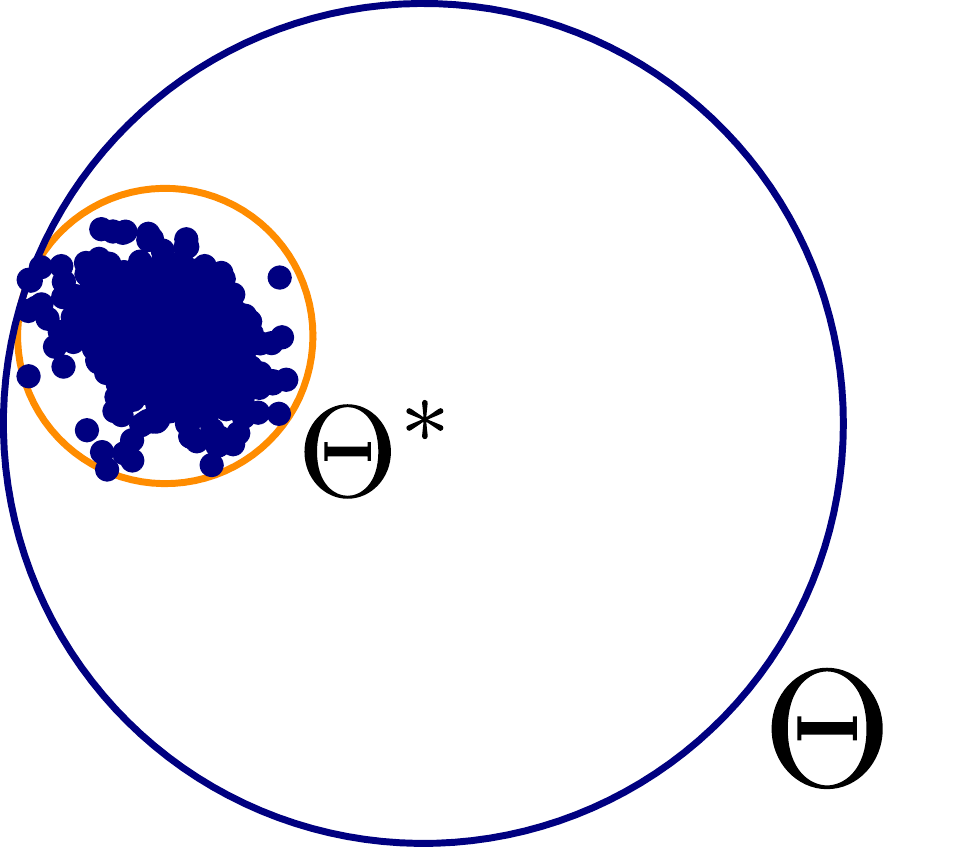}
	\hspace{0.15\linewidth}
	\vspace{-1mm}
	\caption{\label{fig:similarity}
		Random projection of ERM parameters of 1-shot (left) and 32-shot~(right) Mini-Wiki tasks, described in Section~\ref{sec:empirical}.
	}
	\vspace{-2mm}
\end{figure}

We now formalize our similarity assumption on the tasks $t\in[T]$: their optimal actions $\theta_t^\ast$ lie within a small subset $\Theta^\ast$ of the action space.
This is natural for studying gradient-based meta-learning, as the notion that there exists a meta-parameter $\phi$ from which a good parameter for any individual task is reachable with only a few steps implies that they are all close together.
We develop algorithms whose TAR scales with the diameter $D^\ast$ of $\Theta^\ast$;
notably, this means they will not do much worse if $\Theta^\ast=\Theta$, i.e. if the tasks are not related in this way, but will do well if $D^\ast\ll D$.
Importantly, our methods will not require knowledge of $\Theta^\ast$.
\begin{Set}\label{set:exact}
	Each task $t\in[T]$ has $m_t$ convex loss functions $\ell_{t,i}:\Theta\mapsto\mathbb R$ that are $G_t$-Lipschitz on-average.
	Let $\theta_t^\ast\in\argmin_{\theta\in\Theta}\sum_{i=1}^{m_t}\ell_{t,i}(\theta)$ be the minimum-norm optimal fixed action for task $t$.
	Define $\Theta^\ast\subset\Theta$ to be the minimal subset containing $\theta_t^\ast~\forall~t\in[T]$.
	Assume that $\Theta^\ast$ has non-empty interior (and thus $T>1$).
\end{Set}

Note $\theta_t^\ast$ is unique as the minimum of $\|\cdot\|^2$, a strongly convex function, over minima of a convex function.
The algorithms in Section~\ref{subsec:fmrl} assume an efficient oracle computing $\theta_t^\ast$.

\subsection{Parameterizing Bregman Regularizers}\label{subsec:bregman}

Following the main idea of gradient-based meta-learning, our goal is to learn a $\phi\in\Theta$ such that an online algorithm such as OGD starting from $\phi$ will have low regret.
We thus treat regret as our objective and observe that in the regret of FTRL \eqref{eq:regret}, the regularizer $R$ effectively encodes a distance from the initialization to $\theta^\ast$.
This is clear in the Euclidean geometry for $R(\theta)=\frac12\|\theta-\phi\|_2^2$, but can be extended
via the {\em Bregman divergence} \cite{bregman:67}, defined for $f:S\mapsto\mathbb R$ everywhere-sub-differentiable and convex as
$$\Breg_f(x||y)=f(x)-f(y)-\langle\nabla f(y),x-y\rangle$$
The Bregman divergence has many useful properties \cite{banerjee:05} that allow us to use it almost directly as a parameterized regularization function.
However, in order to use OCO for the meta-update we also require it to be strictly convex in the second argument, a property that holds for the Bregman divergence of both the $\ell_2$ regularizer and the entropic regularizer $R(\theta)=\langle\theta,\log\theta\rangle$ used for online learning over the probability simplex, e.g. with expert advice.

\begin{Def}\label{def:regularizer}
	Let $R:S\mapsto\mathbb{R}$ be 1-strongly-convex w.r.t. norm $\|\cdot\|$ on convex $S\subset\mathbb{R}^d$.
	Then we call the Bregman divergence $\Breg_R(x||y):S\times S\mapsto\mathbb{R}_+$ a {\bf Bregman regularizer} if $\Breg_R(x||\cdot)$ is strictly convex for any fixed $x\in S$.
\end{Def}

Within each task, the regularizer is parameterized by the second argument and acts on the first.
More specifically, for $R=\frac12\|\cdot\|_2^2$ we have $\Breg_R(\theta||\phi)=\frac12\|\theta-\phi\|_2^2$, and so in the case of FTRL and OGD, $\phi$ is a parameterization of the regularization and the initialization, respectively.
In the case of the entropic regularizer, the associated Bregman regularizer is the KL-divergence from $\phi$ to $\theta$ and thus meta-learning $\phi$ can very explicitly be seen as learning a prior.

Finally, we use Bregman regularizers to formally define our parameterized learning algorithms: 
\begin{Def}\label{def:ftrl}
	$\FTRL_{\eta,\phi}$, for $\eta\in\mathbb{R}_+,\phi\in\Theta$, where $\Theta$ is some bounded convex subset $\Theta\subset\mathbb{R}^d$, plays
	\vspace{-1mm}
	$$\theta_t=\argmin_{\theta\in\Theta}\Breg_R(\theta||\phi)+\eta\sum_{s<t}\ell_s(\theta)
	\vspace{-1mm}
	$$
	for Bregman regularizer $\Breg_R$.
	Similarly, $\OMD_{\eta,\phi}$ plays
	\vspace{-1mm}
	$$\theta_t=\argmin_{\theta\in\Theta}\Breg_R(\theta||\phi)+\eta\sum_{s<t}\langle\nabla_s,\theta\rangle
	\vspace{-2mm}
	$$
\end{Def}
Here FTRL and OMD correspond to the meta-regularization \eqref{eq:reg} and meta-initialization \eqref{eq:maml} approaches, respectively.
As $\Breg_R(\cdot||\phi)$ is strongly-convex, both algorithms have the same regret bound \eqref{eq:regret}, allowing us to analyze them jointly.


\subsection{Follow-the-Meta-Regularized-Leader}\label{subsec:fmrl}

We now specify the first variant of our main algorithm, Follow-the-Meta-Regularized-Leader (\Eph).
First assume the diameter $D^\ast$ of $\Theta^\ast$, as measured by the square root of the maximum Bregman divergence between any two points, is known.
Starting with $\phi_1\in\Theta$, run $\FTRL_{\eta,\phi_t}$ or $\OMD_{\eta,\phi_t}$ with $\eta\propto\frac{D^\ast}{\sqrt m}$ on the losses in each task $t$.
After each task, compute $\phi_{t+1}$ using an OCO meta-update algorithm operating on the Bregman divergences $\Breg_R(\theta_t^\ast||\cdot)$.
For $D^\ast$ unknown, make an underestimate $\varepsilon>0$ and multiply it by a factor $\gamma>1$ each time $\Breg_R(\theta_t^\ast||\phi_t)>\varepsilon^2$.

The following is a regret bound for this algorithm when the meta-update is either {\em Follow-the-Leader} (FTL), which plays the minimizer of all past losses, or OGD with adaptive step size.
We call this \Eph variant {\em Follow-the-Average-Leader} (FAL) because in the case of FTL the algorithm uses the mean of the previous optimal parameters in hindsight as the initialization.
Pseudo-code for this and other variants is given in Algorithm~\ref{alg:fmrl}.
For brevity, we state results for constant $G_t=G,m_t=m~\forall~t$;
detailed statements are in the supplement together with the full proof.

\begin{Thm}\label{thm:fml}
	In Setting~\ref{set:exact}, the FAL variant of Algorithm~\ref{alg:fmrl} with task similarity guess $\varepsilon=D\frac{1+\log T}T$, tuning parameter $\gamma=\frac{1+\log T}{\log T}$, and $\Breg_R$ that is Lipschitz on $\Theta^\ast$ achieves TAR
	$$\TAR\le\mathcal O\left(D^\ast+\frac{D\log T}{D^\ast T}\right)\sqrt m$$
	for diameter $D^\ast=\max_{\theta,\phi\in\Theta^\ast}\sqrt{\Breg_R(\theta||\phi)}$ of $\Theta^\ast$.
\end{Thm}
\begin{proof}[Proof Sketch]
	We give a proof for $R(\cdot)=\frac12\|\cdot\|_2^2$ and known task similarity, i.e. $\varepsilon=D^\ast,\gamma=1$.
	Denote the divergence to $\theta_t^\ast$ by $\Delta_t(\phi)=\Breg_R(\theta_t^\ast||\phi)=\frac12\|\theta_t^\ast-\phi\|_2^2$ and let $\phi^\ast=\frac1T\sum_{t=1}^T\theta_t^\ast$.
	Note $\Delta_t$ is 1-strongly-convex and $\phi^\ast$ is the minimizer of their sum, with the variance $\bar D^2=\frac1T\sum_{t=1}^T\Delta_t(\phi^\ast)\le{D^\ast}^2$.
	Now by Definition~\ref{def:tar}:
	\begin{align*}
		\TAR
		&=\frac1T\sum_{t=1}^T\left(\sum_{i=1}^m\ell_{t,i}(\theta_{t,i})-\min_{\theta_t\in\Theta}\sum_{i=1}^m\ell_{t,i}(\theta_t)\right)\\
		&\le\frac1T\sum_{t=1}^T\frac{\Delta_t(\phi_t)}\eta+\eta G^2m\\
		&=\frac1T\sum_{t=1}^T\frac{\Delta_t(\phi_t)-\Delta_t(\phi^\ast)}\eta+\frac1T\sum_{t=1}^T\frac{\Delta_t(\phi^\ast)}\eta+\eta G^2m
	\end{align*}
	The first two lines apply the regret bound \eqref{eq:regret} of FTRL and OMD.
	The key step is the last one, with the regret is split into the loss of the meta-update algorithm on the left and the loss if we had always initialized at the mean $\phi^\ast$ of the optimal actions $\theta_t^\ast$ on the right.
	Since $\Delta_1,\dots,\Delta_T$ are 1-strongly-convex with minimizer $\phi^\ast$, and since each $\phi_t$ is determined by playing FTL or OGD on these same functions, the left term is the regret of these algorithms on strongly-convex functions, which is known to be $\mathcal{O}(\log T)$ \cite{bartlett:08,kakade:08}.
	Substituting the definition of $\phi^\ast$ and $\eta=\frac{D^\ast}{G\sqrt m}$ sets the right term to
	\vspace{-2mm}
	$$\frac1T\sum_{t=1}^T\frac{\Delta_t(\phi^\ast)}\eta+\eta G^2m=G\bar D\sqrt m+GD^\ast\sqrt m
	\vspace{-6mm}$$
	\vspace{-3mm}
\end{proof}
The full proof uses the doubling trick to tune task similarity $D^\ast$, requiring an analysis of the location of meta-parameter $\phi_t$ to ensure that we only increase the guess when needed.
The extension to non-Euclidean geometries uses a novel logarithmic regret bound for FTL over Bregman regularizers.
\begin{Rem}\label{rem:dev}
	Note that if we know the variance $\bar D^2$ of the task parameters from their mean $\phi^\ast$, setting $\eta_t=\frac{\bar D}{G_t\sqrt{m_t}}$ in Algorithm~\ref{alg:fmrl} and following the analysis above replaces $D^\ast$ in Theorem~\ref{thm:fml} with $\bar D$,
	which is better since $\bar D\le D^\ast$ and is furthermore less sensitive to possible outlier tasks.
\end{Rem}

\begin{algorithm}[!t]
	\DontPrintSemicolon
	\KwData{
		\begin{itemize}
			\vspace{-1.5mm}
			\item initialization $\phi_1$ in action space $\Theta$
			\vspace{-2.5mm}
			\item meta-update algorithm $\META_\phi$ ($\FTL$ or $\OGD$)
			\vspace{-2.5mm}
			\item within-task algorithm $\TASK_{\eta,\phi}$ ($\FTRL$ or $\OMD$) with Bregman regularizer $\Breg_R$ w.r.t. $\|\cdot\|$
			\vspace{-2.5mm}
			\item Lipschitz constant $G_t$ w.r.t. $\|\cdot\|_\ast$ on each task $t$
			\vspace{-2.5mm}
			\item similarity guess $\varepsilon>0$ and tuning parameter $\gamma\ge1$
		\end{itemize}
	}
	\vspace{-2.5mm}
	\tcp{set first-task similarity guess to be the full action space}
	$D_1\gets\max_{\theta\in\Theta}\sqrt{\Breg_R(\theta||\phi_1)}+\varepsilon$\\
	$k\gets0$\\
	\For{$t\in[T]$}{
		\tcp{set learning rate using task similarity guess; run within-task algorithm}
		$\eta_t\gets\frac{D_t}{G_t\sqrt{m_t}}$\\
		\For{$i\in[m_t]$}{
			$\theta_{t,i}\gets\TASK_{\eta_t,\phi_t}(\ell_{t,1},\sd,\ell_{t,i-1})$\\
			suffer loss $\ell_{t,i}(\theta_{t,i})$\\
		}
		\tcp{compute meta-update vector $\theta_t$ depending on \Eph variant}
		\uCase{$\FAL$}{
			$\theta_t\gets\argmin_{\theta\in\Theta}\sum_{i=1}^{m_t}\ell_{t,i}(\theta)$
		}
		\Case{$\FLI$-\text{\em Online}}{
			$\theta_t\gets\TASK_{\eta_t,\phi_t}(\ell_{t,1},\sd,\ell_{t,m_t})$
		}
		\Case{$\FLI$-\text{\em Batch}}{
			$\theta_t\gets\frac1{m_t}\sum_{i=1}^{m_t}\theta_{t,i}$
		}
		\tcp{increase task similarity guess if violated; run meta-update}
		\If{$D_t<\sqrt{\Breg_R(\theta_t||\phi_t)}$}{
			$k\gets k+1$\\
		}
		$D_{t\pl1}\gets\gamma^k\varepsilon$\\
		$\phi_{t\pl1}\gets\META_{\theta_1}(\{\Breg_R(\theta_s||\cdot)G_s\sqrt{m_s}\}_{s=1}^t)$
	}
	\caption{\label{alg:fmrl}
		Follow-the-Meta-Regularized-Leader (\Eph) meta-algorithm for meta-learning.
		For the FAL variant we assume $\argmin_{\theta\in\Theta}L(\theta)$ returns the minimum-norm $\theta$ among all minimizers of $L$ over $\Theta$.
		For $\META=\OGD$ we assume $R(\cdot)=\frac12\|\cdot\|_2^2$ and adaptive step size $\left(\sum_{s<t}\sqrt{m_s}\right)^{-1}$ at each time $t$.
	}
\end{algorithm}

Theorem~\ref{thm:fml} shows that the TAR of \Eph scales with task similarity $D^\ast$, and that if tasks are not similar then we only do a constant factor worse than FTRL or OMD.
This shows that gradient-based meta-learning is useful in convex settings:
under a simple notion of similarity, having more tasks yields better performance than the $\mathcal O(D\sqrt m)$ regret of single-task learning.
The algorithm also scales well and in the $\ell_2$ setting is similar to Reptile \cite{nichol:18}.

However, it is easy to see that an even simpler ``strawman" algorithm achieves regret only a constant factor worse: 
at time $t+1$, simply initialize FTRL or OMD  using the optimal parameter $\theta_t^\ast$ of task $t$.
Of course, in the few-shot setting of small $m$, a reduction in the average regret is still practically significant;
we observe this empirically in Figure~\ref{fig:samples}.
Indeed, in the proof of Theorem~\ref{thm:fml} the regret converges to that obtained by always playing the mean optimal action, which will not occur when playing the strawman algorithm.
Furthermore, the following lower-bound on the task-averaged regret, a multi-task extension of \citet[Theorem~4.2]{abernethy:08}, shows that such constant factor reductions are the best we can achieve under our task similarity assumption:

\begin{Thm}\label{thm:lower}
	Assume $d\ge3$ and that for each $t\in[T]$ an adversary must play a sequence of $m$ convex $G$-Lipschitz functions $\ell_{t,i}:\Theta\mapsto\mathbb R$ whose optimal actions in hindsight $\argmin_{\theta\in\Theta}\sum_{i=1}^m\ell_{t,i}(\theta)$ are contained in some fixed $\ell_2$-ball $\Theta^\ast\subset\Theta$ with center $\phi^\ast$ and diameter $D^\ast$.
	Then the adversary can force the agent to have TAR at least $\frac{GD^\ast}{4}\sqrt m$.
\end{Thm}

More broadly, this lower bound shows that the learning-theoretic benefits of gradient-based meta-learning are inherently limited without stronger assumptions on the tasks.
Nevertheless, \Eph-style algorithms are very attractive from a practical perspective, as their memory and computation requirements per iteration scale linearly in the dimension and not at all in the number of tasks.


\section{Provable Guarantees for Practical Gradient-Based Meta-Learning}\label{sec:provable}

In the previous section we gave an algorithm with access to the best actions in hindsight $\theta_t^\ast$ of each task that can learn a good meta-initialization or meta-regularization.
While $\theta_t^\ast$ is efficiently computable in some cases, often it is more practical to use an approximation.
This holds in the deep learning setting, e.g. \citet{nichol:18} use the average within-task gradient.
Furthermore, in the batch setting a more natural similarity notion depends on the true risk minimizers and not the optimal actions for a few samples.
In this section we first show how two simple variants of \Eph handle these settings, one for the adversarial setting which uses the final action on task $t$ as the meta-update and one for the stochastic setting using the average iterate.
We call these methods FLI-Online and FLI-Batch, respectively, where {\em FLI} stands for {\em Follow-the-Last-Iterate}.
We then provide an online-to-batch conversion result for TAR that implies good generalization guarantees when any of the variants of \Eph are run in the distributional LTL setting.

\subsection{Simple-to-Compute Meta-Updates}\label{subsec:fli}

To achieve guarantees using approximate meta-updates we need to make some assumptions on the within-task loss functions.
This is unavoidable because we need estimates of the optimal actions of different tasks to be nearby;
in general, for some $\theta\in\Theta$ a convex function $f:\Theta\mapsto\mathbb R$ can have small $f(\theta)-f(\theta^\ast)$ but large $\|\theta-\theta^\ast\|$ if $f$ does not increase quickly away from the minimum.
This makes it impossible to use guarantees on the loss of an estimate of $\theta_t^\ast$ to bound its distance from $\theta_t^\ast$.
We therefore make assumptions that some aggregate loss, e.g. the expectation or sum of the within-task losses, satisfies the following growth condition:
\begin{Def}\label{def:growth}
	A function $f:\Theta\mapsto\mathbb R$ has {\bf $\alpha$-quadratic-growth} ($\alpha$-QG) w.r.t. $\|\cdot\|$ for $\alpha>0$ if for any $\theta\in\Theta$ and $\theta^\ast$ its closest minimum of $f$ we have
	\vspace{-1mm}
	$$\frac\alpha2\|\theta-\theta^\ast\|^2\le f(\theta)-f(\theta^\ast)
	\vspace{-2mm}
	$$
\end{Def}
QG has recently been used to provide fast rates for GD that hold for practical problems such as LASSO and logistic regression under data-dependent assumptions \cite{karimi:16,garber:19}.
It can be shown when $f(\theta)=g(A\theta)$ for $g$ strongly-convex and some $A\in\mathbb{R}^{m\times d}$;
in this case $\alpha\ge\sigma_{\min}(A)$ \cite{karimi:16}.
Note that $\alpha$-QG is also a weaker condition than $\alpha$-strong-convexity.

\begin{figure}
	\includegraphics[width=0.48\linewidth]{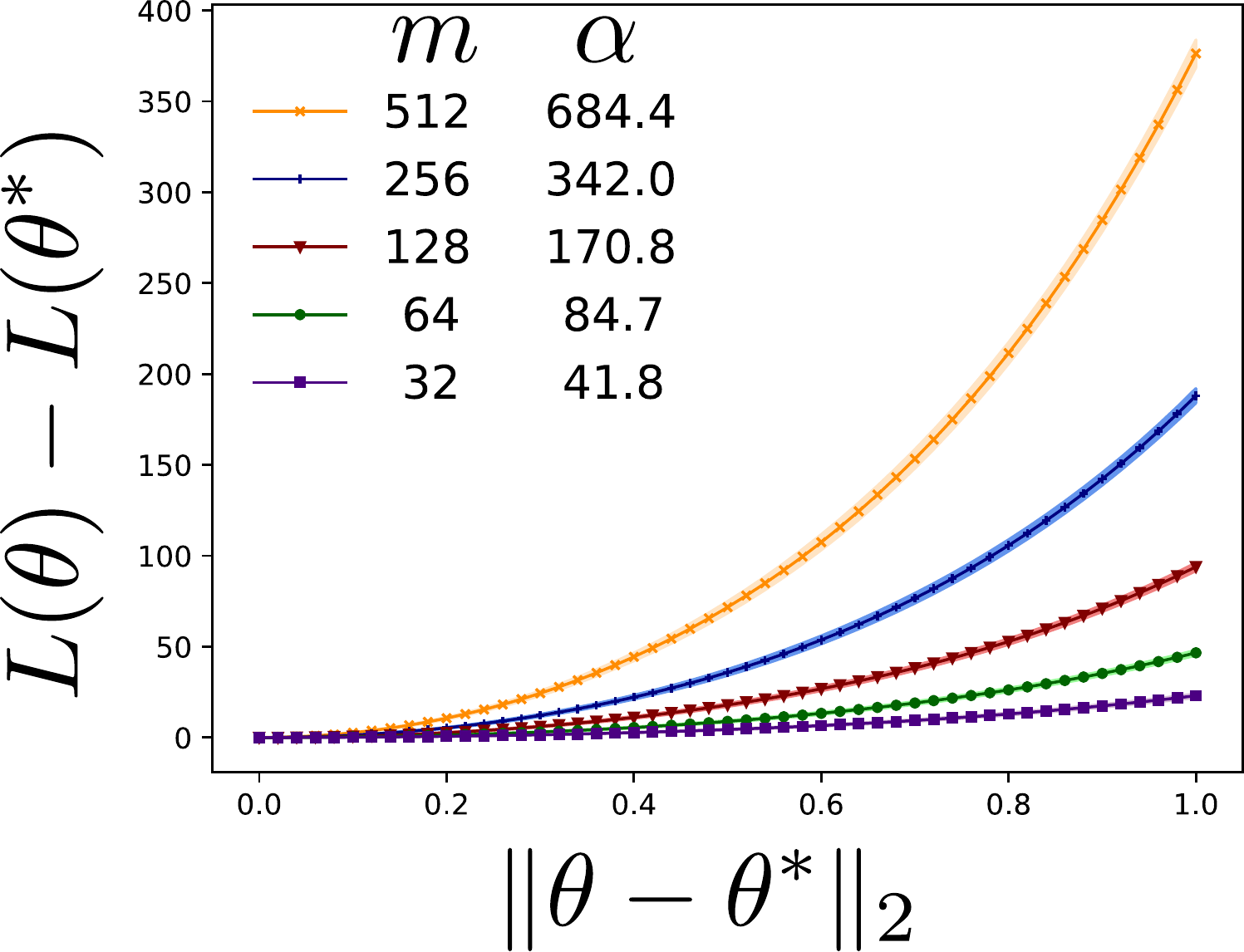}
	\hfill
	\includegraphics[width=0.48\linewidth]{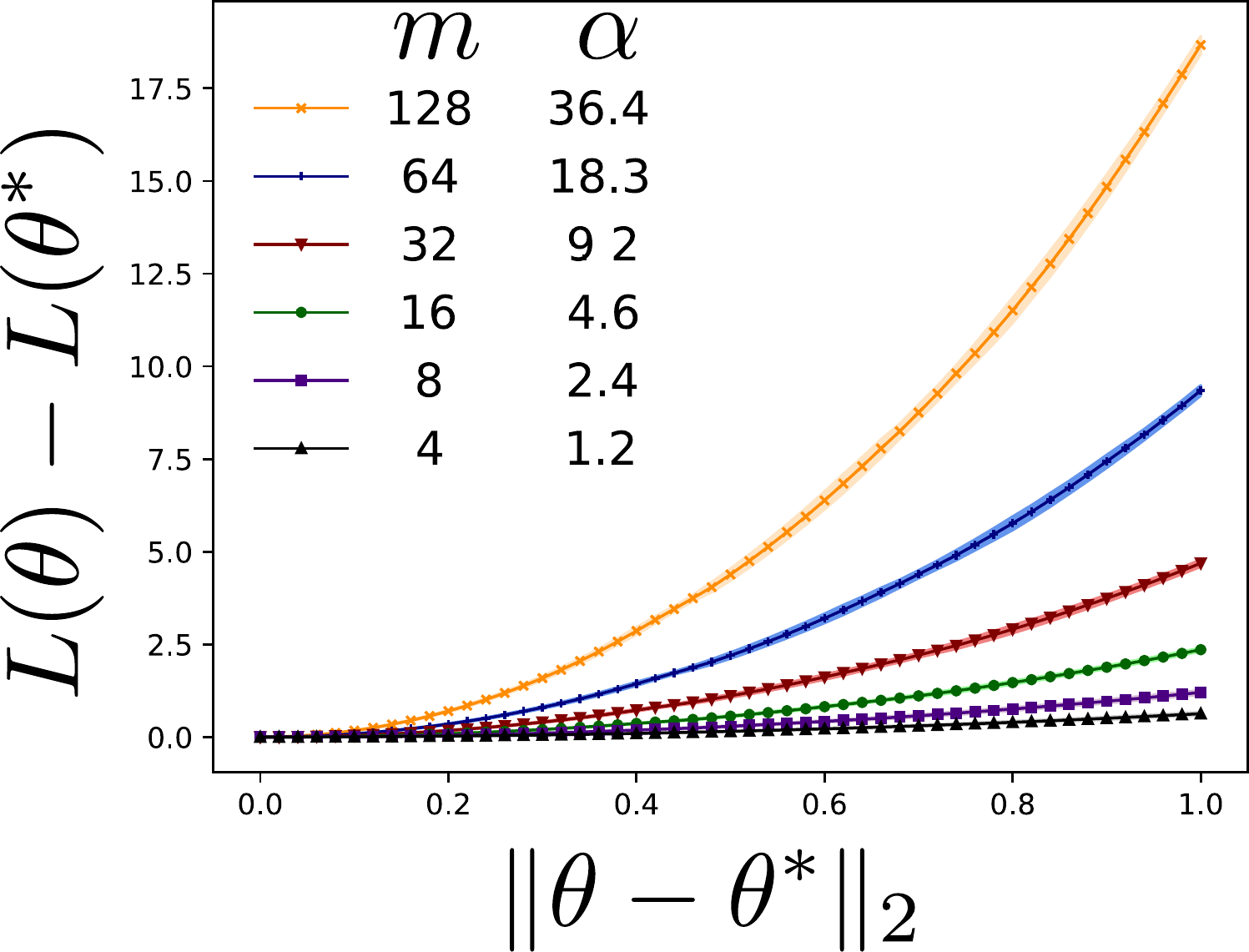}
	\caption{\label{fig:growth}
		Plot of the smallest $L(\theta)-L(\theta^\ast)$ as $\|\theta-\theta^\ast\|_2$ increases for logistic regression over a mixture of four 50-dimensional Gaussians (left) and over a four-class text classification task over 50-dimensional CBOW (right).
		For both the $\alpha$ factor of the quadratic-growth condition scales linearly with the number of samples $m$.
	}
\end{figure}

To prove FLI guarantees, we require in Setting~\ref{set:fliadv} that some notion of average loss on each task grows quadratically away from the optimum, which is shown to hold in both a real and a synthetic setting in Figure~\ref{fig:growth}.
\begin{Set}\label{set:fliadv}
In Setting~\ref{set:exact}, for each task $t\in[T]$ define average loss $L_t$ according to one of the following two cases:
\begin{itemize}
\vspace{-4mm}
\item[(a)] $L_t(\theta)=\frac1{m_t}\sum_{i=1}^{m_t}\ell_{t,i}(\theta)$
\vspace{-2mm}
\item[(b)] assume losses $\ell_{t,i}:\Theta\mapsto[0,1]$ are i.i.d. from distribution $\mathcal P_t$ s.t. 
$L_t(\theta)=\E_{\mathcal P_t}\ell(\theta)$ has a unique minimum 
\vspace{-4mm}
\end{itemize}
\vspace{-4mm}
Assume the corresponding $L_t$ in each case is $\alpha$-QG w.r.t. $\|\cdot\|$ and define $\Theta^\ast\subset\Theta$ s.t. $\Theta^\ast\supset\argmin_{\theta\in\Theta}L(\theta)~\forall~t\in[T]$.
\end{Set}
Here case (b) is the batch-within-online setting, also studied by \citet{alquier:17}.
In this case the distance defining the similarity is between the true-risk minimizers and {\em not} the optimal parameters in hindsight.
Under such data-dependent assumptions we have the following bound on using approximate meta-updates:
\begin{Thm}\label{thm:fliadv}
	In Setting~\ref{set:fliadv}(a), the FLI-Online variant of Algorithm~\ref{alg:fmrl} with $\varepsilon=\Omega\left(\frac1{\sqrt[6]m}\right)$, tuning parameter $\gamma\ge1$, and within-task algorithm FTRL with Bregman regularizer $\Breg_R$ for $R$ strongly-smooth w.r.t. $\|\cdot\|$ achieves TAR
	\vspace{-1mm}
	$$\TAR\le\mathcal O\left(D^\ast+\frac D{D^\ast}\left(\frac{\log T}T+o_m(1)\right)\right)\sqrt m\vspace{-1mm}$$
	for $D^\ast$ as in Theorem~\ref{thm:fml} and $o_m(1)=\mathcal O(m^{-\frac16})$.
	In Setting~\ref{set:fliadv}(b) the same bound holds w.p. $1-\delta$ and $o_m(1)=\mathcal O\left(m^{-\frac16}\sqrt{\log\frac{Tm}\delta}\right)$ for both the FAL and FLI-Batch variants and using either FTRL or OMD within-task. 
\end{Thm}
This bound is very similar to Theorem~\ref{thm:fml} apart from a per-task error term due to the use of an estimate of $\theta_t^\ast$.

\subsection{Distributional Learning-to-Learn}

While gradient-based LTL methods are largely online, their goals are often statistical.
The usual setting due to \citet{baxter:00} assumes a distribution $\mathcal Q$ over task-distributions $\mathcal P$ over functions $\ell$, which can correspond to a single-sample loss.
Given i.i.d. samples from each of $T$ i.i.d. task-samples $\mathcal P_t\sim\mathcal Q$, we seek to do learn how to do well given $m$ samples from a new distribution $\mathcal P\sim\mathcal Q$.
Here we hope that samples from $\mathcal Q$ can reduce the amount needed from $\mathcal P$.


Theorem~\ref{thm:batch} gives an online-to-batch conversion for which low TAR implies low expected risk of a new task sampled from $\mathcal Q$.
For \Eph, the procedure draws $t\sim\mathcal U[T]$, runs $\FTRL_{\eta_t,\phi_t}$ or $\OMD_{\eta,\phi_t}$ on samples from $\mathcal P\sim\mathcal Q$, and outputs the average iterate $\bar\theta$.
Such guarantees on random or mean iterates are standard, although in practice the last iterate is used.
The proof uses Jensen's inequality to combine two standard conversions \cite{cesa-bianchi:04}.

\begin{Thm}\label{thm:batch}
	Suppose convex losses $\ell_{t,i}:\Theta\mapsto[0,1]$ are drawn i.i.d. from $\mathcal P_t\sim\mathcal Q,\{\ell_{t,i}\}_i\sim\mathcal P_t^m$ for some distribution $\mathcal Q$ over task distributions $\mathcal P_t$.
	Let $\mathcal A_t$ be the state (e.g. the initialization $\phi_t$ and similarity guess $D_t$) before task $t\in[T]$ of an algorithm $\mathcal A$ with TAR $\TAR$.
	Then w.p. $1-\delta$ if $m$ loss functions $\{\ell_i\}_i\sim\mathcal P^m$ are sampled from task distribution $\mathcal P\sim\mathcal Q$, running $\mathcal A_t$ on these losses will generate $\theta_1,\dots,\theta_m\in\Theta$ s.t. their mean $\bar\theta$ satisfies
	\vspace{-2mm}
	$$\E_{t\sim\mathcal U[T]}\E_{\begin{smallmatrix}\ell\sim\mathcal P\\\mathcal P\sim\mathcal Q\end{smallmatrix}}\E_{\mathcal P^m}\ell(\bar\theta)
	=\E_{\begin{smallmatrix}\ell\sim\mathcal P\\\mathcal P\sim\mathcal Q\end{smallmatrix}}\ell(\theta^\ast)+\frac\TAR m+\sqrt{\frac8T\log\frac1\delta}
	\vspace{-2mm}
	$$
\end{Thm}

\begin{figure}
	\ifdefined\Ephemeral
		\includegraphics[width=0.49\linewidth]{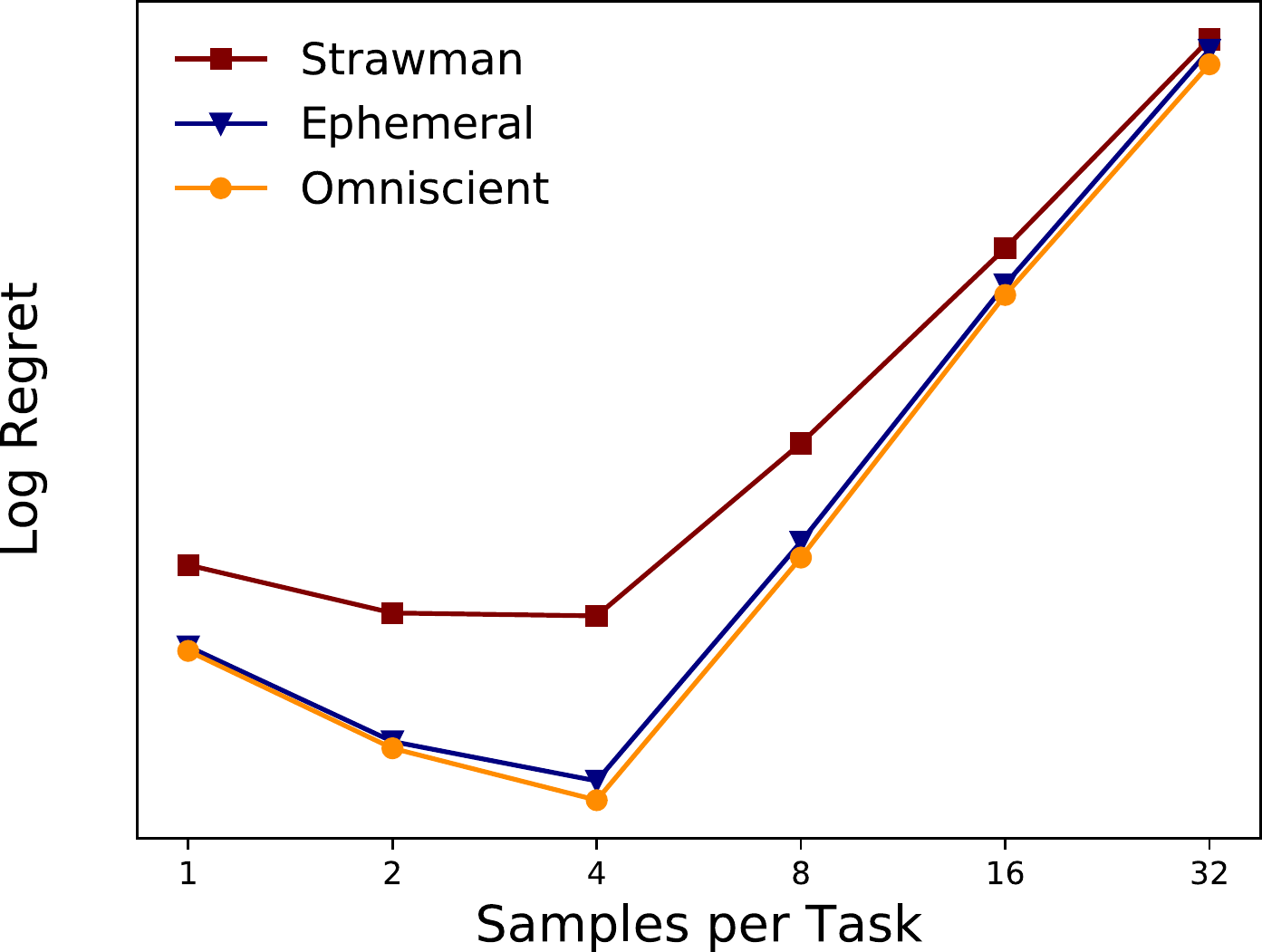}
	\else
		\includegraphics[width=0.49\linewidth]{strawman.pdf}
	\fi
	\hfill
	\includegraphics[width=0.49\linewidth]{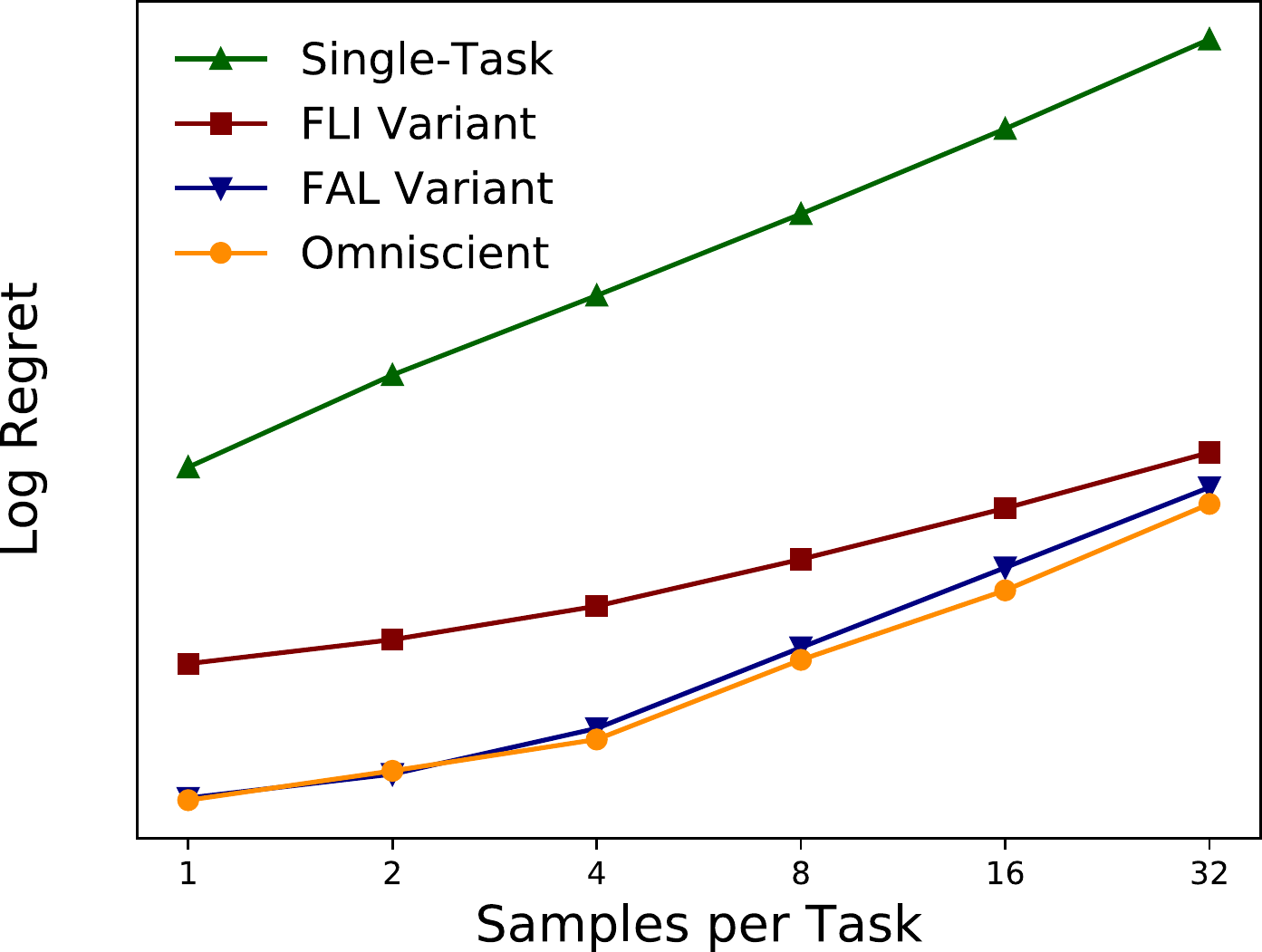}
	\caption{\label{fig:samples}
		TAR of \Eph and the strawman method for FTRL (left) and of variants of \Eph for OGD (right).
		\Eph is much better than the strawman at low $m$, showing the significance of Theorem~\ref{thm:fml} in the few-shot case.
		As predicted by Theorem~\ref{thm:fliadv}, FLI regret converges to that of FAL as $m$ increases.
	}
\end{figure}

\section{Empirical Results}\label{sec:empirical}

An important aspect of \Eph is its practicality.
n particular, FLI-Batch is scalable without modification to high-dimensional, non-convex models.
This is demonstrated by the success of Reptile \cite{nichol:18}, a sub-case of our method that competes with MAML on standard meta-learning benchmarks.
Given this evidence, empirically our goal is to validate our theory in the convex setting, although we also examine implications for deep meta-learning.

\subsection{Convex Setting}

We introduce a new dataset of 812 classification tasks, each consisting of sentences from one of four Wikipedia pages which we use as labels.
It is derived from the raw super-set of the Wiki3029 corpus collected by \citet{arora:19}.
We call the new dataset {\em Mini-Wiki} and make it available in the supplement.
Our use of text classification to examine the convex setting is motivated by the well-known effectiveness of linear models over simple representations \cite{wang:12,arora:18}.
We use logistic regression over 50-dimensional continuous-bag-of-words (CBOW) using GloVe embeddings \cite{pennington:14}.
The similarity of these tasks is verified by seeing if their optimal parameters are close together.
As shown before in Figure~\ref{fig:similarity}, we find when $\Theta$ is the unit ball that even in the 1-shot setting the tasks have non-vacuous similarity;
for 32-shots the parameters are contained in a set of radius 0.32.

We next compare \Eph to the ``strawman" algorithm from Section~\ref{sec:meta}, which uses the previous optimal action as the initialization.
For both algorithms we use similarity guess $\varepsilon=0.1$ and tune with $\gamma=1.1$.
As expected, in Figure~\ref{fig:samples} we see that \Eph is superior to the strawman algorithm, especially for few-shot learning, demonstrating that our TAR improvement is significant in the low-sample regime.
We also see that FLI-Batch, which uses approximate meta-updates, approaches FAL as the number of samples increases and thus its estimate improves.

\begin{figure}
	\centering
	\ifdefined\Ephemeral
		\includegraphics[width=0.55\linewidth]{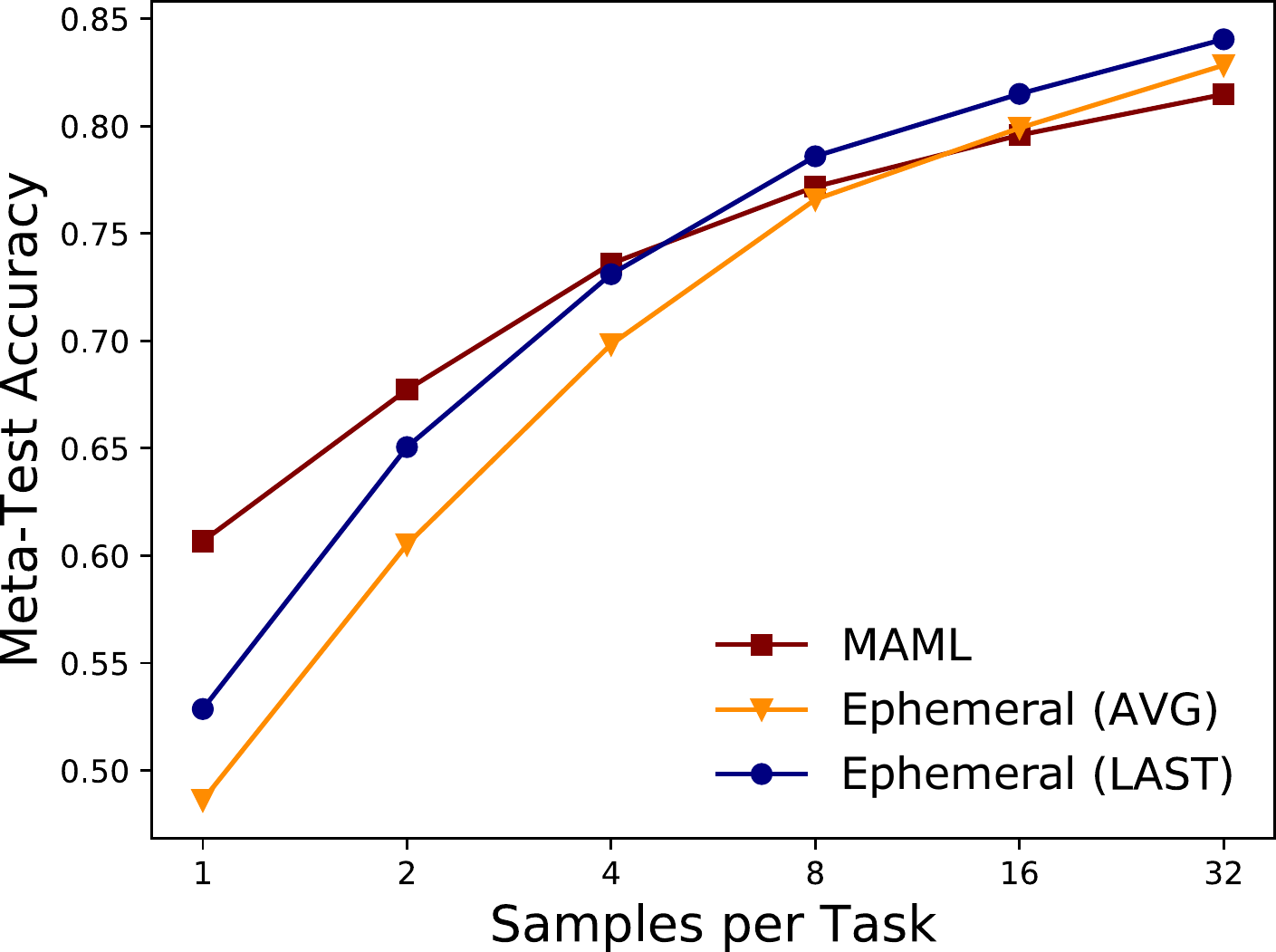}
	\else
		\includegraphics[width=0.55\linewidth]{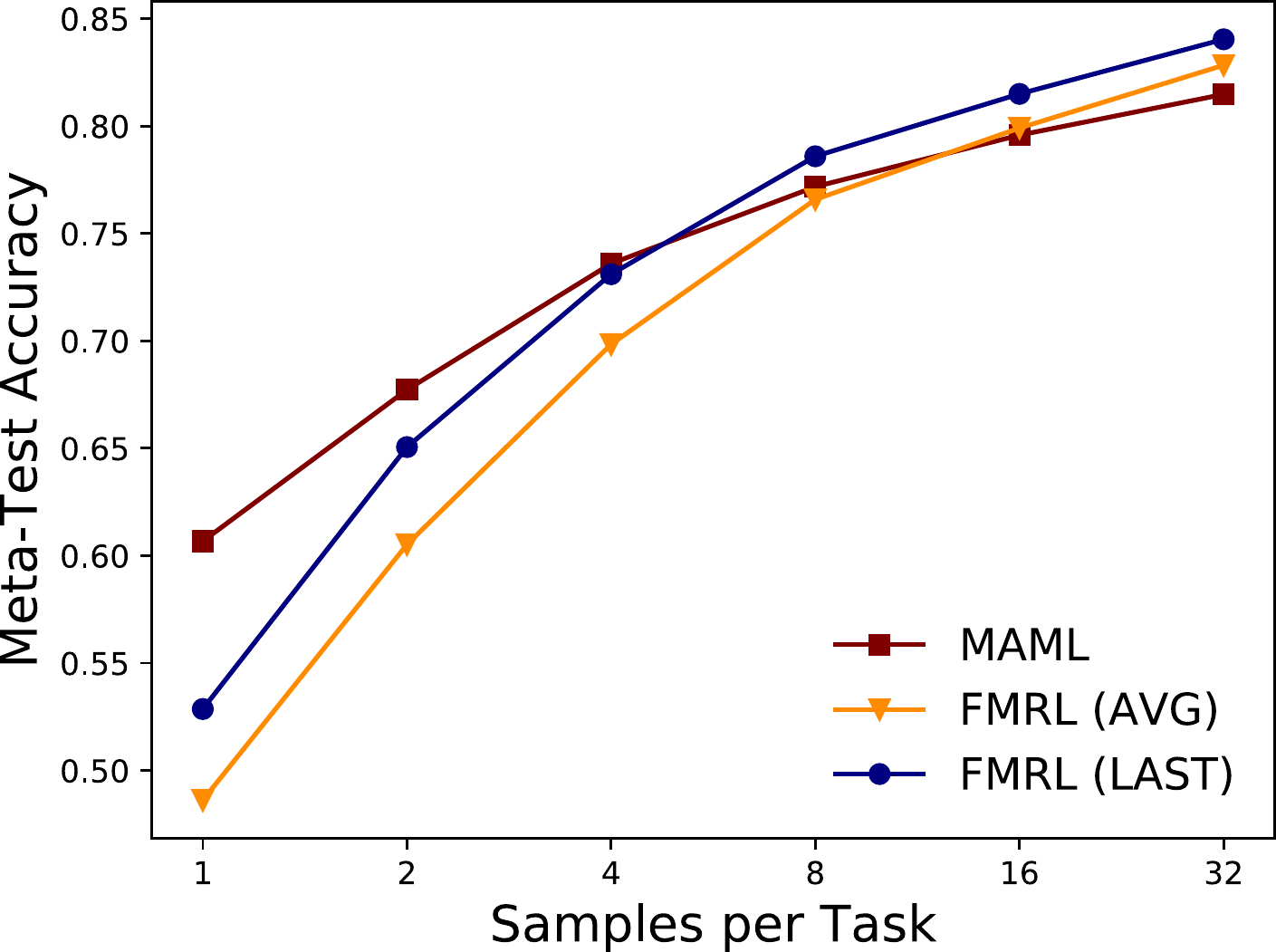}
	\fi
	\caption{\label{fig:batch}
		Meta-test accuracy of MAML and \Eph in the batch setting.
		Both using the average iterate, as recommended by online-to-batch conversion, and using the last iterate, as done in practice, provides performance comparable to that of MAML.
	}
\end{figure}

\begin{figure*}
	\centering
	\includegraphics[width=0.25\linewidth]{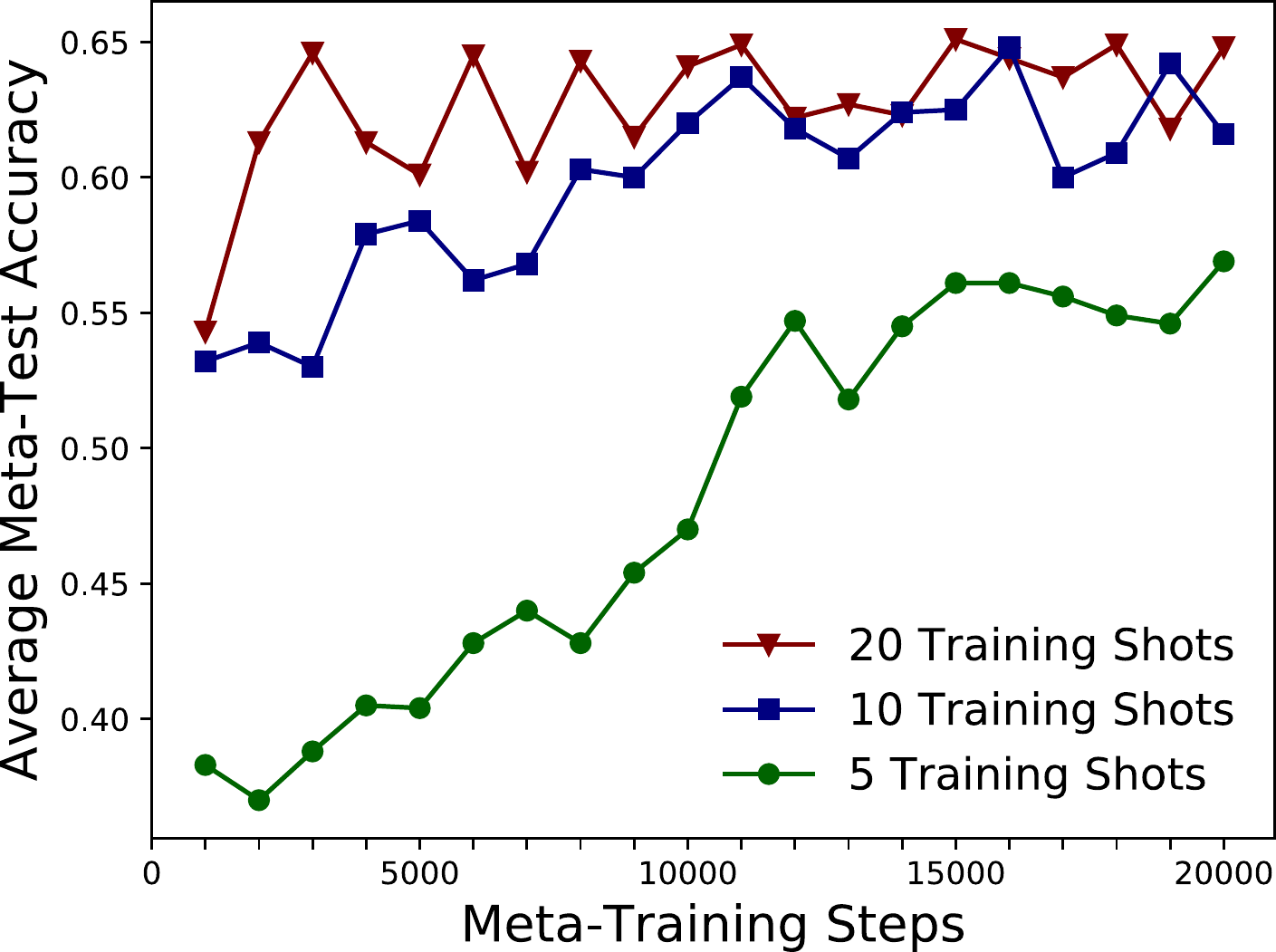}
	\hspace{0.1\linewidth}
	\includegraphics[width=0.25\linewidth]{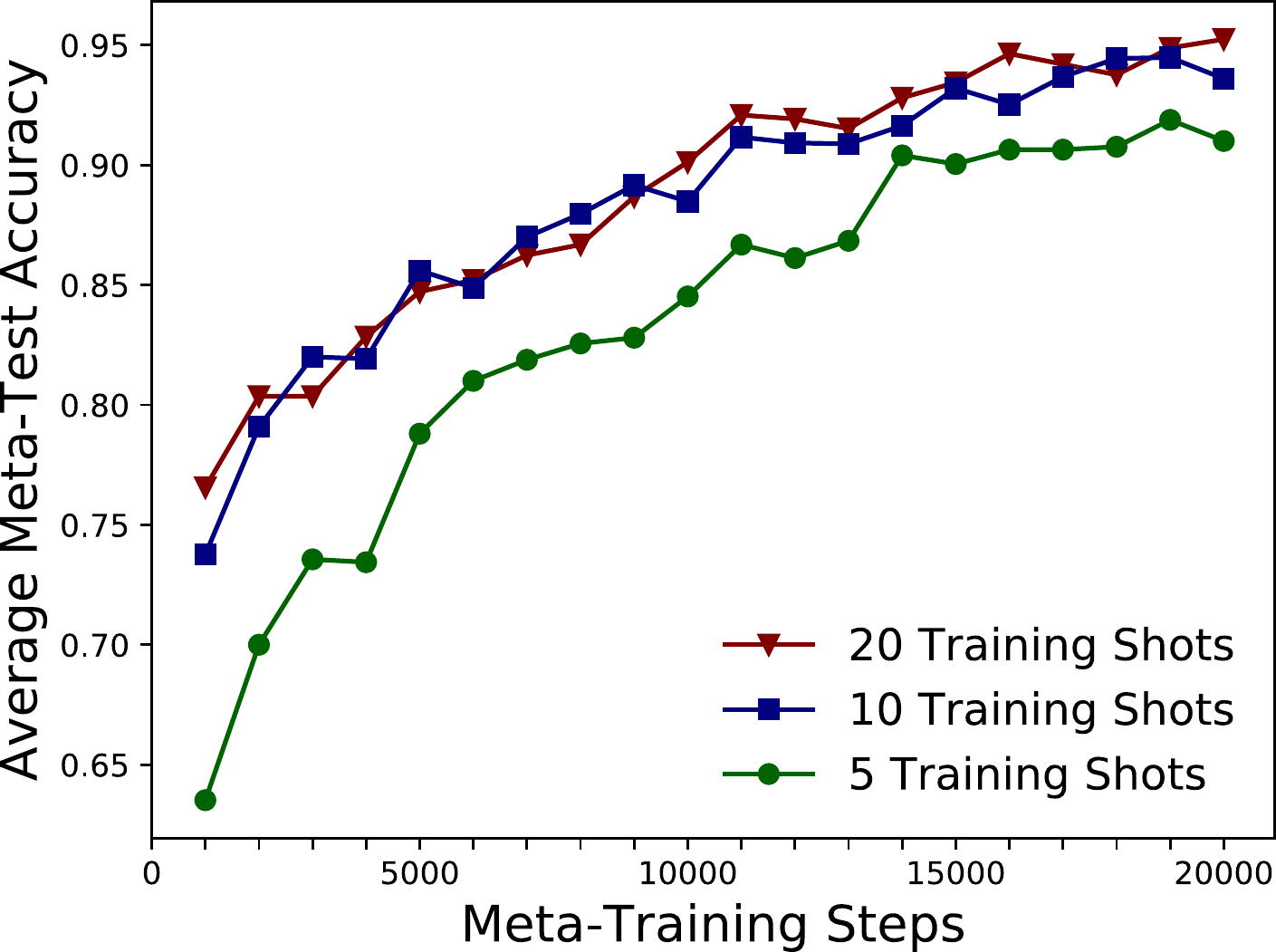}
	\hspace{0.1\linewidth}
	\includegraphics[width=0.25\linewidth]{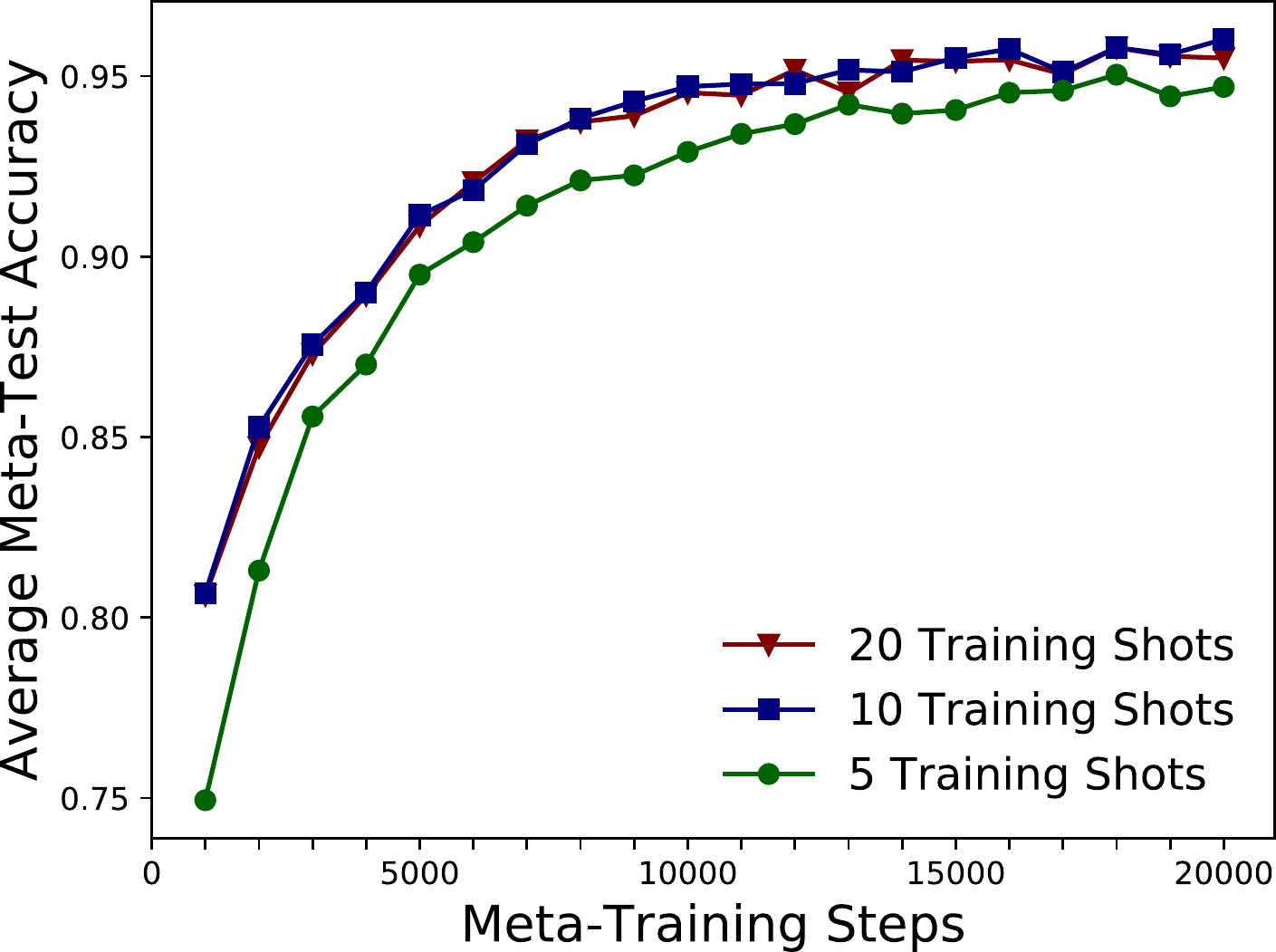}
	\vspace{-2mm}
	\caption{\label{fig:shot}
		Performance of Reptile (the FLI variant of \Eph using OGD within-task) on 5-shot 5-way Mini-ImageNet (left), 1-shot 5-way Omniglot (center), and 5-shot 20-way Omniglot (right) while varying the number of {\em training} samples.
		Increasing the number of samples per training task improves performance even when using the same number of samples at meta-test time.
		\vspace{-2mm}
	}
\end{figure*}

\begin{figure*}
	\centering
	\includegraphics[width=0.25\linewidth]{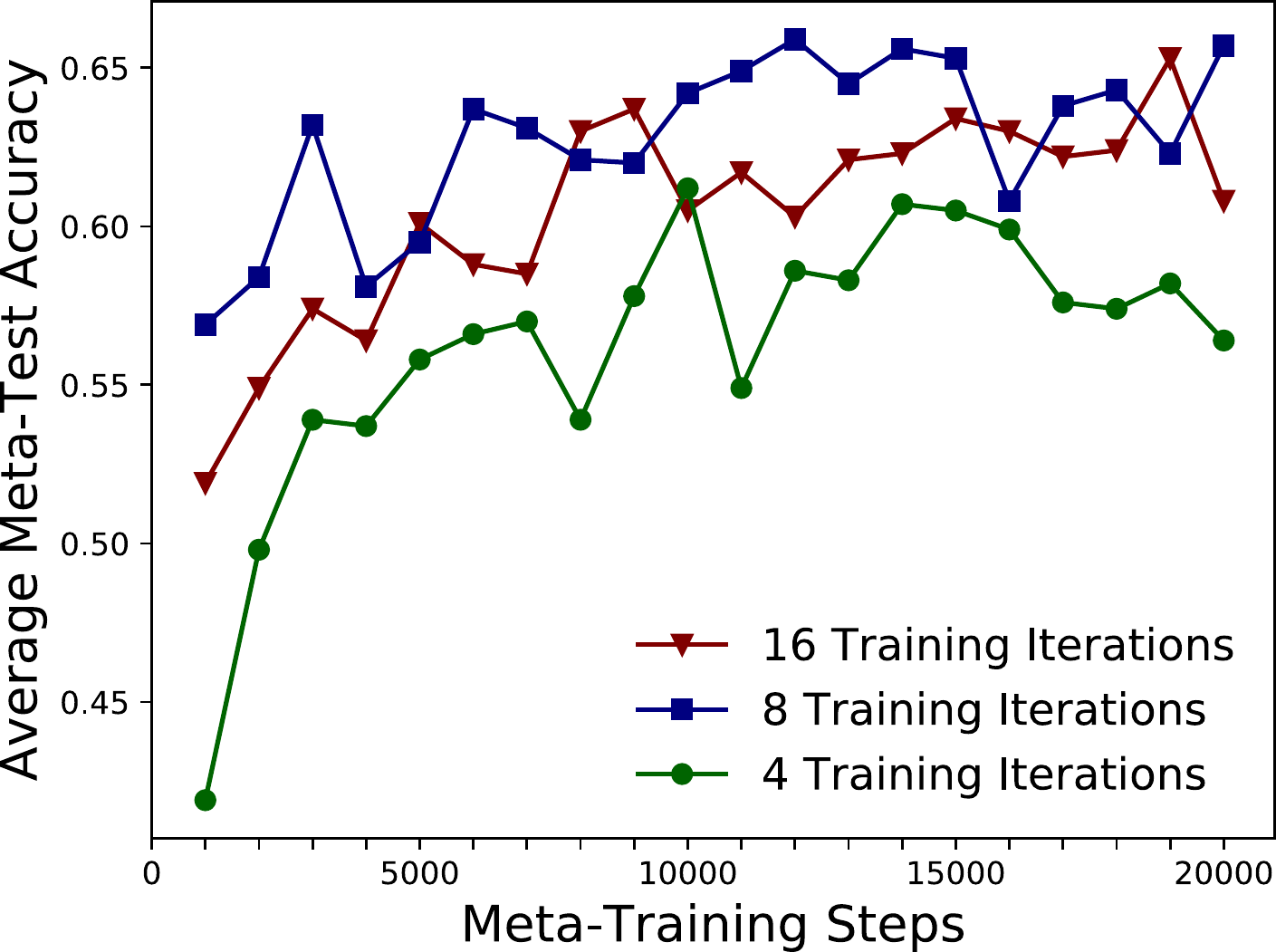}
	\hspace{0.1\linewidth}
	\includegraphics[width=0.25\linewidth]{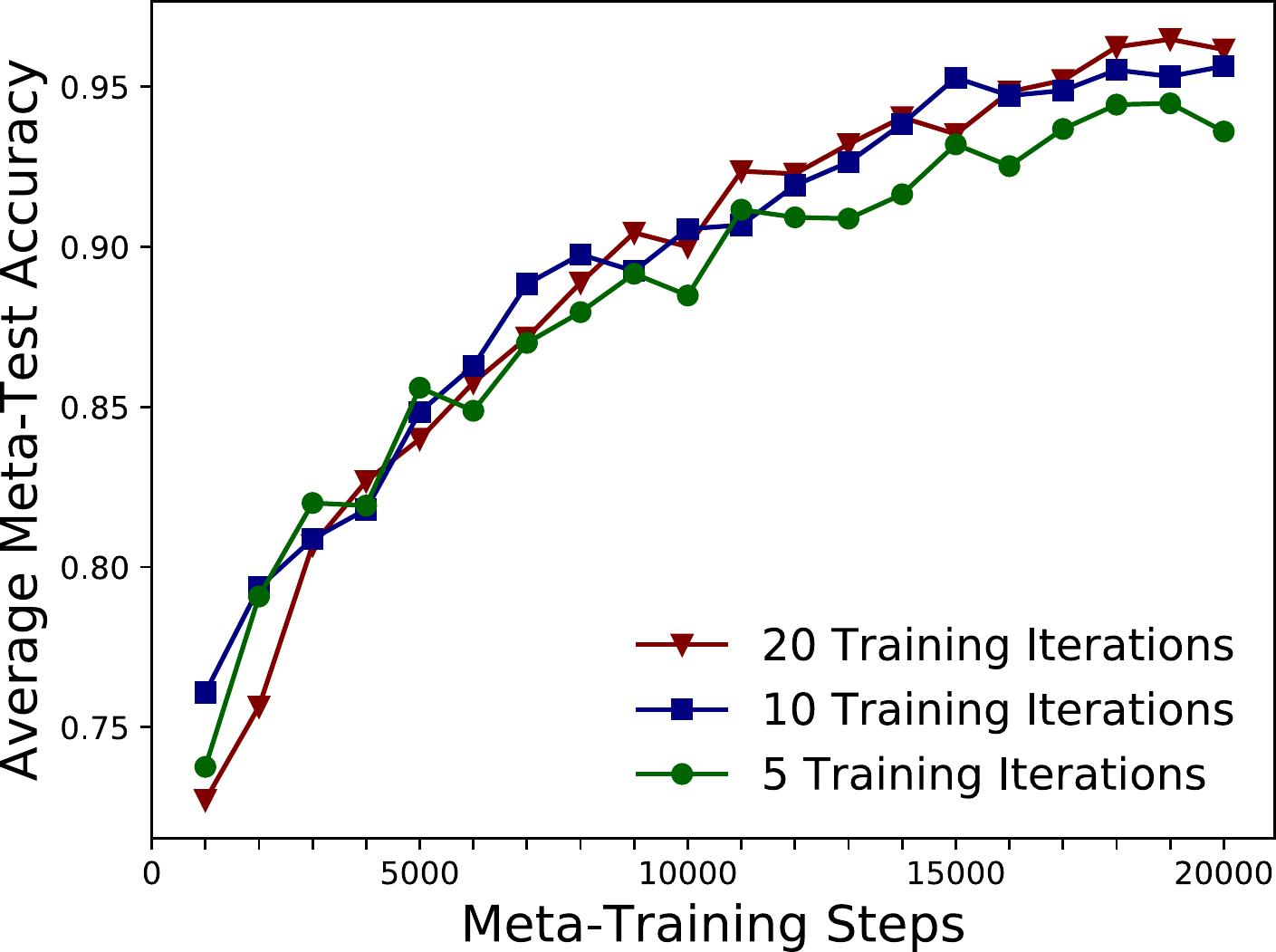}
	\hspace{0.1\linewidth}
	\includegraphics[width=0.25\linewidth]{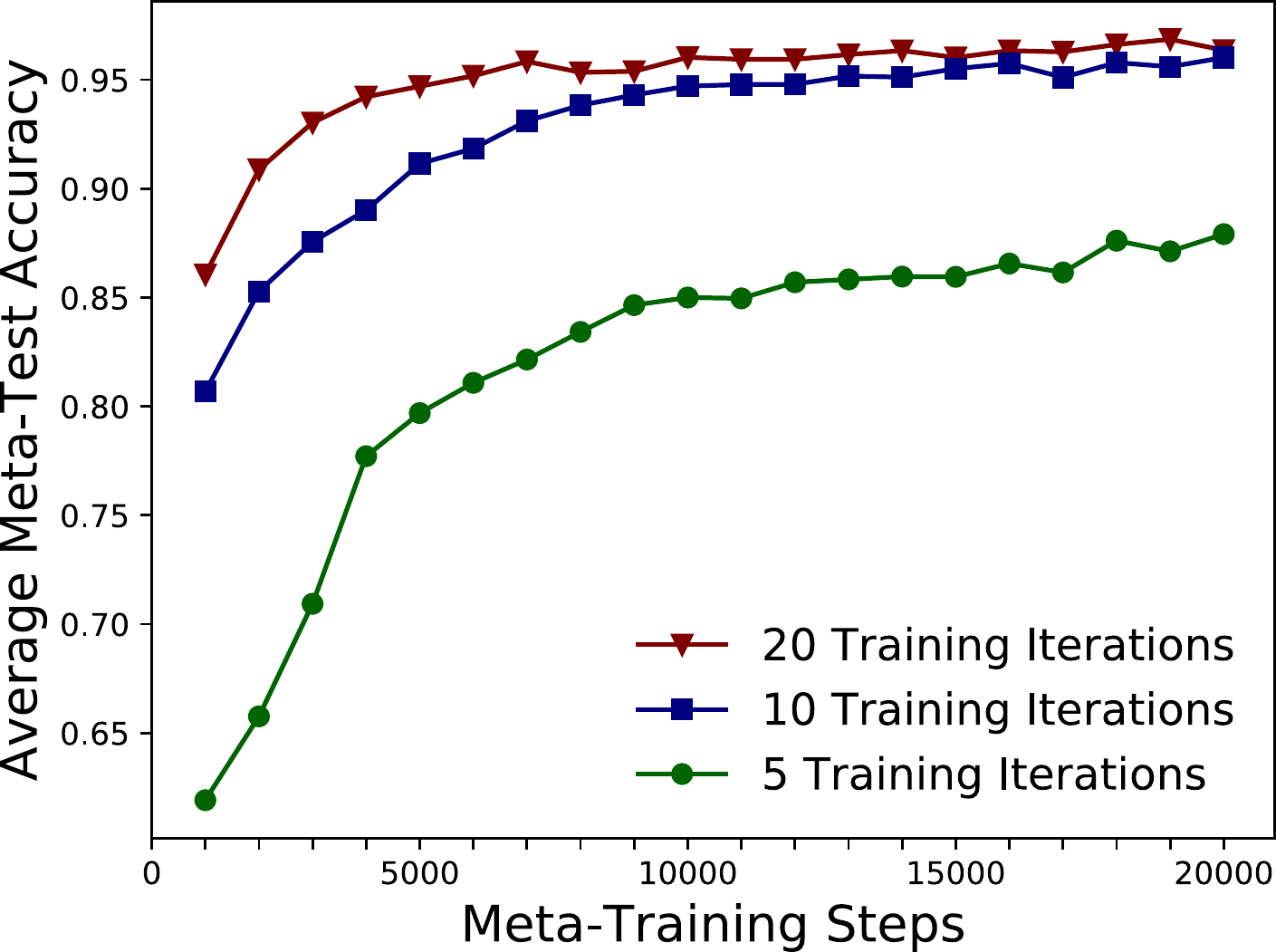}
	\vspace{-2mm}
	\caption{\label{fig:iter}
		Performance of Reptile (the FLI variant of \Eph using OGD within-task) on 5-shot 5-way Mini-ImageNet (left), 1-shot 5-way Omniglot (center), and 5-shot 20-way Omniglot (right) while varying the number of {\em training} iterations.
		The benefit of more iterations is not clear for Mini-ImageNet, but an improvement is seen on Omniglot.
		The number of iterations at meta-test time is 50.
		\vspace{-2mm}
	}
\end{figure*}

Finally, we evaluate \Eph and (first-order) MAML in the statistical setting.
On each task we standardize data using the mean and deviation of the training features.
For \Eph we use the FAL variant with OGD as the within-task algorithm, with learning rate set using the average deviation of the task parameters from the mean parameter, as suggested in Remark~\ref{rem:dev}.
For MAML, we use grid search to determine the within-task and meta-update learning rates.
As shown in Figure~\ref{fig:batch}, despite using no tuning, \Eph performs comparably to MAML  -- slightly better for $m\ge 8$ and slightly worse for $m<4$.

\subsection{Deep Learning}

While our method generalizes Reptile, an effective meta-learning method \cite{nichol:18}, we can still examine if our theory can help neural network LTL.
We study modifications of Reptile on 5-way and 20-way Omniglot \cite{lake:11} and 5-way Mini-ImageNet classification \cite{ravi:17} using the same networks as \citet{nichol:18}.
As in these works, we evaluate in the {\em transductive} setting, where test points are evaluated in batch.

Our theory points to the importance of accurately computing the within-task parameter for the meta-update; 
Theorem~\ref{thm:fml} assumes access to this parameter, whereas Theorems~\ref{thm:fliadv} allow computational and stochastic approximations that result in an additional error term decaying with number of task-examples.
This becomes relevant in the non-convex setting with many tasks, where it is infeasible to find even a local optimum.
Thus we see how a better estimate of the within-task parameter for the meta-update may lead to higher accuracy.
We can attain a better estimate by using more samples to reduce stochastic noise or by running more gradient steps on each task to reduce approximation error.
It is not obvious that these changes will improve performance -- it may be better to learn using the same settings at meta-train and meta-test time.
However, for 5-shot evaluation the Reptile authors do indeed use more than 5 task samples -- 10 for Omniglot and 15 for Mini-ImageNet.
Similarly, they use far fewer within-task gradient steps -- 5 for Omniglot and 8 for Mini-ImageNet -- at meta-train time than the 50 iterations used for evaluation.

We study how the two settings -- the number of task samples and within-task iterations -- affect meta-test performance.
In Figure~\ref{fig:shot}, we see that more task-samples provide a significant improvement, with fewer meta-iterations needed for good test performance.
Reducing this number is equivalent to reducing task-sample complexity, although for a better approximation each task needs more samples.
We also see in Figure~\ref{fig:iter} that taking more gradient steps, which does not use more samples, can also help performance, especially on 20-way Omniglot.
However, on Mini-ImageNet using than 8 iterations reduces performance;
this may be due to over-fitting on specific tasks, with task similarity likely holding for the true rather than empirical risk minimizers, as in Setting~\ref{set:fliadv}(b).
The broad patterns shown above also hold for several other settings, which we discuss in the supplement.

%


\section{Conclusion}\label{sec:conclusion}

In this paper we study a broad class of gradient-based meta-learning methods using the theory of OCO, proving their usefulness compared to single-task learning under a closeness assumption on task parameters.
The guarantees of our algorithm, \Eph, can be extended to approximate meta-updates, the batch-within-online setting, and statistical LTL.
Apart from these results, the algorithm's simplicity makes it extensible to settings of practical interest such as federated learning and differential privacy.
Future work can consider more sophisticated notions of task-similarity, such as multi-modal or evolving settings, and theory for practical and scalable shared-representation-learning.


\section*{Acknowledgments}

This work was supported in part by DARPA FA875017C0141, National Science Foundation grants CCF-1535967, IIS-1618714, IIS-1705121, and IIS-1838017, a Microsoft Research Faculty Fellowship, an Okawa Grant, a Google Faculty Award, an Amazon Research Award, an Amazon Web Services Award, and a Carnegie Bosch Institute Research Award. Any opinions, findings and conclusions or recommendations expressed in this material are those of the author(s) and do not necessarily reflect the views of DARPA, the National Science Foundation, or any other funding agency.

\bibliography{references}
\bibliographystyle{icml2019}

\newpage
\onecolumn
\appendix


\section{Background and Results for Online Convex Optimization}\label{app:background}

Throughout the appendix we assume all subsets are convex and in $\mathbb{R}^d$ unless explicitly stated.
Let $\|\cdot\|_\ast$ be the dual norm of $\|\cdot\|$, which we assume to be any norm on $\mathbb{R}^d$, and note that the dual norm of $\|\cdot\|_2$ is itself.
For sequences of scalars $\sigma_1,\dots,\sigma_T\in\mathbb{R}$ we will use the notation $\sigma_{1:t}$ to refer to the sum of the first $t$ of them.
In the online learning setting, we will use the shorthand $\nabla_t$ to denote the subgradient of $\ell_t:\Theta\mapsto\mathbb R$ evaluated at action $\theta_t\in\Theta$.
We will use $\Conv(S)$ to refer to the convex hull of a set of points $S$ and $\Proj_S(\cdot)$ to be the projection to any convex subset $S\subset\mathbb R^d$.

\subsection{Convex Functions}\label{subsec:functions}

We first state the related definitions of {\em strong convexity} and {\em strong smoothness}:
\begin{Def}\label{def:convex}
	An everywhere sub-differentiable function $f:S\mapsto\mathbb R$ is {\bf $\alpha$-strongly-convex} w.r.t. norm $\|\cdot\|$ if
	$$f(y)\ge f(x)+\langle\nabla f(x),y-x\rangle+\frac\alpha2\|y-x\|^2~\forall~x,y\in S$$
\end{Def}
\begin{Def}\label{def:smooth}
	An everywhere sub-differentiable function $f:S\mapsto\mathbb{R}$ is {\bf $\beta$-strongly-smooth} w.r.t. norm $\|\cdot\|$ if
	$$f(y)\le f(x)+\langle\nabla f(x),y-x\rangle+\frac\beta2\|y-x\|^2~\forall~x,y\in S$$
\end{Def}

%

We now turn to the {\em Bregman divergence} and a discussion of several useful properties \cite{bregman:67,banerjee:05}:
\begin{Def}\label{def:bregman:app}
	Let $f:S\mapsto\mathbb{R}$ be an everywhere sub-differentiable strictly convex function.
	Its {\bf Bregman divergence} is defined as
	$$\Breg_f(x||y)=f(x)-f(y)-\langle\nabla f(y),x-y\rangle$$
	The definition directly implies that $\Breg_f(\cdot||y)$ preserves the (strong or strict) convexity of $f$ for any fixed  $y\in S$.
	Strict convexity further implies $\Breg_f(x||y)\ge0~\forall~x,y\in S$, with equality iff $x=y$.
	Finally, if $f$ is $\alpha$-strongly-convex, or $\beta$-strongly-smooth, w.r.t. $\|\cdot\|$ then Definition~\ref{def:convex} implies $\Breg_f(x||y)\ge\frac\alpha2\|x-y\|^2$, or $\Breg_f(x||y)\le\frac\beta2\|x-y\|^2$, respectively.
\end{Def}
\begin{Clm}\label{clm:mean}
	Let $f:S\mapsto\mathbb{R}$ be a strictly convex function on $S$, $\alpha_1,\dots,\alpha_n\in\mathbb{R}$ be a sequence satisfying $\alpha_{1:n}>0$, and $x_1,\dots,x_n\in S$.
	Then 
	$$\bar x=\frac{1}{\alpha_{1:n}}\sum_{i=1}^n\alpha_ix_i=\argmin_{y\in S}\sum_{i=1}^n\alpha_i\Breg_f(x_i||y)$$
\end{Clm}
\begin{proof}
	$\forall~y\in S$ we have
	\begin{align*}
		\sum_{i=1}^n\alpha_i\left(\Breg_f(x_i||y)-\Breg_f(x_i||\bar x)\right)
		&=\sum_{i=1}^n\alpha_i\left(f(x_i)-f(y)-\langle\nabla f(y),x_i-y\rangle-f(x_i)+f(\bar x)+\langle\nabla f(\bar x),x_i-\bar x\rangle\right)\\
		&=\left(f(\bar x)-f(y)+\langle\nabla f(y),y\rangle\right)\alpha_{1:n}+\sum_{i=1}^n\alpha_i\left(-\langle\nabla f(\bar x),\bar x\rangle+\langle\nabla f(\bar x)-\nabla f(y),x_i\rangle\right)\\
		&=\left(f(\bar x)-f(y)-\langle\nabla f(y),\bar x-y\rangle\right)\alpha_{1:n}\\
		&=\alpha_{1:n}\Breg_f(\bar x||y)
	\end{align*}
	By Definition~\ref{def:bregman:app} the last expression has a unique minimum at $y=\bar x$.
\end{proof}

\newpage
\subsection{Standard Online Algorithms}\label{subsec:algorithms}

Here we provide a review of the online algorithms we use.
Recall that in this setting our goal is minimizing regret:
\begin{Def}\label{def:regret}
	The {\bf regret} of an agent playing actions $\{\theta_t\in\Theta\}_{t\in[T]}$ on a sequence of loss functions $\{\ell_t:\Theta\mapsto\mathbb R\}_{t\in[T]}$ is
	$$\R_T=\sum_{t=1}^T\ell_t(\theta_t)-\min_{\theta\in\Theta}\sum_{t=1}^T\ell_t(\theta)$$
\end{Def}
Within-task our focus is on two closely related meta-algorithms, Follow-the-Regularized-Leader (FTRL) and (linearized lazy) Online Mirror Descent (OMD).
\begin{Def}\label{def:app:ftrl}
	Given a strictly convex function $R:\Theta\mapsto\mathbb R$, starting point $\phi\in\Theta$, fixed learning rate $\eta>0$, and a sequence of functions $\{\ell_t:\Theta\mapsto\mathbb R\}_{t\ge1}$, {\bf Follow-the-Regularized Leader ($\FTRL_{\phi,\eta}^{(R)}$)} plays
	$$\theta_t=\argmin_{\theta\in\Theta}\Breg_R(\theta||\phi)+\eta\sum_{s<t}\ell_s(\theta)$$
\end{Def}
\begin{Def}\label{def:omd}
	Given a strictly convex function $R:\Theta\mapsto\mathbb R$, starting point $\phi\in\Theta$, fixed learning rate $\eta>0$, and a sequence of functions $\{\ell_t:\Theta\mapsto\mathbb R\}_{t\ge1}$, {\bf lazy linearized Online Mirror Descent ($\OMD_{\phi,\eta}^{(R)}$)} plays
	$$\theta_t=\argmin_{\theta\in\Theta}\Breg_R(\theta||\phi)+\eta\sum_{s<t}\langle\nabla_s,\theta\rangle$$
\end{Def}

These formulations make the connection between the two algorithms -- their equivalence in the linear case $\ell_s(\cdot)=\langle\nabla_s,\cdot\rangle$ -- very explicit.
There exists a more standard formulation of OMD that is used to highlight its generalization of OGD -- the case of $R(\cdot)=\frac12\|\cdot\|_2^2$ -- and the fact that the update is carried out in the dual space induced by $R$ \citep[Section~5.3]{hazan:15}.
However, we will only need the following regret bound satisfied by both \citep[Theorems~2.11 and~2.15]{shalev-shwartz:11}
\begin{Thm}\label{thm:ftrlomd}
	Let $\{\ell_t:\Theta\mapsto\mathbb R\}_{t\in[T]}$ be a sequence of convex functions that are $G_t$-Lipschitz w.r.t. $\|\cdot\|$ and let $R:S\mapsto\mathbb R$ be 1-strongly-convex.
	Then the regret of both $\FTRL_{\eta,\phi}^{(R)}$ and $\OMD_{\eta,\phi}^{(R)}$ is bounded by
	$$\R_T\le\frac{\Breg_R(\theta^\ast||\phi)}\eta+\eta G^2T$$
	for all $\theta^\ast\in\Theta$ and $G^2\ge\frac1T\sum_{t=1}^TG_t^2$.
\end{Thm}

We next review the online algorithms we use for the meta-update.
The main requirement here is logarithmic regret guarantees for the case of strongly convex loss functions, which is satisfied by two well-known algorithms:
\begin{Def}\label{def:ftl}
	Given a sequence of strictly convex functions $\{\ell_t:\Theta\mapsto\mathbb R\}_{t\ge1}$, {\bf Follow-the-Leader (FTL)} plays arbitrary $\theta_1\in\Theta$ and for $t>1$ plays
	$$\theta_t=\argmin_{\theta\in\Theta}\sum_{s<t}\ell_s(\theta)$$
\end{Def}
\begin{Def}\label{def:aogd}
	Given a sequence of functions $\{\ell_t:\Theta\mapsto\mathbb R\}_{t\ge1}$ that are $\alpha_t$-strongly-convex w.r.t. $\|\cdot\|_2$, {\bf Adaptive OGD (AOGD)} plays arbitrary $\theta_1\in\Theta$ and for $t>1$ plays
	$$\theta_{t+1}=\Proj_\Theta\left(\theta_t-\frac1{\alpha_{1:t}}\nabla f(\theta_t)\right)$$
\end{Def}

\citet[Theorem~2]{kakade:08} and \citet[Theorem~2.1]{bartlett:08} provide for FTL and AOGD, respectively, the following regret bound:
\begin{Thm}\label{thm:ftlaogd}
	Let $\{\ell_t:\Theta\mapsto\mathbb R\}_{t\in[T]}$ be a sequence of convex functions that are $G_t$-Lipschitz and $\alpha_t$-strongly-convex w.r.t. $\|\cdot\|$.
	Then the regret of both FTL and AOGD is bounded by
	$$\R_T\le\frac12\sum_{t=1}^T\frac{G_t^2}{\alpha_{1:t}}$$
\end{Thm}

\newpage
One further useful fact about FTL and AOGD is that when run on a sequence of Bregman regularizers $\Breg_R(\theta_1||\cdot),\dots,\Breg_R(\theta_T||\cdot)$ they will play points in the convex hull $\Conv(\{\theta_t\}_{t\in[T]})$:

\begin{Clm}\label{clm:hull}
	Let $R:\Theta\mapsto\mathbb R$ be 1-strongly-convex w.r.t. $\|\cdot\|$ and consider any $\theta_1,\dots,\theta_T\in\Theta^\ast$ for some convex subset $\Theta^\ast\subset\Theta$.
	Then for loss sequence $\alpha_1\Breg_R(\theta_1||\cdot),\dots,\alpha_T\Breg_R(\theta_T||\cdot)$ for any positive scalars $\alpha_1,\dots,\alpha_T\in\mathbb{R}_+$, if we assume $\phi_1\in\Theta^\ast$ then FTL will play $\phi_t\in\Theta^\ast~\forall~t$ and AOGD will as well if we further assume $R(\cdot)=\frac12\|\cdot\|^2$.
\end{Clm}
\begin{proof}
	The proof for FTL follows directly from Claim~\ref{clm:mean} and the fact that the weighted average of a set of points is in their convex hull.
	For AOGD we proceed by induction on $t$.
	The base case $t=1$ holds by the assumption $\phi_t\in\Theta^\ast$.
	In the inductive case, note that $\Breg_R(\theta_t||\phi_t)=\frac12\|\theta_t-\phi_t\|_2^2$ so the gradient update is $\phi_{t+1}=\phi_t+\frac{\alpha_t}{\alpha_{1:t}}(\theta_t-\phi_t)$, which is on the line segment between $\phi_t$ and $\theta_t$, so the proof is complete by the convexity of $\Theta^\ast\ni\phi_t,\theta_t$.
\end{proof}

\subsection{Online-to-Batch Conversion}\label{subsec:batch}

Finally, as we are also interested in distributional meta-learning, we discuss some techniques for converting regret guarantees into generalization bounds, which are usually named {\em online-to-batch conversions}.
We state some standard results below:
\begin{Prp}\label{prp:o2bexp}
	If a sequence of bounded convex loss functions $\{\ell_t:\Theta\mapsto\mathbb R\}_{t\in[T]}$ drawn i.i.d. from some distribution $\mathcal D$ is given to an online algorithm with regret bound $\R_T$ that generates a sequence of actions $\{\theta_t\in\Theta\}_{t\in[T]}$ then for $\bar\theta=\frac1T\theta_{1:T}$ and any $\theta^\ast\in\Theta$ we have
	$$\E_{\mathcal D^T}\E_{\ell\sim\mathcal D}\ell(\bar\theta)\le\E_{\ell\sim\mathcal D}\ell(\theta^\ast)+\frac{\R_T}T$$
\end{Prp}
\begin{proof}
	Applying Jensen's inequality and using the fact that $\theta_t$ only depends on $\ell_1,\dots,\ell_{t-1}$ we have
	\begin{align*}
	\E_{\mathcal D^T}\E_{\ell\sim\mathcal D}\ell(\bar\theta)
	\le\frac1T\E_{\mathcal D^T}\sum_{t=1}^T\E_{\ell_t'\sim\mathcal D}\ell_t'(\theta_t)
	&=\frac1T\E_{\{\ell_t\}\sim\mathcal D^T}\left(\sum_{t=1}^T\E_{\ell_t'\sim\mathcal D}\ell_t'(\theta_t)-\ell_t(\theta_t)\right)+\frac1T\E_{\{\ell_t\}\sim\mathcal D^T}\left(\sum_{t=1}^T\ell_t(\theta_t)\right)\\
	&\le\frac1T\sum_{t=1}^T\E_{\{\ell_s\}_{s<t}\sim\mathcal D^{t-1}}\left(\E_{\ell_t'\sim\mathcal D}\ell_t'(\theta_t)-\E_{\ell_t\sim\mathcal D}\ell_t(\theta_t)\right)+\frac{\R_T}T+\frac1T\sum_{t=1}^T\E_{\ell\sim\mathcal D}\ell(\theta^\ast)\\
	&=\frac{\R_T}T+\E_{\ell\sim\mathcal D}\ell(\theta^\ast)
	\end{align*}
\end{proof}
%
\begin{Prp}\label{prp:o2b}
	If a sequence of loss functions $\{\ell_t:\Theta\mapsto[0,1]\}_{t\in[T]}$ drawn i.i.d. from some distribution $\mathcal D$ is given to an online algorithm that generates a sequence of actions $\{\theta_t\in\Theta\}_{t\in[T]}$ then the following inequalities each hold w.p. $1-\delta$:
	$$\frac1T\sum_{t=1}^T\E_{\ell\sim\mathcal D}\ell(\theta_t)\le\frac1T\sum_{t=1}^T\ell_t(\theta_t)+\sqrt{\frac2T\log\frac1\delta}
	\qquad\textrm{and}\qquad\frac1T\sum_{t=1}^T\E_{\ell\sim\mathcal D}\ell(\theta_t)\ge\frac1T\sum_{t=1}^T\ell_t(\theta_t)-\sqrt{\frac2T\log\frac1\delta}$$
\end{Prp}
Note that \citet{cesa-bianchi:04} only prove the first inequality;
the second follows via the same argument but applying the symmetric version of the Azuma-Hoeffding inequality \cite{azuma:67}.
\begin{Cor}\label{cor:o2bwhp}
	If a sequence of loss functions $\{\ell_t:\Theta\mapsto[0,1]\}_{t\in[T]}$ drawn i.i.d. from some distribution $\mathcal D$ is given to an online algorithm with regret bound $\R_T$ that generates a sequence of actions $\{\theta_t\in\Theta\}_{t\in[T]}$ then
	$$\E_{t\sim\mathcal U[T]}\E_{\ell\sim\mathcal D}\ell(\theta_t)\le\E_{\ell\sim\mathcal D}\ell(\theta^\ast)+\frac{\R_T}T+\sqrt{\frac8T\log\frac1\delta}\qquad\textrm{w.p. }1-\delta$$
	for any $\theta^\ast\in\Theta$.
\end{Cor}
\begin{proof}
	By Proposition~\ref{prp:o2b} we have
	$$\frac1T\sum_{t=1}^T\E_{\ell\sim\mathcal D}\ell(\theta_t)
	\le\frac1T\sum_{t=1}^T\ell_t(\theta^\ast)+\frac{\R_T}T+\sqrt{\frac2T\log\frac1\delta}
	\le\E_{\ell\sim\mathcal D}\ell(\theta^\ast)+\frac{\R_T}T+\sqrt{\frac8T\log\frac1\delta}$$
\end{proof}

\clearpage
\section{Proofs of Theoretical Results}

In this section we prove the main guarantees on task-averaged regret for our algorithms, as, lower bounds showing that the results are tight up to constant factors, and online-to-batch conversion guarantees for statistical LTL.
We first define some necessary definitions, notations, and general assumptions.

%
\begin{Set}\label{set:general}
	Using the data given to Algorithm~\ref{alg:fmrl} define the following quantities:
	\begin{itemize}
		\item convenience coefficients $\sigma_t=G_t\sqrt{m_t}$
		\item the sequence of update parameters $\{\hat\theta_t\in\Theta\}_{t\in[T]}$ with average update parameter $\hat\phi=\frac1{\sigma_{1:T}}\sum_{t=1}^T\sigma_t\hat\theta_t$
		\item a sequence of reference parameters $\{\theta_t'\in\Theta\}_{t\in[T]}$ with average reference parameter $\phi'=\frac1{\sigma_{1:T}}\sum_{t=1}^T\sigma_t\theta_t'$
		\item a sequence $\{\theta_t^\ast\in\Theta\}_{t\in[T]}$ of optimal parameters in hindsight
		\item we will say we are in the ``Exact" case if $\hat\theta_t=\theta_t'=\theta_t^\ast~\forall~t$ and the ``Approx" case otherwise
		\item $\kappa\ge1,\Delta_t^\ast\ge0$ s.t. $\sum_{t=1}^T\alpha_t\Breg_R(\theta_t^\ast||\phi_t)\le\sum_{t=1}^T\alpha_t\Delta_t^\ast+\kappa\sum_{t=1}^T\alpha_t\Breg_R(\hat\theta_t||\phi_t)$ for some nonnegative $\alpha_t$
		\item $\nu\ge1,\Delta'\ge0$ s.t. $\sum_{t=1}^T\sigma_t\Breg_R(\hat\theta_t||\hat\phi)\le\Delta'+\nu\sum_{t=1}^T\sigma_t\Breg_R(\theta_t'||\phi')$
		\item $\Delta_{\max}\ge0$ s.t. $\frac12\|\theta_t'-\hat\theta_t\|^2\le\Delta_{\max}~\forall~t\in[T]$
		\item average deviation $\bar D^2=\frac1{\sigma_{1:T}}\sum_{t=1}^T\sigma_t\Breg_R(\theta_t'||\phi')$ of the reference parameters; assumed positive
		\item task diameter $D^\ast=\max_{\theta,\phi\in\Conv(\{\theta_t'\}_{t\in[T]})}\sqrt{\Breg_R(\theta||\phi)}$; assumed positive
		\item action diameter $D^2=\max\{{D^\ast}^2,\max_{\theta\in\Theta}\Breg_R(\theta||\phi_1)\}$ in the Exact case or $\max_{\theta,\phi\in\Theta}\Breg_R(\theta||\phi)$ in the Approx case
		\item universal constant $C'$ s.t. $\|\theta\|\le C'\|\theta\|_2~\forall~\theta\in\Theta$ and $\ell_2$-diameter $D'=\max_{\theta,\phi}\|\theta-\phi\|_2$ of $\Theta$
		\item upper bound $G'$ on the Lipschitz constants of the functions $\{\Breg_R(\hat\theta_t||\cdot)\}_{t\in[T]}$ over $\Conv(\{\hat\theta_t\}_{t=1}^T)$
		\item we will say we are in the ``Nice" case if $\Breg_R(\theta||\cdot)$ is 1-strongly-convex and $\beta$-strongly-smooth w.r.t. $\|\cdot\|~\forall~\theta\in\Theta$
		\item in the general case $\META$ is FTL; in the Nice case $\META$ may instead be AOGD re-initialized at $\theta_1^\ast$
		\item convenience indicator $\iota=1_{\META=\FTL}$
		\item effective meta-action space $\hat\Theta=\Conv(\{\hat\theta_t\}_{t\in[T]})$ if $\META$ is FTL or $\Theta$ if $\META$ is AOGD
		\item $\TASK_{\eta,\phi}=\FTRL_{\eta,\phi}^{(R)}$ or $\OMD_{\eta,\phi}^{(R)}$
	\end{itemize}
	We make the following assumptions:
	\begin{itemize}
		\item the loss functions $\ell_{t,i}$ are convex $\forall~t,i$
		\item at time $t=1$ the update algorithm $\META$ plays $\phi_1\in\Theta$ satisfying $\max_{\theta\in\Theta}\Breg_R(\theta||\phi_1)<\infty$
		\item in the Approx case $R$ is $\beta$-strongly-smooth for some $\beta\ge1$
	\end{itemize}
\end{Set}

\newpage
\subsection{Upper Bound}\label{subsec:upper}

We first prove a technical result on the performance of FTL on a sequence of Bregman regularizers.
We start by lower bounding the regret of FTL when the loss functions are quadratic.
\begin{Lem}\label{lem:quad}
	For any $\theta_1,\dots,\theta_T\in S$ and positive scalars $\alpha_1,\dots,\alpha_T\in\mathbb{R}_+$ define $\phi_t=\frac{1}{\alpha_{1:t}}\sum_{s=1}^t\alpha_t\theta_t$ and let $\phi_0$ be any point in $S$.
	Then
	$$\sum_{t=1}^T\alpha_t\|\theta_t-\phi_{t-1}\|_2^2-\sum_{t=1}^T\alpha_t\|\theta_t-\phi_T\|_2^2\ge0$$
\end{Lem}
\begin{proof}
	We proceed by induction on $T$.
	The base case $T=1$ follows directly since $\phi_1=\theta_1$ and so the second term is zero.
	In the inductive case we have
	$$\sum_{t=1}^{T-1}\alpha_t\|\theta_t-\phi_{t-1}\|_2^2-\sum_{t=1}^{T-1}\alpha_t\|\theta_t-\phi_{T-1}\|_2^2\ge0$$
	so it suffices to show
	$$\phi_{T-1}=\argmin_{\theta_T}\sum_{t=1}^T\alpha_t\|\theta_t-\phi_{t-1}\|_2^2-\sum_{t=1}^T\alpha_t\|\theta_t-\phi_T\|_2^2$$
	in which case $\phi_T=\phi_{T-1}$ and both added terms are zero, preserving the inequality.
	The gradient and Hessian are
	$$2\alpha_T(\theta_T-\phi_{T-1})+\frac{2\alpha_T}{\alpha_{1:T}}\sum_{t=1}^{T-1}\alpha_t(\theta_t-\phi_T)-2\alpha_T(\theta_T-\phi_T)\left(1-\frac{\alpha_T}{\alpha_{1:T}}\right)$$
	$$2\alpha_T\left(1-\frac{\alpha_T\alpha_{1:T-1}}{\alpha_{1:T}^2}-1+\frac{2\alpha_T}{\alpha_{1:T}}-\frac{\alpha_T^2}{\alpha_{1:T}^2}\right)I=\frac{2\alpha_T^2}{\alpha_{1:T}}I\succeq0$$
	so the problem is strongly convex and thus has a unique global minimum.
	Setting the gradient to zero yields
	$$0=\theta_T-\phi_{T-1}+\frac{1}{\alpha_{1:T}}\sum_{t=1}^{T-1}\alpha_t\theta_t-\frac{1}{\alpha_{1:T}}\sum_{t=1}^{T-1}\alpha_t\phi_T-\theta_T+\frac{\alpha_T}{\alpha_{1:T}}\theta_T+\phi_T-\frac{\alpha_T}{\alpha_{1:T}}\phi_T=\phi_T-\phi_{T-1}\implies\theta_T=\phi_{T-1}$$
\end{proof}
We use this to show logarithmic regret of FTL when the loss functions are Bregman regularizers with changing first arguments.
Note that such functions are in general only strictly convex, so the bounds from Theorem~\ref{thm:ftlaogd} cannot be applied directly.
\begin{Lem}\label{lem:ftl}
	Let $\Breg_R$ be a Bregman regularizer on $S$ w.r.t. $\|\cdot\|$ and consider any $\theta_1,\dots,\theta_T\in S$.
	Then for loss sequence $\alpha_1\Breg_R(\theta_1||\cdot),\dots,\alpha_T\Breg_R(\theta_T||\cdot)$ for any positive scalars $\alpha_1,\dots,\alpha_T\in\mathbb{R}_+$ we have regret bound
	$$\R_T\le\frac{G_R^2+1}{2}\sum_{t=1}^T\frac{\alpha_t}{\alpha_{1:t}}$$
	where $G_R$ is the Lipschitz constant of the Bregman regularizer $\Breg_R(\theta_t||\cdot)$ for any $t\in[T]$ on $S$ w.r.t. the Euclidean norm.
\end{Lem}
\begin{proof}
	Defining $\bar\phi=\frac{1}{\alpha_{1:T}}\sum_{t=1}^T\alpha_t\theta_t$, we apply Claim~\ref{clm:mean} and Lemma~\ref{lem:quad} to get
	\begin{align*}
	\R_T
	&=\sum_{t=1}^T\alpha_t\Breg_R(\theta_t||\phi_t)-\min_{\phi\in S}\sum_{t=1}^T\alpha_t\Breg_R(\theta_t||\phi)\\
	&\le\sum_{t=1}^T\alpha_t\Breg_R(\theta_t||\phi_t)-\sum_{t=1}^T\alpha_t\Breg_R(\theta_t||\bar\phi)+\frac12\sum_{t=1}^T\alpha_t\|\theta_t-\phi_t\|_2^2-\frac12\sum_{t=1}^T\alpha_t\|\theta_t-\bar\phi\|_2^2\\
	&=\sum_{t=1}^T\alpha_t\Breg_R(\theta_t||\phi_t)+\frac{\alpha_t}{2}\|\theta_t-\phi_t\|_2^2-\min_{\phi\in S}\sum_{t=1}^T\alpha_t\Breg_R(\theta_t||\phi)+\frac{\alpha_t}{2}\|\theta_t-\phi\|_2^2
	\end{align*}
	Since Bregman regularizers are convex in the second argument, the above is the regret of playing FTL on a sequence of $a_t$-strongly-convex losses.
	Applying \citet[Theorem~2]{kakade:08} yields the result.
\end{proof}

\newpage
The following result is our main theorem; 
Theorems~\ref{thm:fml} and~\ref{thm:fliadv} will follow as corollaries.

\begin{Thm}\label{thm:upper}
	In Setting~\ref{set:general}, Algorithm~\ref{alg:fmrl} has TAR bounded as
	\begin{align*}
		\TAR
		&\le\frac1T\left((2\kappa D+\varepsilon)\sigma_1+\frac{\kappa C}{\rho D^\ast}\sum_{t=1}^T\frac{\sigma_t^2}{\sigma_{1:t}}+\kappa\left(\frac{\nu \bar D^2}{\rho D^\ast}+\gamma(\rho D^\ast+\mathcal E)+\varepsilon\right)\sigma_{1:T}\right)\\
		&\qquad+\frac1T\left(\frac{\Delta_{1:T}^\ast}\varepsilon+\frac{\kappa\Delta'}{\rho D^\ast}+\sum_{k=0}^{\lfloor\log_\gamma\frac{\rho D^\ast+\mathcal E}\varepsilon\rfloor}\left(\frac{\kappa(\rho D^\ast+\mathcal E)}{\gamma^k\varepsilon}+\gamma^k\varepsilon\right)\sigma_{t_k}\right)
	\end{align*}
	for $C=\frac{{G'}^2}2$ in the Nice case or otherwise $C=\frac{C'D'(G'+1)}2$, $\rho=1$ in the Exact case or $\rho=2\sqrt\beta$ in the Approx case, and $\mathcal E=2\sqrt{2\beta\Delta_{\max}}$.
\end{Thm}
\begin{proof}
	We first use the $\beta$-strong-smoothness of $R$ to provide a bound in the Approx setting of the distance from the initialization to the update parameter at each time $t\in[T]$:
	\begin{align*}
		\Breg_R(\hat\theta_t||\phi_t)
		\le\frac\beta2\|\hat\theta_t-\phi_t\|^2
		&\le\beta\left(\|\hat\theta_t-\theta_t'\|^2+\|\theta_t'-\phi_t\|^2\right)\\
		&\le\beta\left(\|\hat\theta_t-\theta_t'\|^2+\max_{s<t}2\|\theta_t'-\theta_s'\|^2+2\|\theta_s'-\hat\theta_s\|^2\right)\\
		&\le4\beta{D^\ast}^2+4\beta\max_t\|\theta_t'-\hat\theta_t\|^2\\
		&\le4\beta{D^\ast}^2+8\beta\Delta_{\max}
	\end{align*}
	Combining this bound with the Exact setting assumption yields $\Breg_R(\hat\theta_t||\phi_t)\le\rho^2{D^\ast}^2+8\beta\Delta_{\max}\le\rho^2{D^\ast}^2+\mathcal E^2~\forall~t\in[T]$.	
	We now turn to analyzing the regret by defining two ``cheating" sequences: $\tilde\phi_t=\phi_t$ on all $t$ except $t=1$, when we set $\tilde\phi_1=\theta_1^\ast$;
	similarly, $\tilde D_t=D_t$ on all $t$ except $t=1$ and any $t$ s.t. $\Breg_R(\hat\theta_t||\phi_t)>D_t^2$, when we set $\tilde D_t=\rho D^\ast+\mathcal E$.
	In order to do this we add outside of the summation the corresponding regret of the true sequences whenever one of them is not the same as its ``cheating" sequence.
	Note that by this definition all upper bounds of $\Breg_R(\hat\theta_t||\phi_t)$ also upper bound $\Breg_R(\hat\theta_t||\tilde\phi_t)$.
	Furthermore the times $t$ s.t. $\Breg_R(\theta_t^\ast||\phi_t)>D_t^2$ corresponds exactly to the times that the violation count $k$ is incremented in Algorithm~\ref{alg:fmrl} and thus this occurs at most $\log_\gamma\frac{\rho D^\ast+\mathcal E}\varepsilon$ times, as we multiply the diameter guess by $\gamma$ each time it happens, which together with Lemma~\ref{clm:hull} ensures that $\phi_t$ remains within $\max\{\gamma(\rho D^\ast+\mathcal E),\varepsilon\}$ of all the reference parameters $\theta_t'$.
	We index these times by $k=0,\dots$, so that at each $k$ the agent uses $\eta_{t_k}$ set using $\gamma^k\varepsilon$.
	\begin{align*}
		\TAR T
		&=\sum_{t=1}^T\frac{\Breg_R(\theta_t^\ast||\phi_t)}{\eta_t}+\eta_tG_t^2m_t\\
		&\le\frac{\Delta_{1:T}^\ast}\varepsilon+\sum_{t=1}^T\left(\frac{\kappa\Breg_R(\hat\theta_t||\phi_t)}{D_t}+D_t\right)\sigma_t\qquad\textrm{(substitute $\eta_t=\frac{D_t}{G_t\sqrt{m_t}}$ and $D_t\ge\varepsilon$)}\\
		&\le\left(\frac{\kappa\Breg_R(\hat\theta_1||\phi_1)}{D_1}+D_1\right)\sigma_1+\frac{\Delta_{1:T}^\ast}\varepsilon\qquad\textrm{(substitute cheating sequence)}\\
		&\qquad+\sum_{t=1}^T\left(\frac{\kappa\Breg_R(\hat\theta_t||\tilde\phi_t)}{\tilde D_t}+\tilde D_t\right)\sigma_t+\sum_{k=0}^{\lfloor\log_\gamma\frac{\rho D^\ast+\mathcal E}\varepsilon\rfloor}\left(\frac{\kappa\Breg_R(\hat\theta_{t_k}||\phi_{t_k})}{\gamma^k\varepsilon}+\gamma^k\varepsilon\right)\sigma_{t_k}\\
		&\le((\kappa+1)D+\varepsilon)\sigma_1+\frac{\Delta_{1:T}^\ast}\varepsilon+\kappa\sum_{t=1}^T\left(\frac{\Breg_R(\hat\theta_t||\tilde\phi_t)}{\tilde D_t}+\tilde D_t\right)\sigma_t+\sum_{k=0}^{\lfloor\log_\gamma\frac{\rho D^\ast+\mathcal E}\varepsilon\rfloor}\left(\frac{\kappa(\rho D^\ast+\mathcal E)}{\gamma^k\varepsilon}+\gamma^k\varepsilon\right)\sigma_{t_k}\\
	\end{align*}
	
	\newpage
	We now bound the third term.
	For any $t\in[T]$ define $B_t^2=\Breg_R(\hat\theta_t||\tilde\phi_t)$ and $f_t(x)=\frac{B_t^2}x+x$.
	Its derivative $\partial_xf_t=1-\frac{B_t^2}{x^2}$ is nonnegative on $x\ge B_t$.
	Thus when $\tilde D_t\le\rho D^\ast+\mathcal E$ we have $f(\tilde D_t)\le f(\rho D^\ast+\mathcal E)$, as by definition both are greater than $B_t$ and so $f_t$ is increasing on the interval between them.
	On the other hand, for $\tilde D_t\ge\rho D^\ast+\mathcal E$, either $\tilde D_t\le\gamma(\rho D^\ast+\mathcal E)$ by the tuning rule or, if we initialized $\varepsilon>\rho D^\ast+\mathcal E$, then $\tilde D_t=\varepsilon~\forall~t\in[T]$, so either way we have $f_t(\tilde D_t)\le\frac{B_t^2}{\rho D^\ast}+\max\{\gamma(\rho D^\ast+\mathcal E),\varepsilon\}~\forall~t\in[T]$.
	Since $\gamma>1$ this bounds $f(\tilde D_t)$ in the previous case $\tilde D_t\le\rho D^\ast+\mathcal E$ as well, so we have
	\begin{align*}
		\TAR T
		&\le((1+\kappa) D+\varepsilon)\sigma_1+\frac{\Delta_{1:T}^\ast}\varepsilon+\kappa\sum_{t=1}^T\left(\frac{\Breg_R(\hat\theta_t||\tilde\phi_t)}{\tilde D_t}+\tilde D_t\right)\sigma_t+\sum_{k=0}^{\lfloor\log_\gamma\frac{\rho D^\ast+\mathcal E}\varepsilon\rfloor}\left(\frac{\kappa(\rho D^\ast+\mathcal E)}{\gamma^k\varepsilon}+\gamma^k\varepsilon\right)\sigma_{t_k}\\
		&\le(2\kappa D+\varepsilon)\sigma_1+\frac{\Delta_{1:T}^\ast}\varepsilon+\kappa\sum_{t=1}^T\left(\frac{\Breg_R(\hat\theta_t||\tilde\phi_t)}{\rho D^\ast}+\gamma(\rho D^\ast+\mathcal E)+\varepsilon\right)\sigma_t+\sum_{k=0}^{\lfloor\log_\gamma\frac{\rho D^\ast+\mathcal E}\varepsilon\rfloor}\left(\frac{\kappa(\rho D^\ast+\mathcal E)}{\gamma^k\varepsilon}+\gamma^k\varepsilon\right)\sigma_{t_k}\\
		&\le(2\kappa D+\varepsilon)\sigma_1+\frac{\Delta_{1:T}^\ast}\varepsilon+\frac\kappa{\rho D^\ast}\sum_{t=1}^T\left(\Breg_R(\hat\theta_t||\tilde\phi_t)-\Breg_R(\hat\theta_t||\hat\phi)\right)\\
		&\qquad+\kappa\sum_{t=1}^T\left(\frac{\Breg_R(\hat\theta_t||\hat\phi)}{\rho D^\ast}+\gamma(\rho D^\ast+\mathcal E)+\varepsilon\right)\sigma_t+\sum_{k=0}^{\lfloor\log_\gamma\frac{\rho D^\ast+\mathcal E}\varepsilon\rfloor}\left(\frac{\kappa(\rho D^\ast+\mathcal E)}{\gamma^k\varepsilon}+\gamma^k\varepsilon\right)\sigma_{t_k}\\
		&\le(2\kappa D+\varepsilon)\sigma_1+\frac{\Delta_{1:T}^\ast}\varepsilon+\frac{\kappa C}{\rho D^\ast}\sum_{t=1}^T\frac{\sigma_t^2}{\sigma_{1:t}}+\frac{\kappa\Delta'}{\rho D^\ast}\qquad\textrm{(Thm.~\ref{thm:ftlaogd} and Lem.~\ref{lem:ftl})}\\
		&\qquad+\kappa\sum_{t=1}^T\left(\frac{\nu\Breg_R(\theta_t'||\phi')}{\rho D^\ast}+\gamma(\rho D^\ast+\mathcal E)+\varepsilon\right)\sigma_t+\sum_{k=0}^{\lfloor\log_\gamma\frac{\kappa(\rho D^\ast+\mathcal E)}\varepsilon\rfloor}\left(\frac{\kappa(\rho D^\ast+\mathcal E)}{\gamma^k\varepsilon}+\gamma^k\varepsilon\right)\sigma_{t_k}\\
		&=(2\kappa D+\varepsilon)\sigma_1+\frac{\Delta_{1:T}^\ast}\varepsilon+\frac{\kappa C}{\rho D^\ast}\sum_{t=1}^T\frac{\sigma_t^2}{\sigma_{1:t}}+\frac{\kappa\Delta'}{\rho D^\ast}\\
		&\qquad+\kappa\left(\frac{\nu \bar D^2}{\rho D^\ast}+\gamma(\rho D^\ast+\mathcal E)+\varepsilon\right)\sigma_{1:T}+\sum_{k=0}^{\lfloor\log_\gamma\frac{\rho D^\ast+\mathcal E}\varepsilon\rfloor}\left(\frac{\kappa(\rho D^\ast+\mathcal E)}{\gamma^k\varepsilon}+\gamma^k\varepsilon\right)\sigma_{t_k}
	\end{align*}
\end{proof}

The following result corresponds to the general case of Theorem~\ref{thm:fml}.
\begin{Cor}\label{cor:exact}
	In the Exact case of Setting~\ref{set:general}, if $G_t=G,m_t=m~\forall~t\in[T]$, the FAL variant of Algorithm~\ref{alg:fmrl} has TAR
	$$\TAR\le\left(\frac{2D+2\varepsilon+\frac C{D^\ast}(1+\log T)+\frac\gamma{\gamma-1}\left(\frac{{D^\ast}^2}\varepsilon+D^\ast\right)}T+\frac{\bar D^2}{D^\ast}+\gamma D^\ast+\varepsilon\right)G\sqrt m$$
	If we assume known $D$, picking $\varepsilon=D\frac{1+\log T}T$ and $\gamma=\frac{1+\log T}{\log T}$ yields
	$$\TAR\le\left(\left(6D+\frac C{D^\ast}\right)\frac{1+\log T}T+\frac92D^\ast\right)G\sqrt m$$
\end{Cor}
\begin{proof}
	For $K=\lfloor\log_\gamma\frac{D^\ast}\varepsilon\rfloor$ we have
	$$\sum_{k=0}^K\left(\frac{{D^\ast}^2}{\gamma^k\varepsilon}+\gamma^k\varepsilon\right)
	=\frac{(\gamma^{K+1}-1)({D^\ast}^2+\gamma^K\varepsilon^2)}{\gamma^K(\gamma-1)\varepsilon}
	\le\frac\gamma{\gamma-1}\left(\frac{{D^\ast}^2}\varepsilon+D^\ast\right)$$
	The result follows by noting that in the exact case we have $\kappa=\nu=\rho=1,\Delta_{1:T}^\ast=\Delta'=\Delta_{\max}=0$, and substituting $\sum_{t=1}^T\frac1t\le(1+\log T)$.
\end{proof}

\subsection{Lower Bound}\label{subsec:lower}

The following lower bound, which extends Theorem~4.2 of \citet{abernethy:08} to the multi-task setting, shows that the previous TAR guarantees are optimal up to a constant multiplicative factor.
Note that while the result is stated in terms of the task divergence $D^\ast$, since $D^\ast\ge \bar D$ the same lower bound holds for the average task deviation as well.

\begin{Thm}\label{thm:app:lower}
	Suppose the action space is $\Theta\subset\mathbb{R}^d$ for $d\ge3$ and for each task $t\in[T]$ an adversary must play a a sequence of $m_t$ convex $G_t$-Lipschitz functions $\ell_{t,i}:\Theta\mapsto\mathbb R$ whose optimal actions in hindsight $\argmin_{\theta\in\Theta}\sum_{i=1}^{m_t}\ell_{t,i}(\theta)$ are contained in some fixed $\ell_2$-ball $\Theta^\ast\subset\Theta$ with center $\phi^\ast$ and diameter $D^\ast$.
	Then the adversary can force the agent to have task-averaged regret at least $\frac{D^\ast}{4T}\sum_{t=1}^TG_t\sqrt{m_t}$.
\end{Thm}
\begin{proof}
	Let $\{\theta_{t,i}\}_{i=1}^m$ be the sequence of actions of the agent on task~$t$.
	Define $c(\theta)=\frac{G_t}2\max\{0,\|\theta-\phi^\ast\|_2-D^\ast\}$, which is 0 on $\Theta^\ast$ and an upward-facing cone with vertex $\left(\phi^\ast,-\frac{G_tD^\ast}2\right)$ and slope $\frac{G_t}2$ on the complement.
	The strategy of the adversary at round $i$ of task $t$ will be to play $\ell_{t,i}(\theta)=\langle\nabla_{t,i},\theta-\phi^\ast\rangle+c(\theta)$, where $\nabla_{t,i}$ satisfies $\|\nabla_{t,i}\|_2=\frac{G_t}2$, $\langle\nabla_{t,i},\theta_{t,i}-\phi^\ast\rangle=0$, and $\langle\nabla_{t,i},\nabla_{t,1:i-1}\rangle=0$.
	Such a $\nabla_{t,i}$ always exists for $d\ge3$.
	Note that these conditions imply that along any direction from $\phi^\ast$ the total loss $\sum_{i=1}^{m_t}\ell_{t,i}(\theta)$ is increasing outside $\Theta^\ast$ and so is minimized inside $\Theta^\ast$, so we have
	$$\min_{\theta\in\Theta}\sum_{i=1}^{m_t}\ell_{t,i}(\theta)
	=\min_{\theta\in\Theta^\ast}\sum_{i=1}^{m_t}\langle\nabla_{t,i},\theta-\phi^\ast\rangle
	=\min_{\|\theta-\phi^\ast\|_2\le \frac{D^\ast}2}\langle\theta-\phi^\ast,\nabla_{t,1:m_t}\rangle
	=-\frac{D^\ast}2\|\nabla_{t,1:m_t}\|_2$$
	Note that the condition $\langle\nabla_{t,i},\theta_{t,i}-\phi^\ast\rangle=0$ and the nonnegativity of $c(\theta)$ implies that the loss of the agent is at least 0, and so the agent's regret on task $t$ satisfies $\R_{m_t}\ge\frac{D^\ast}2\|\nabla_{t,1:m_t}\|_2$.
	By the condition $\langle\nabla_{t,i},\nabla_{t,1:i-1}\rangle=0$ we have that
	$$\|\nabla_{t,1:i}\|_2^2
	=\|\nabla_{t,i}+\nabla_{t,1:i-1}\|_2^2
	=\|\nabla_{t,i}\|_2^2+\|\nabla_{t,1:i-1}\|_2^2
	=\frac{G_t^2}4+\|\nabla_{t,i-1}\|_2^2
	$$
	and so by induction on $i$ with base case $\|\nabla_{t,1}\|_2=\frac{G_t}2$ we have $\|\nabla_{t,1:m_t}\|_2=\frac{G_t}2\sqrt{m_t}\implies\R_{m_t}\ge\frac{G_tD^\ast}4\sqrt{m_t}$.
	Substituting the regret on each task into $\TAR=\frac1T\sum_{t=1}^T\R_{m_t}$ completes the proof.
\end{proof}

\newpage
\subsection{Task-Averaged Regret for Approximate Meta-Updates}

For the Approx variants of FMRL we need a bound on the distance between the last or average iterate of FTRL/OMD and the best parameter in hindsight.
This necessitates further assumptions on the loss functions besides convexity, as a task may otherwise have functions with very small losses, even far away from the optimal parameter, in which case the last iterate of FTRL/OMD will be far away if the initial point is far away from the optimum.
Here we make use of the $\alpha$-QG assumption on the average loss functions to obtain stability of the estimates w.r.t. the true loss.
%
%

\begin{Lem}\label{lem:advstab}
	Let $\ell_1,\dots,\ell_m$ be a sequence of convex losses on $\Theta$ with $L(\theta)=\frac1m\sum_{i=1}^m\ell(\theta)$ being $\alpha$-QG w.r.t. $\|\cdot\|$ and define $\hat\theta=\argmin_{\theta\in\Theta}\Breg_R(\theta||\phi)+\eta mL(\theta)$ to be the last iterate of running $\FTRL_{\eta,\phi}^{(R)}$ for $\eta>0,\phi\in\Theta$, and $R:\Theta\mapsto\mathbb R$ 1-strongly-convex w.r.t. $\|\cdot\|$.
	Then the closest minimum $\theta^\ast\in\Theta$ of $L$ to $\hat\theta$ satisfies
	$$\frac12\|\theta^\ast-\hat\theta\|^2\le\frac{\Breg_R(\theta^\ast||\phi)-\Breg_R(\hat\theta||\phi)}{\alpha\eta m}$$
\end{Lem}
\begin{proof}
	We have by definition of $\theta'$ and $\hat\theta$ that
	$$\Breg_R(\theta^\ast||\phi)+\eta mL(\theta^\ast)
	\ge\Breg_R(\hat\theta||\phi)+\eta mL(\hat\theta)$$
	On the other hand since $L$ is $\alpha$-QG we have that
	$$L(\hat\theta)\ge L(\theta^\ast)+\frac\alpha2 \|\theta^\ast-\hat\theta\|^2$$
	Multiplying the second inequality by $\eta m$ and adding it to the first yields the result.
\end{proof}

\begin{Prp}\label{prp:fliadv}
	In Setting~\ref{set:general} , if for each task $t\in[T]$ the losses $\ell_{t,1},\dots,\ell_{t,m_t}$ satisfy the $\alpha$-QG condition as in Lemma~\ref{lem:advstab} and $\varepsilon\ge\max_t\frac{4\beta G_t}{\alpha\sqrt{m_t}}$, then for $\hat\theta_t$ set according to the FLI-Online algorithm and $\theta_t^\ast=\theta_t'~\forall~t\in[T]$ we have
	$$\kappa=4\beta,
	\qquad\Delta_t^\ast=0~\forall~t\in[T],
	\qquad\nu=3\beta,
	\qquad\Delta'=\frac{6\beta D^2}{\alpha\varepsilon}\sum_{t=1}^T\frac{G_t\sigma_t}{\sqrt{m_t}},
	\qquad\Delta_{\max}=\max_t\frac{D^2G_t}{\alpha\varepsilon\sqrt{m_t}}$$
\end{Prp}
\begin{proof}
	Applying the triangle inequality, Jensen's inequality, and Lemma~\ref{lem:advstab} yields the first two values:
	$$\|\theta_t^\ast-\phi_t\|^2
	\le2\|\theta_t^\ast-\hat\theta_t\|^2+2\|\hat\theta_t-\phi_t\|^2
	\le\frac{4\Breg_R(\theta_t^\ast||\phi_t)}{\alpha\eta_tm_t}+4\Breg_R(\hat\theta_t||\phi_t)
	\le\frac{2\beta}{\alpha\eta_tm_t}\|\theta_t^\ast-\phi_t\|^2+4\Breg_R(\hat\theta_t||\phi_t)$$
	$$\implies\Breg_R(\theta_t^\ast||\phi_t)
	\le\frac\beta2\|\theta_t^\ast-\phi_t\|^2
	\le\frac{2\beta\Breg_R(\hat\theta_t||\phi_t)}{1-\frac{2\beta}{\alpha\eta_tm_t}}
	\le4\beta\Breg_R(\hat\theta_t||\phi_t)$$
	Here in the last step we used the fact that $\varepsilon\ge\frac{4\beta G_t}{\alpha\sqrt{m_t}}\implies\eta_t\ge\frac{4\beta}{\alpha m_t}~\forall~t\in[T]$.
	For the next two values, noting that for FLI-Online, $\theta_t^\ast=\theta_t'~\forall~t\in[T]$ we have by the triangle inequality and Titu's lemma that
	$$\|\phi'-\hat\phi\|^2
	=\frac1{(\sigma_{1:T})^2}\left\|\sum_{t=1}^T\sigma_t\theta_t'-\sum_{t=1}^T\sigma_t\hat\theta_t\right\|^2
	\le\frac1{(\sigma_{1:T})^2}\left(\sum_{t=1}^T\sigma_t\|\theta_t'-\hat\theta_t\|\right)^2
	\le\frac1{\sigma_{1:T}}\sum_{t=1}^T\sigma_t\|\theta_t'-\hat\theta_t\|^2$$
	Therefore since $\eta\ge\frac\varepsilon{\sigma_t}$ and $\Breg_R(\theta_t^\ast||\phi_t)\le D^2$ we have that
	$$\sum_{t=1}^T\sigma_t\Breg_R(\hat\theta_t||\hat\phi)
	\le\frac{3\beta}2\sum_{t=1}^T\sigma_t(\|\hat\theta_t-\theta_t'\|^2+\|\theta_t'-\phi'\|^2+\|\phi'-\hat\phi\|^2)
	\le3\beta\sum_{t=1}^T\sigma_t\left(\frac{2\Breg_R(\theta_t^\ast||\phi_t)}{\alpha\eta_tm_t}+\Breg_R(\theta_t'||\phi')\right)$$
	The last value follows directly by Lemma~\ref{lem:advstab}, $\eta_t\ge\frac\varepsilon{\sigma_t}$, and the bound $D^2$ on the maximum Bregman divergence.
\end{proof}

The following upper bound yields Theorem~\ref{thm:fliadv}:
\begin{Cor}
	In the Approx. case of Setting~\ref{set:general}, if $G_t=G,m_t=m~\forall~t\in[T],\gamma=\frac{1+\log T}{\log T}$, and $\varepsilon=\frac{4\beta G}{\alpha\sqrt[6]m}+D\frac{1+\log T}T$ then the FLI-Online variant of Algorithm~\ref{alg:fmrl} has TAR
	$$\TAR=\mathcal O\left(\frac D{D^\ast}\left(\frac{\log T}T+\frac1{\sqrt[6]m}\right)+D^\ast\right)G\sqrt m$$
\end{Cor}
\begin{proof}
	Substitute Proposition~\ref{prp:fliadv} into Theorem~\ref{thm:upper} and simplify.
\end{proof}

\begin{Lem}\label{lem:stonear}
	Let $\ell_1,\dots,\ell_m:\Theta\mapsto[0,1]$ be a sequence of convex losses on $\Theta$ drawn i.i.d. from some distribution $\mathcal D$ with risk $\E_{\ell\sim\mathcal D}\ell$ being $\alpha$-QG w.r.t. $\|\cdot\|$ and let $\theta^\ast\in\argmin_{\theta\in\Theta}\sum_{i=1}^m\ell_i(\theta)$ be any of the optimal actions in hindsight.
	Then w.p. $1-\delta$ the closest minimum $\theta'\in\Theta$ of $\E_{\ell\sim\mathcal D}\ell$ to $\theta^\ast$ satisfies
	$$\frac12\|\theta^\ast-\theta'\|^2\le\sqrt{\frac8{\alpha^2m}\log\frac2\delta}$$
\end{Lem}
\begin{proof}
	By definition of $\theta^\ast$ and $\theta'$ we have w.p. $1-\delta$ that
	\begin{align*}
	\frac\alpha2\|\theta^\ast-\theta'\|^2
	&\le\frac1m\E_{\{\ell_i\}\sim\mathcal D^m}\sum_{i=1}^m\ell_i(\theta^\ast)-\frac1m\E_{\{\ell_i\}\sim\mathcal D^m}\sum_{i=1}^m\ell_i(\theta')\qquad\textrm{(apply $\alpha$-QG)}\\
	&\le\frac1m\sum_{i=1}^m\ell_i(\theta^\ast)-\frac1m\sum_{i=1}^m\ell_i(\theta')+\sqrt{\frac8m\log\frac2\delta}\qquad\textrm{(apply Prp.~\ref{prp:o2b} twice)}\\
	&\le\sqrt{\frac8m\log\frac2\delta}\qquad\textrm{(definition of $\theta^\ast$)}
	\end{align*}
\end{proof}

\begin{Lem}\label{lem:stoagg}
	Suppose $\forall~t\in[T]$ the r.v. $Q_t$ satisfies $0\le Q_t\le B$ a.s. and $Q_t\le\sqrt{\frac8{m_t}\log\frac2\delta}$ w.p. $1-\delta$ for any $\delta\in(0,1)$.
	Then for nonnegative $\alpha_1,\dots,\alpha_T$ we have w.p. $1-\gamma$ for any $\gamma\in(0,1)$ that
	$$\sum_{t=1}^T\alpha_tQ_t\le\frac{2B\alpha_{\max}}3\log\frac1\gamma+2\sum_{t=1}^T\alpha_t\sqrt{\frac{1+4\log(Bm_t)}{m_t}\left(1+\log\frac1\gamma\right)}$$
\end{Lem}
\begin{proof}
	Define convenience coefficients $\beta_t=\frac{\alpha_t}{\alpha_{1:T}}$, the auxiliary sequence $Z_t=\beta_tQ_t~\forall~t\in[T]$, the martingale sequence $Y_0=0,Y_t=Z_{1:t}-\E Z_{1:t}~\forall~t\in[T]$ and the  associated martingale difference sequence $X_t=Y_t-Y_{t-1}~\forall~t\in[T]$.
	By substituting $\delta=\frac2{Bm_t}$ we then have
	$$\mathbb E_{t-1}X_t^2
	=\mathbb E_{t-1}(Y_t-Y_{t-1})^2
	=\beta_t^2\E(Q_t-\E Q_t)^2
	\le\beta_t^2\E Q_t^2
	\le\beta_t^2\left(\frac8{m_t}\log\frac2\delta+\delta B\right)
	\le\frac{2+8\log(Bm_t)}{m_t}\beta_t^2$$
	Note further that using $\delta=\frac2{\sqrt{Bm_t}}$ and Jensen's inequality we have
	$$\E Q_t
	\le\sqrt{\frac8{m_t}\log\frac2\delta}+\delta B
	\le\sqrt{\frac{4+8\log(Bm_t)}{m_t}}$$
	Noting that $Q_t\le B$ a.s. $\implies X_t\le B$ a.s., we have by Freedman's inequality \citep[Theorem~1.6]{freedman:75} that
	$$\mathbb P\left(\sum_{t=1}^T\beta_tQ_t\ge\tau+2\sum_{t=1}^T\beta_t\sqrt{\frac{1+2\log(Bm_t)}{m_t}}\right)
	\le\mathbb P\left(\sum_{t=1}^T\beta_tQ_t\ge\tau+\sum_{t=1}^T\beta_t\E Q_t\right)
	\le\exp\left(-\frac{\tau^2}{2\sigma^2+\frac{2B\beta_{\max}}3\tau}\right)$$
	for $\tau\ge0,\sigma^2=\sum_{t=1}^T\frac{2+8\log(Bm_t)}{m_t}\beta_t^2$.
	Substituting
	$\tau=\frac{2\beta_{\max}}3\log\frac1\gamma+\sqrt{2\sigma^2\log\frac1\gamma}$ yields
	$$\mathbb P\left(\sum_{t=1}^T\beta_tQ_t\ge\frac{2B\beta_{\max}}3\log\frac1\gamma+2\sum_{t=1}^T\beta_t\sqrt{\frac{1+2\log(Bm_t)}{m_t}}+\sqrt{2\log\frac1\gamma\sum_{t=1}^T\frac{2+8\log(Bm_t)}{m_t}\beta_t^2}\right)\le\gamma$$
\end{proof}

\begin{Prp}\label{prp:flisto}
	In Setting~\ref{set:general}, if for each task $t\in[T]$ the losses $\ell_{t,1},\dots,\ell_{t,m_t}$ and reference parameter $\theta_t'$ satisfy the $\alpha$-QG condition as in Lemma~\ref{lem:stonear}, then for $\hat\theta_t=\theta_t^\ast$ set according to the FAL algorithm we have w.p. $1-\delta$ that $\kappa=1,\nu=3\beta$,
	$$\Delta_t^\ast=0~\forall~t\in[T],
	\quad\Delta'=\frac{4\beta}\alpha\left(\sigma_{\max}\log\frac2\delta+3\sum_{t=1}^T\sigma_t\sqrt{\frac{1+4\log m_t}{m_t}\left(1+\log\frac2\delta\right)}\right),
	\quad\Delta_{\max}=\frac4\alpha\sqrt{\frac1{m_{\min}}\log\frac{2T}\delta}$$
\end{Prp}
\begin{proof}
	$\kappa=1$ and $\Delta_t^\ast=0~\forall~t\in[T]$ because $\hat\theta_t=\theta_t^\ast~\forall~t\in[T]$.
	Applying Titu's lemma as in the proof of Proposition~\ref{prp:fliadv} yields the values of $\nu$ and $\Delta'$ w.p. $1-2\delta$:
	\begin{align*}
	\sum_{t=1}^T\sigma_t\Breg_R(\hat\theta_t||\hat\phi)
	&\le\frac{3\beta}2\sum_{t=1}^T\sigma_t(\|\theta_t^\ast-\theta_t'\|^2+\|\theta_t'-\phi'\|^2+\|\phi'-\hat\phi\|^2)\\
	&\le3\beta\sum_{t=1}^T\sigma_t\left(\|\theta_t^\ast-\theta_t'\|^2+\Breg_R(\theta_t'||\phi')\right)\\
	&\le\frac{4\beta\sigma_{\max}}\alpha\log\frac1\delta+\frac{12\beta}\alpha\sum_{t=1}^T\sigma_t\sqrt{\frac{1+4\log m_t}{m_t}\left(1+\log\frac2\delta\right)}+3\beta\sum_{t=1}^T\sigma_t\Breg_R(\theta_t'||\phi')
	\end{align*}
	Here in the last step we applied Lemma~\ref{lem:stoagg} on $Q_t=\frac\alpha2\|\theta_t^\ast-\theta_t'\|^2$, which is 1-bounded by Lemma~\ref{lem:stonear}.
	The value of $\Delta_{\max}$ follows directly by Lemma~\ref{lem:stonear} w.p. $1-2\delta$.
\end{proof}

The following upper bound yields the FAL result in Theorem~\ref{thm:fliadv}:
\begin{Cor}
	In the Approx. case of Setting~\ref{set:general}, if $G_t=G,m_t=m~\forall~t\in[T],\gamma=\frac{1+\log T}{\log T}$, and $\varepsilon=D\frac{1+\log T}T$ then the FAL variant of Algorithm~\ref{alg:fmrl} has TAR
	$$\TAR=\mathcal O\left(\frac D{D^\ast}\left(\frac{\log T}T+\sqrt{\frac1{\sqrt[3]m}\log\frac{Tm}\delta}\right)+D^\ast\right)G\sqrt m$$
\end{Cor}
\begin{proof}
	Substitute Proposition~\ref{prp:flisto} into Theorem~\ref{thm:upper} and simplify.
\end{proof}

\begin{Lem}\label{lem:stostab}
	Let $\ell_1,\dots,\ell_m:\Theta\mapsto[0,1]$ be a sequence of $G_i$-Lipschitz convex losses on $\Theta$ drawn i.i.d. from some distribution $\mathcal D$ with risk $\E_{\ell\sim\mathcal D}\ell$ being $\alpha$-QG w.r.t. $\|\cdot\|$ and define $\hat\theta=\frac1m\theta_{1:m}$ to be the the average iterate of running $\FTRL_{\eta,\phi}^{(R)}$ or $\OMD_{\eta,\phi}^{(R)}$ on $\ell_1,\dots,\ell_m$ for $\eta>0,\phi\in\Theta$, and $R:\Theta\mapsto\mathbb R$ 1-strongly convex w.r.t. $\|\cdot\|$.
	Then w.p. $1-\delta$ the closest minimum $\theta'\in\Theta$ of $\E_{\ell\sim\mathcal D}\ell$ to $\hat\theta$ satisfies
	$$\frac12\|\theta'-\hat\theta\|^2\le\frac{\Breg_R(\theta'||\phi)+\eta^2 G^2m+\eta\sqrt{8m\log\frac2\delta}}{\alpha\eta m}$$
	where $G^2=\frac1m\sum_{i=1}^mG_i^2$.
\end{Lem}
\begin{proof}
	By definition of $\hat\theta$ and $\theta'$ we have w.p. $1-\delta$ that
	\begin{align*}
	\frac\alpha2\|\hat\theta-\theta'\|^2
	&\le\frac1m\E_{\{\ell_i\}\sim\mathcal D^m}\sum_{i=1}^m\ell_i(\theta_i)-\frac1m\E_{\{\ell_i\}\sim\mathcal D^m}\sum_{i=1}^m\ell_i(\theta')\qquad\textrm{(apply $\alpha$-QG and Jensen's inequality)}\\
	&\le\frac1m\sum_{i=1}^m\ell_i(\theta_i)-\frac1m\sum_{i=1}^m\ell_i(\theta')+\sqrt{\frac8m\log\frac2\delta}\qquad\textrm{(apply Prp.~\ref{prp:o2b} twice)}\\
	&\le\frac{\frac1\eta\Breg_R(\theta'||\phi)+\eta G^2m}m+\sqrt{\frac8m\log\frac2\delta}\qquad\textrm{(substitute the regret of FTRL/OMD)}
	\end{align*}
\end{proof}

\newpage
\begin{Prp}\label{prp:approx}
	In Setting~\ref{set:general}, if for each task $t\in[T]$ the losses $\ell_{t,1},\dots,\ell_{t,m_t}$ and reference parameter $\theta_t'$ satisfy the $\alpha$-QG condition as in Lemma~\ref{lem:stostab} and $\varepsilon\ge\max_t\frac{24\beta G_t}{\alpha\sqrt{m_t}}$, then for $\hat\theta_t$ set according to the FLI-Batch algorithm we have w.p. $1-\delta$ that $\kappa=12\beta,\nu=3\beta$,
	$$\Delta_t^\ast=\frac{3\beta}\alpha\left(1+\frac{4\beta G_t}{\alpha\varepsilon}\right)\left(\frac{2\alpha_{\max}}{3\alpha_tT}\log\frac3\delta+2\sqrt{\frac{1+4\log m_t}{m_t}\left(1+\log\frac3\delta\right)}\right)+\frac{12\beta G_t^2(D+\varepsilon)}{\alpha m_t}~\forall~t\in[T]$$
	$$\Delta'=\frac{4\beta\sigma_{\max}}\alpha\log\frac3\delta+\frac{12\beta}\alpha\sum_{t=1}^T\left(\sqrt{\frac{1+4\log m_t}{m_t}\left(1+\log\frac3\delta\right)}+\frac{(D^2+\varepsilon)G_t}{2\varepsilon\sqrt{m_t}}\right)\sigma_t,
	\quad\Delta_{\max}=\frac1\alpha\sqrt{\frac8{m_{\min}}\log\frac{6T}\delta}$$
\end{Prp}
\begin{proof}
	Applying the triangle inequality, Jensen's inequality, Lemma~\ref{lem:stonear}, and Lemma~\ref{lem:stostab} yields w.p. $1-\delta$
	\begin{align*}
	\|\theta_t^\ast-\phi_t\|^2
	&\le3\|\theta_t^\ast-\theta_t'\|^2+3\|\theta_t'-\hat\theta_t\|^2+3\|\hat\theta_t-\phi_t\|^2\\
	&\le\frac3\alpha\sqrt{\frac8{m_t}\log\frac2\delta}+\frac6{\alpha\eta_tm_t}\left(\Breg_R(\theta_t'||\phi_t)+\eta_t^2G_t^2m_t+\eta_t\sqrt{8m_t\log\frac2\delta}\right)+6\Breg_R(\hat\theta_t||\phi_t)\\
	&\le\frac9\alpha\sqrt{\frac8{m_t}\log\frac2\delta}+\frac{12\beta}{\alpha\eta_tm_t}(\|\theta_t'-\theta_t^\ast\|^2+\|\theta_t^\ast-\phi_t\|^2)+\frac{6G_t^2(D+\varepsilon)}{\alpha m_t}+6\Breg_R(\hat\theta_t||\phi_t)\\
	&\le\frac3\alpha\left(1+\frac{4\beta G_t}{\alpha\varepsilon}\right)\sqrt{\frac8{m_t}\log\frac2\delta}+\frac{12\beta}{\alpha\eta_tm_t}\|\theta_t^\ast-\phi_t\|^2+\frac{6G_t^2(D+\varepsilon)}{\alpha m_t}+6\Breg_R(\hat\theta_t||\phi_t)
	\end{align*}
	where we have used the uniqueness of the reference parameter $\theta_t'$.
	The above implies
	$$\Breg_R(\theta_t^\ast||\phi_t)
	\le\frac\beta2\|\theta_t^\ast-\phi_t\|^2
	\le\frac{3\beta}\alpha\left(1+\frac{4\beta G_t}{\alpha\varepsilon}\right)\sqrt{\frac8{m_t}\log\frac2\delta}+\frac{12\beta G^2(D+\varepsilon)}{\alpha m_t}+12\beta\Breg_R(\hat\theta_t||\phi_t)$$
	Here in the last step we used the fact that $\varepsilon\ge\frac{24\beta G_t}{\alpha\sqrt{m_{\min}}}\implies\eta_t\ge\frac{24\beta}{\alpha m_t}~\forall~t\in[T]$.
	Thus by Lemma~\ref{lem:stoagg} w.p. $1-3\delta$
	\begin{align*}
	\sum_{t=1}^T\alpha_t\Breg_R(\theta_t^\ast||\phi_t)
	&\le\frac{3\beta}\alpha\left(1+\frac{4\beta G_t}{\alpha\varepsilon}\right)\left(\frac{2\alpha_{\max}}3\log\frac3\delta+2\sum_{t=1}^T\alpha_t\sqrt{\frac{1+4\log m_t}{m_t}\left(1+\log\frac3\delta\right)}\right)\\
	&\qquad+\frac{12\beta G_t^2(D+\varepsilon)}{\alpha}\sum_{t=1}^T\frac{\alpha_t}{m_t}+12\beta\sum_{t=1}^T\alpha_t\Breg_R(\hat\theta_t||\phi_t)
	\end{align*}
	This yields the values of $\kappa$ and $\Delta_t^\ast~\forall~t\in[T]$.
	We next have by applying Titu's lemma as in the proof of Proposition~\ref{prp:fliadv}
	\begin{align*}
	\sum_{t=1}^T\sigma_t\Breg_R(\hat\theta_t||\hat\phi)
	&\le3\beta\sum_{t=1}^T\sigma_t(\|\hat\theta_t-\theta_t'\|^2+\Breg_R(\theta_t'||\phi'))\\
	&\le\frac{6\beta}\alpha\sum_{t=1}^T\frac{\sigma_t\Breg_R(\theta_t'||\phi_t)}{\eta_tm_t}+\sigma_t\eta_tG^2+\sigma_t\sqrt{\frac8{m_t}\log\frac2\delta}+3\beta\sum_{t=1}^T\sigma_t\Breg_R(\theta_t'||\phi'))\\
	&\le\frac{4\beta\sigma_{\max}}\alpha\log\frac3\delta+\frac{12\beta}\alpha\sum_{t=1}^T\left(\sqrt{\frac{1+4\log m_t}{m_t}\left(1+\log\frac3\delta\right)}+\frac{(D^2+\varepsilon)G_t}{2\varepsilon\sqrt{m_t}}\right)\sigma_t+3\beta\sum_{t=1}^T\sigma_t\Breg_R(\theta_t'||\phi')
	\end{align*}
	This yields the values of $\nu$ and $\Delta'$.
	The value of $\Delta_{\max}$ follows directly by Lemma~\ref{lem:stonear} w.p. $1-3\delta$.
\end{proof}

The following final upper bound yields the FLI-Batch result in Theorem~\ref{thm:fliadv}:
\begin{Cor}
	In the Approx. case of Setting~\ref{set:general}, if $G_t=G,m_t=m~\forall~t\in[T],\gamma=\frac{1+\log T}{\log T}$, and $\varepsilon=\frac{24\beta G}{\alpha\sqrt m}+D\frac{1+\log T}T$ then the FLI-Batch variant of Algorithm~\ref{alg:fmrl} has TAR
	$$\TAR=\mathcal O\left(\frac D{D^\ast}\left(\frac{\log T}T+\sqrt{\frac1{\sqrt[3]m}\log\frac{Tm}\delta}\right)+D^\ast\right)G\sqrt m$$
\end{Cor}
\begin{proof}
	Substitute Proposition~\ref{prp:approx} into Theorem~\ref{thm:upper} and simplify.
\end{proof}

\subsection{Online-to-Batch Conversion for Task-Averaged Regret}

The following yields a bound on the expected transfer risk when randomizing over the output of any TAR-minimizing algorithm when in the setting of statistical LTL.

\begin{Thm}\label{thm:o2bexp}
	Let $\mathcal Q$ be a distribution over distributions $\mathcal P$ over convex loss functions $\ell:\Theta\mapsto[0,1]$.
	A sequence of sequences of loss functions $\{\ell_{t,i}\}_{t\in[T],i\in[m]}$ is generated by drawing $m$ loss functions i.i.d. from each in a sequence of distributions $\{\mathcal P_t\}_{t\in[T]}$ themselves drawn i.i.d. from $\mathcal Q$.
	If such a sequence is given to an meta-learning algorithm with task-averaged regret bound $\TAR$ that has states $\{s_t\}_{t\in[T]}$ at the beginning of each task $t$ then we have w.p. $1-\delta$ for any $\theta^\ast\in\Theta$ that
	$$\E_{t\sim\mathcal U[T]}\E_{\mathcal P\sim\mathcal Q}\E_{\mathcal P^m}\E_{\ell\sim\mathcal P}\ell(\bar\theta)\le\E_{\mathcal P\sim\mathcal Q}\E_{\ell\sim\mathcal P}\ell(\theta^\ast)+\frac\TAR m+\sqrt{\frac8T\log\frac1\delta}$$
	where $\bar\theta=\frac1m\theta_{1:m}$ is generated by randomly sampling $t\in\mathcal U[T]$, running the online algorithm with state $s_t$, and averaging the actions $\{\theta_i\}_{i\in[m]}$.
\end{Thm}
\begin{proof}
	Applying Proposition~\ref{prp:o2bexp}, linearity of expectations, the fact that the regret over 1-bounded loss functions is $m$-bounded, and Proposition~\ref{prp:o2b} yields
	\begin{align*}
	\E_{t\sim\mathcal U[T]}\E_{\mathcal P\sim\mathcal Q}\E_{\mathcal P^m}\E_{\ell\sim\mathcal P}\ell(\bar\theta)
	\le\E_{\mathcal P\sim\mathcal Q}\left(\E_{\ell\sim\mathcal P}\ell(\theta^\ast)+\frac{\R_m(s_t)}m\right)
	&\le\E_{\mathcal P\sim\mathcal Q}\E_{\ell\sim\mathcal P}\ell(\theta^\ast)+\frac1T\sum_{t=1}^T\E_{\mathcal P\sim\mathcal Q}\left(\frac{\R_m(s_t)}m\right)\\
	&=\E_{\mathcal P\sim\mathcal Q}\E_{\ell\sim\mathcal P}\ell(\theta^\ast)+\frac2T\sum_{t=1}^T\E_{\mathcal P\sim\mathcal Q}\left(\frac{\R_m(s_t)}{2m}+\frac12\right)-1\\
	&\le\E_{\mathcal P\sim\mathcal Q}\E_{\ell\sim\mathcal P}\ell(\theta^\ast)+\frac\TAR m+\sqrt{\frac8T\log\frac1\delta}
	\end{align*}
\end{proof}


\newpage
\section{Computing the Quadratic Growth Factor}\label{app:growth}

For our analysis of the FLI variants of Algorithm~\ref{alg:fmrl} we consider a class of functions related to strongly convex functions that satisfy the quadratic growth (QG) condition:
\begin{equation}\label{eq:qg}
\frac\alpha2\|\theta-\theta^\ast\|^2\le f(\theta)-f(\theta^\ast)
\end{equation}
By Theorem~2 of \citet{karimi:16}, in the convex case QG is equivalent, up to multiplicative constants, with the Polyak-\L ojaciewicz (PL) inequality \cite{polyak:63}.
Using the latter condition, \citet{karimi:16} further show that functions of form $f(A\theta)$ for $f$ strongly-convex satisfy the PL inequality, and thus also QG, with constant $\alpha=\Omega(\sigma_{\min}(A))$.
This provides data-dependent guarantees for a variety of practical problems, including least-squares and logistic regression.
\citet{garber:19} shows a similar result for expectations of such functions with the QG constant depending now on $\lambda_{\min}(\E A^TA)$;
in order to do so they assume the constraint set is a polytope, e.g. an $\ell_1$ or $\ell_\infty$ ball.

For our results we require a stronger condition, namely that if $L$ is a sum of $m$ convex losses then $L$ satisfies $\alpha m$-QG.
While this additive property holds directly if the losses are strongly-convex, in the general case it does not.
Furthermore, the spectral lower bound on $\alpha$ studied by \citet{karimi:16} and \citet{garber:19} is an underestimate;
for example, in the strongly-convex case, where $A^TA$ is the identity, the lower bound will be 1 even though their sum is $m$-QG.

Here we derive an alternative approach for verifying $\alpha$-QG for a convex Lipschitz function $f$ constrained to a ball of radius $B$.
Note that since the functions are Lipschitz, we can focus on computing the minimal difference between $f(\theta)$ and $f(\theta^\ast)$ over all $\theta$ located some fixed distance $\delta$ away from any minimizer $\theta^\ast$ of $f$ over the ball:
\begin{align*}
\varepsilon_\delta=\min\quad&f(\theta)-f(\theta^\ast)\\
\textrm{s.t.}\quad&\|\theta-\theta^\ast\|_2^2\ge\delta^2\\
&\|\theta\|_2\le B
\end{align*}
Then if $f$ is $\alpha$-QG, Equation~\ref{eq:qg} implies that $\alpha_\delta=\frac{2\varepsilon_\delta}{\delta^2}$ should be a constant, or equivalently that $\varepsilon_\delta=\Omega(\delta^2)$.
While the above problem is non-convex due to the first constraint, note that 
$$\delta^2
\le\|\theta-\theta^\ast\|_2^2
=\|\theta\|_2^2-2\langle\theta,\theta^\ast\rangle+\|\theta^\ast\|_2^2
\le B^2-2\langle\theta,\theta^\ast\rangle+\|\theta^\ast\|_2^2$$
which is a linear constraint since $\theta^\ast$ is constant.
Therefore we have
\begin{align*}
\varepsilon_\delta\ge\min\quad&f(\theta)-f(\theta^\ast)\\
\textrm{s.t.}\quad&2\langle\theta^\ast,\theta\rangle\le B^2-\delta^2+\|\theta^\ast\|_2^2\\
&\|\theta\|_2\le B
\end{align*}
which is a convex program amenable to standard solvers;
we employ the Frank-Wolfe method \cite{frank:56}.

\newpage

\section{Experimental Details}

\subsection{Constructing Mini-Wikipedia}\label{app:miniwiki}

We briefly describe the construction of Mini-Wiki.
Starting with the raw corpus of the Wiki3029 dataset of \citet{arora:19}, we select those Wikipedia pages whose titles correspond to lemmas in the WordNet corpus \cite{fellbaum:98}.
We then use the hypernymy structure in this corpus to separate the pages into four semantically meaningful meta-classes;
this is necessary when using linear classification as the task similarity only depends on the classifier and not the representation.
Finally, we take the longest sentences from each page to construct $m$-shot tasks of $4m$ samples each, for $m=1,2,4,\dots,32$.
We have made MiniWiki available here: \url{https://github.com/mkhodak/FMRL/blob/master/data/miniwikipedia.tar.gz}.

\subsection{Complete Deep Learning Results}

Below are plots for all evaluations on Omniglot and Mini-ImageNet.
As our algorithm generalizes the Reptile method of \citet{nichol:18}, we use code they make available at \url{https://github.com/openai/supervised-reptile} and vary the parameters \texttt{train-shots} and \texttt{inner-iters}.

\begin{figure}[h!]
	\centering
	\includegraphics[width=0.245\linewidth]{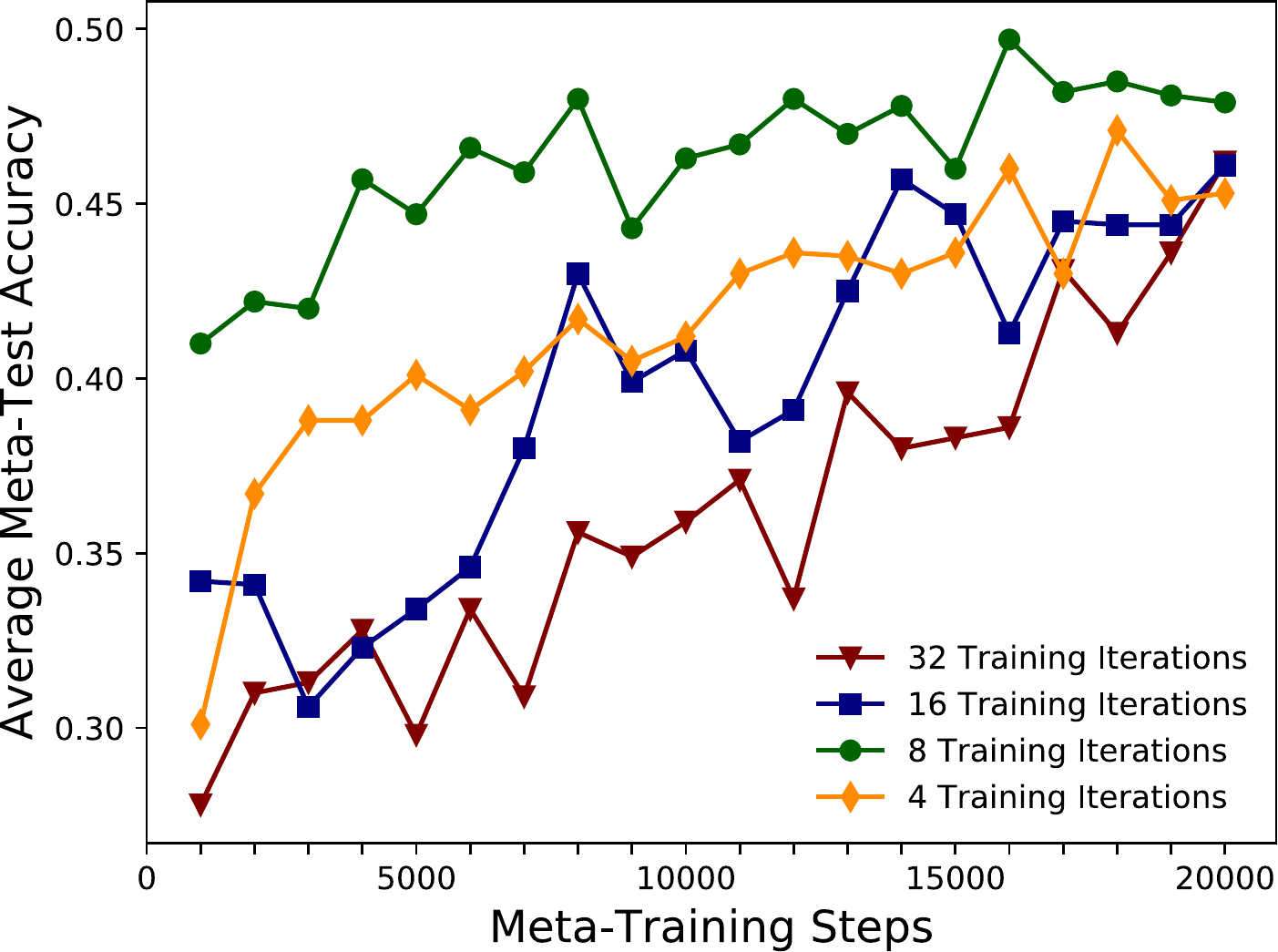}
	\includegraphics[width=0.245\linewidth]{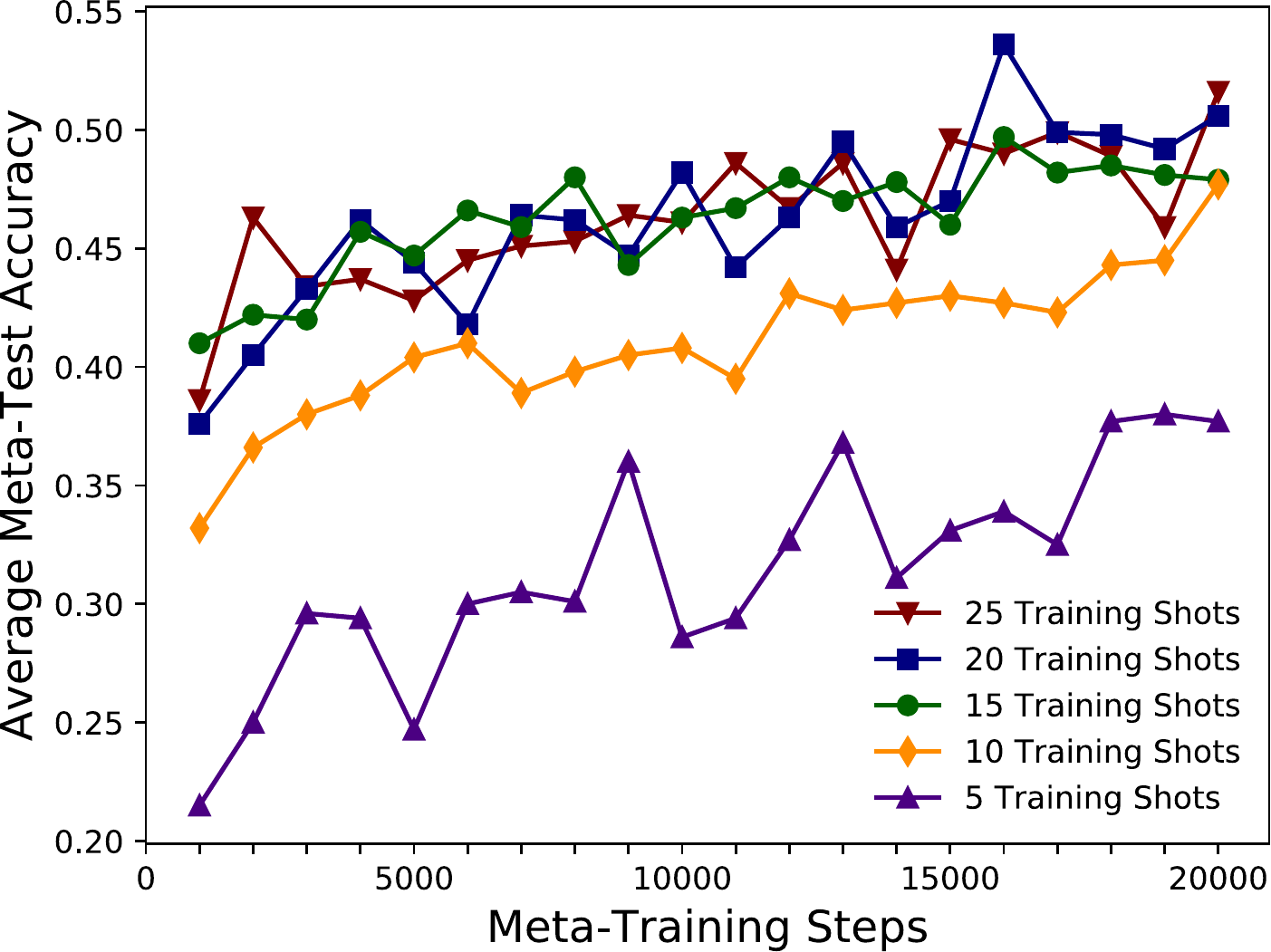}
	\includegraphics[width=0.245\linewidth]{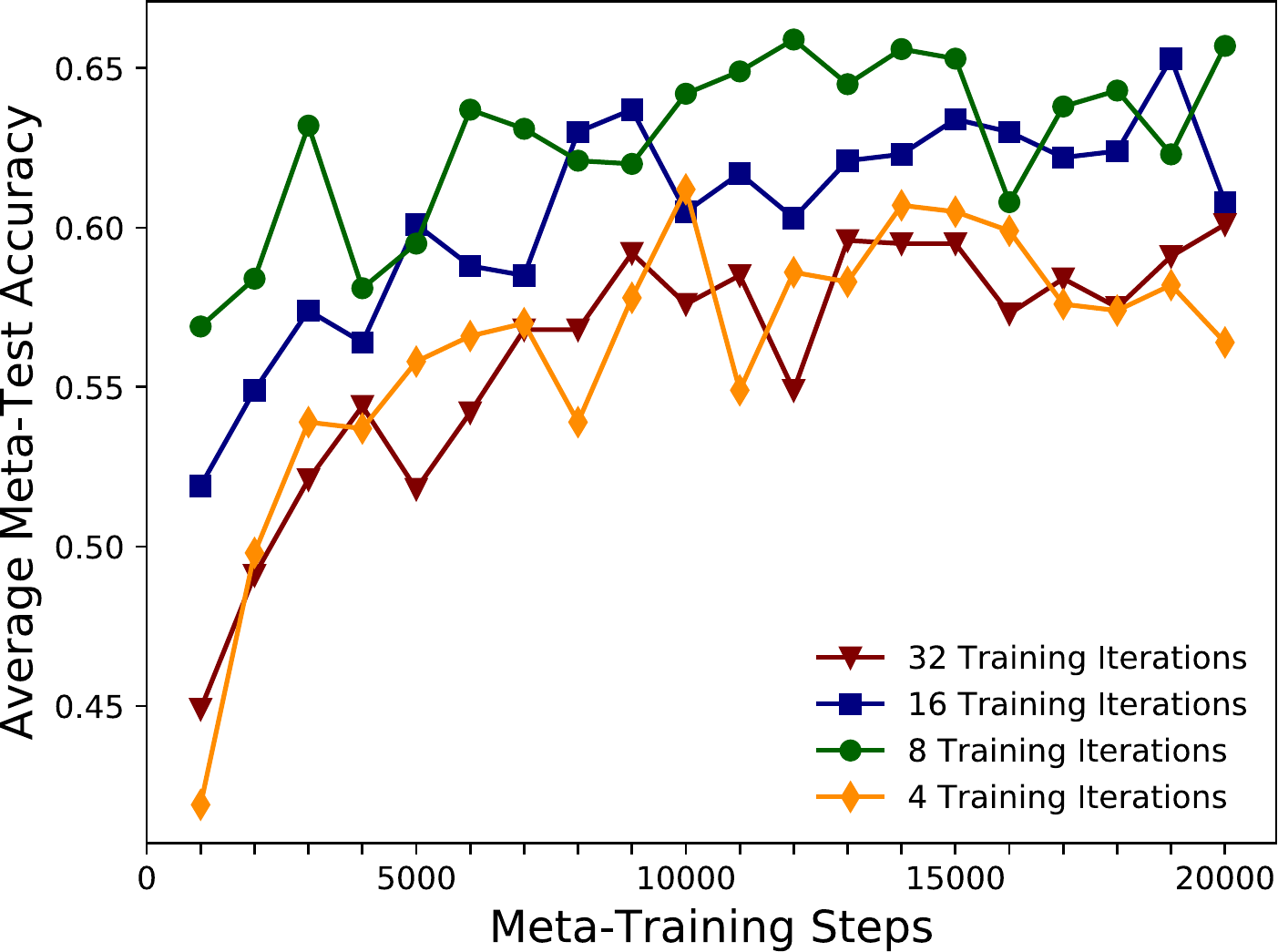}
	\includegraphics[width=0.245\linewidth]{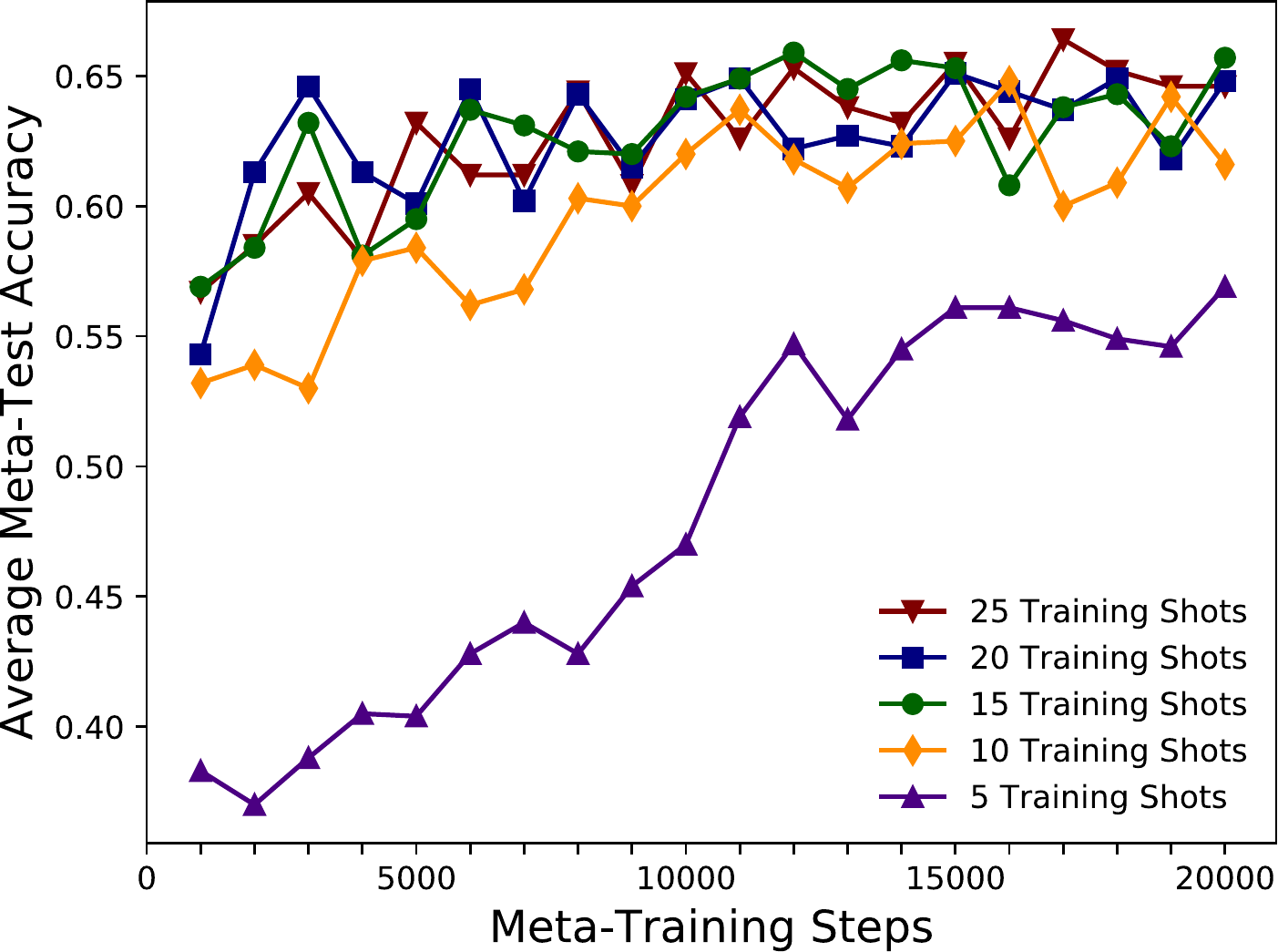}
	\caption{\label{fig:mini5plots}
		Performance of the FLI variant of \Eph with OGD within-task (Reptile) on 5-way Mini-ImageNet when varying the number of task samples and the number of iterations per training task.
		In the left-hand plots we use 1-shot at meta-test time; in the right-hand plots we use 5-shots.
		50 iterations are used at meta-test time in both cases.
	}
\end{figure}

\begin{figure}[h!]
	\centering
	\includegraphics[width=0.245\linewidth]{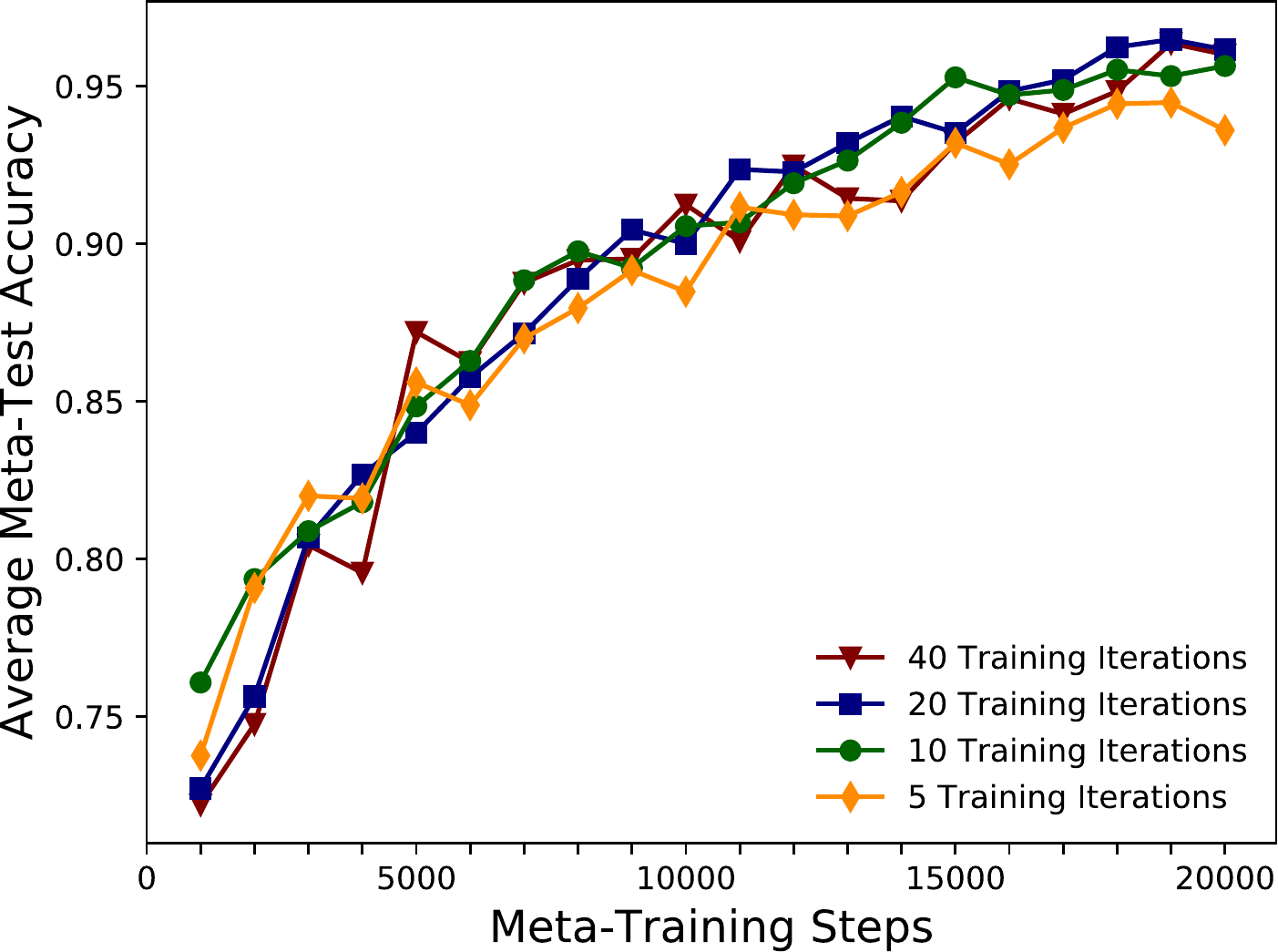}
	\includegraphics[width=0.245\linewidth]{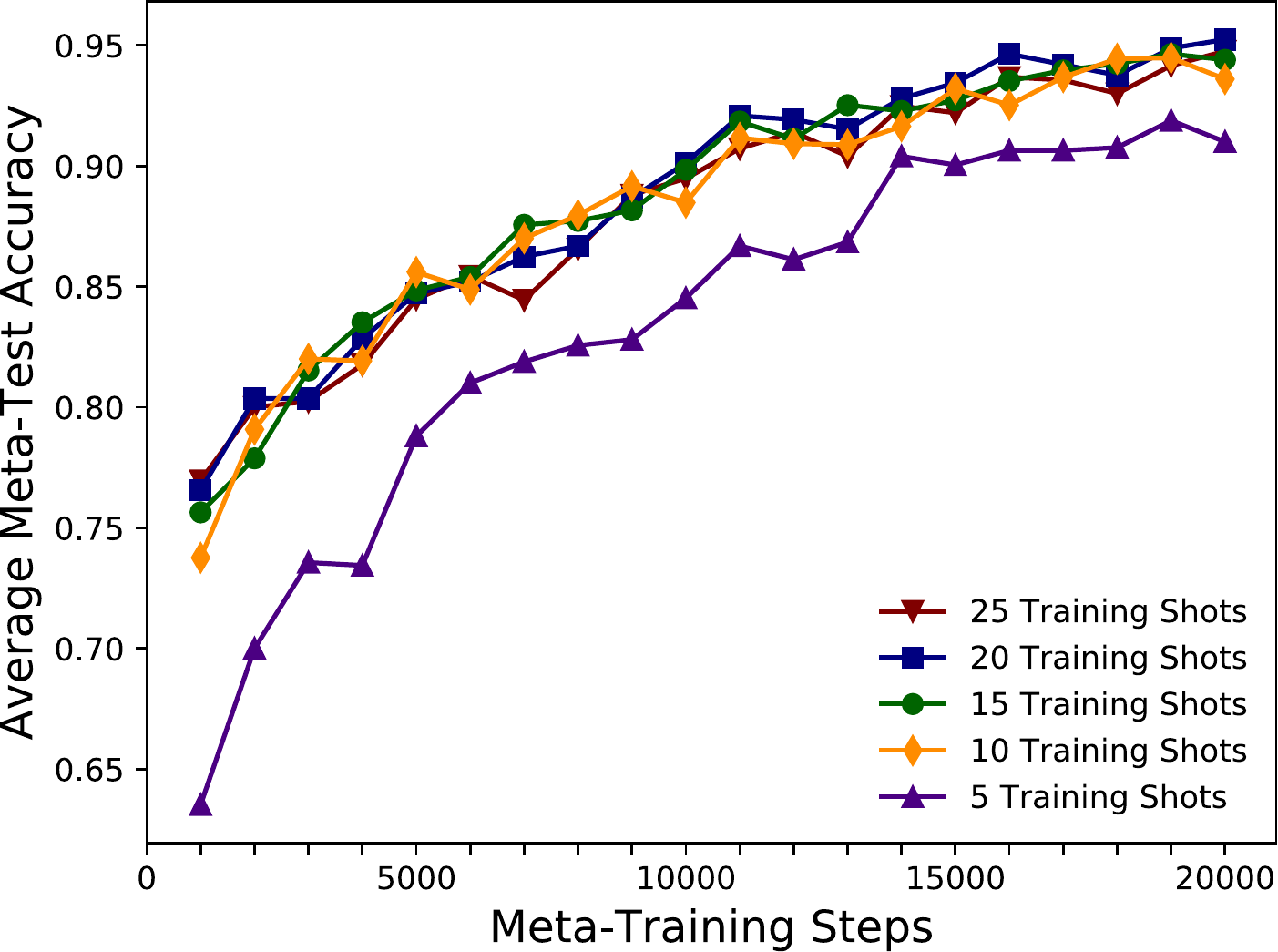}
	\includegraphics[width=0.245\linewidth]{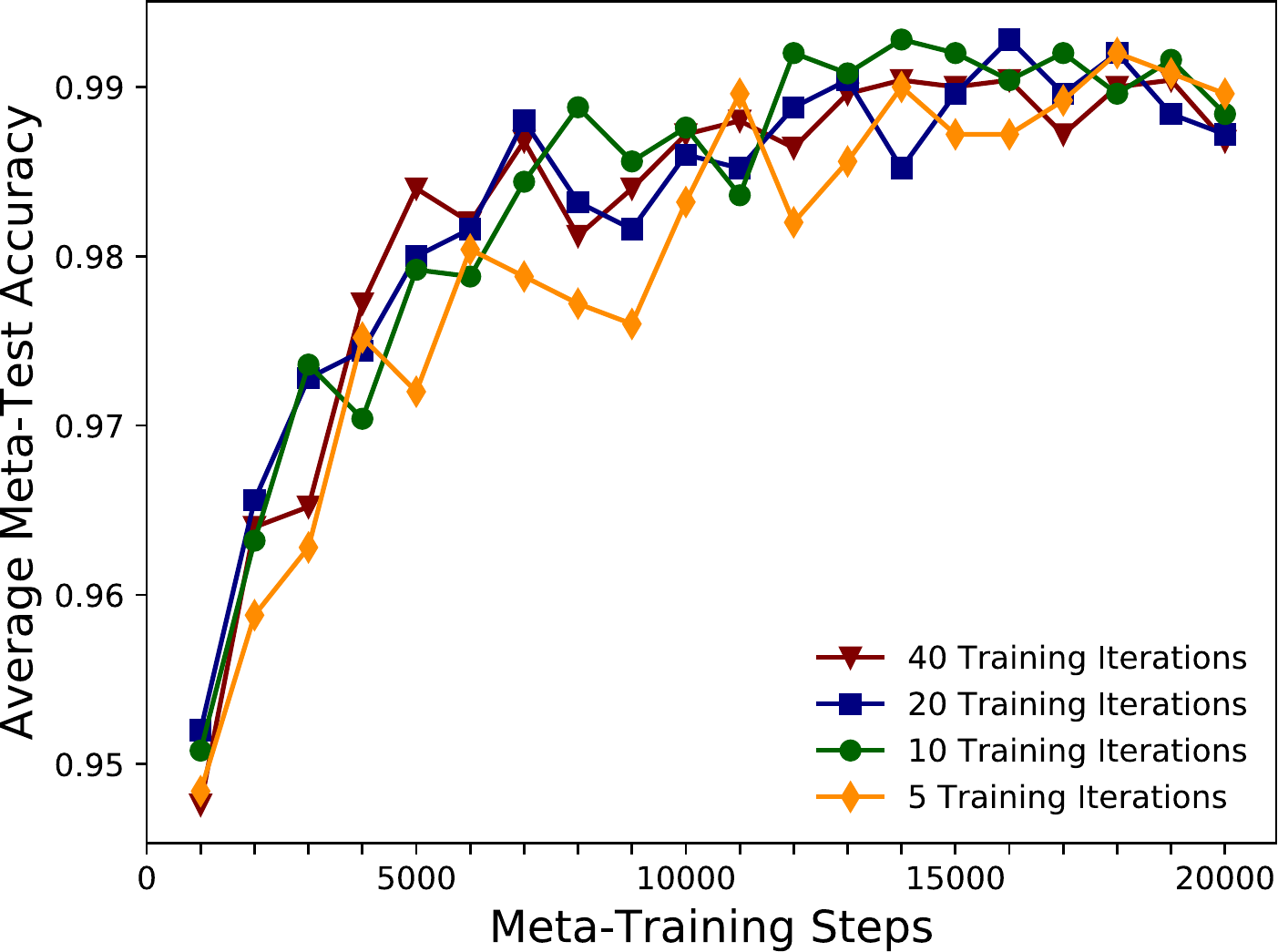}
	\includegraphics[width=0.245\linewidth]{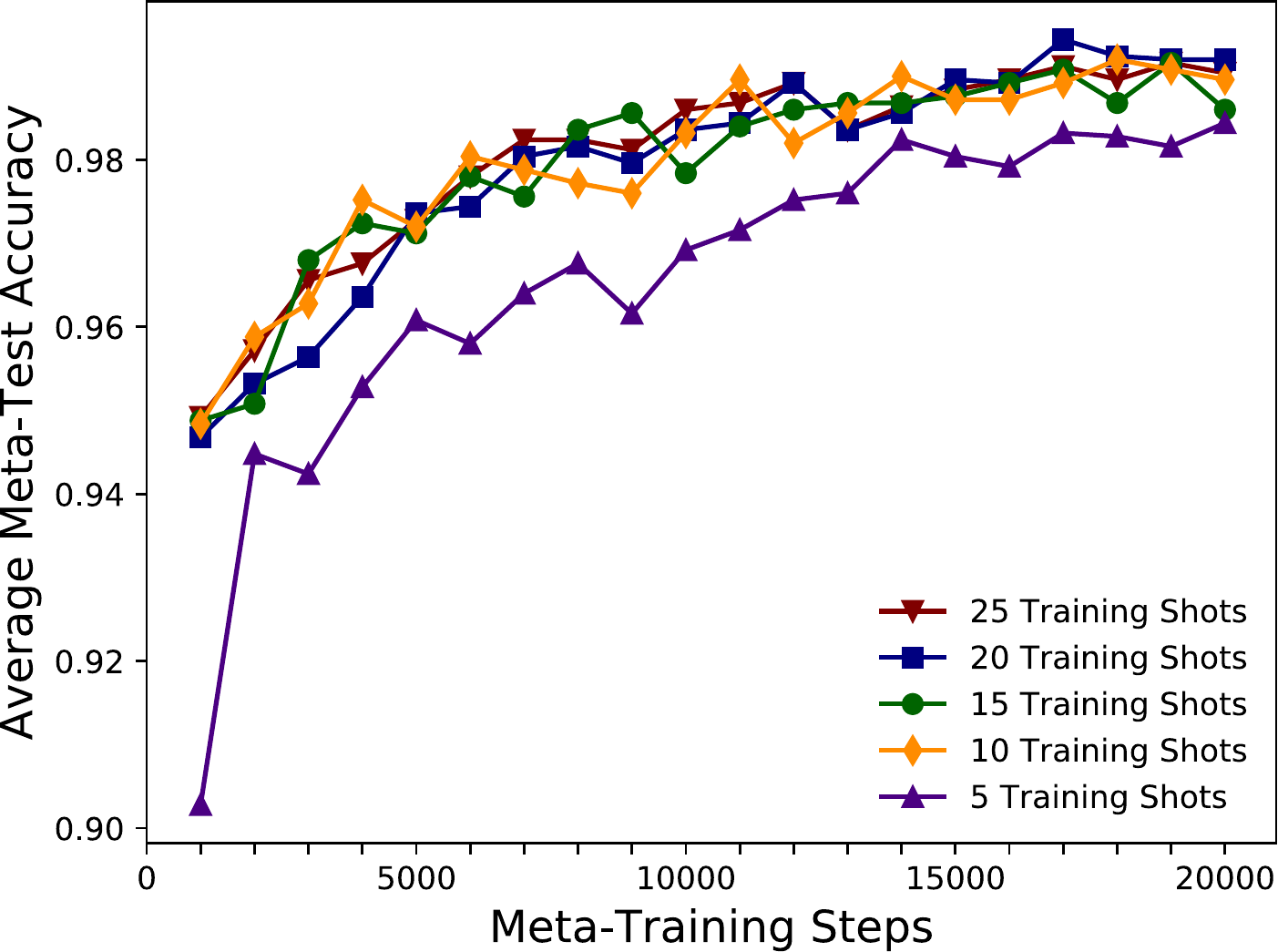}
	\caption{\label{fig:omni5plots}
		Performance of the FLI variant of \Eph with OGD within-task (Reptile) on 5-way Omniglot when varying the number of task samples and the number of iterations per training task.
		In the left-hand plots we use 1-shot at meta-test time; in the right-hand plots we use 5-shots.
		50 iterations are used at meta-test time in both cases.
	}
\end{figure}

\begin{figure}[h!]
	\centering
	\includegraphics[width=0.245\linewidth]{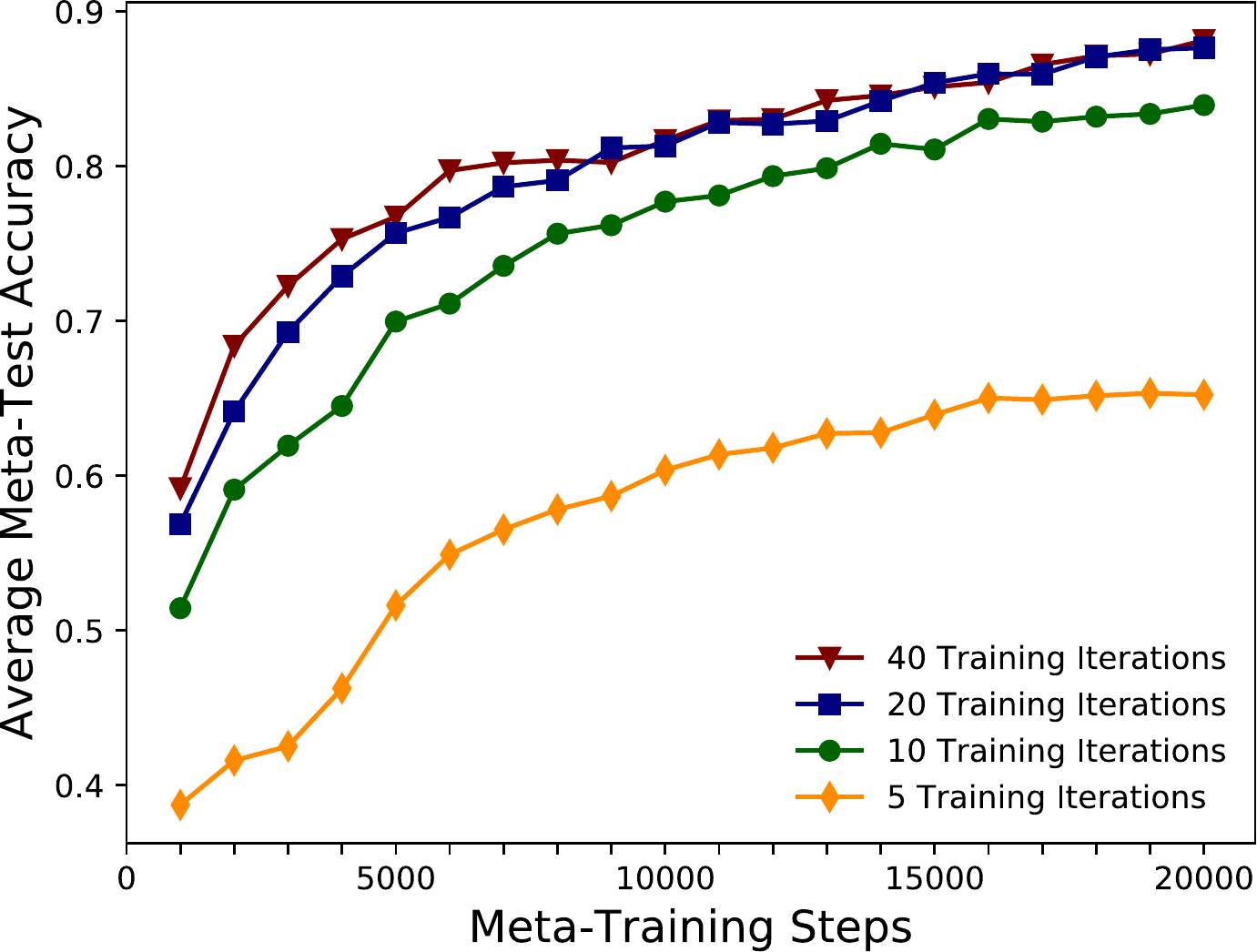}
	\includegraphics[width=0.245\linewidth]{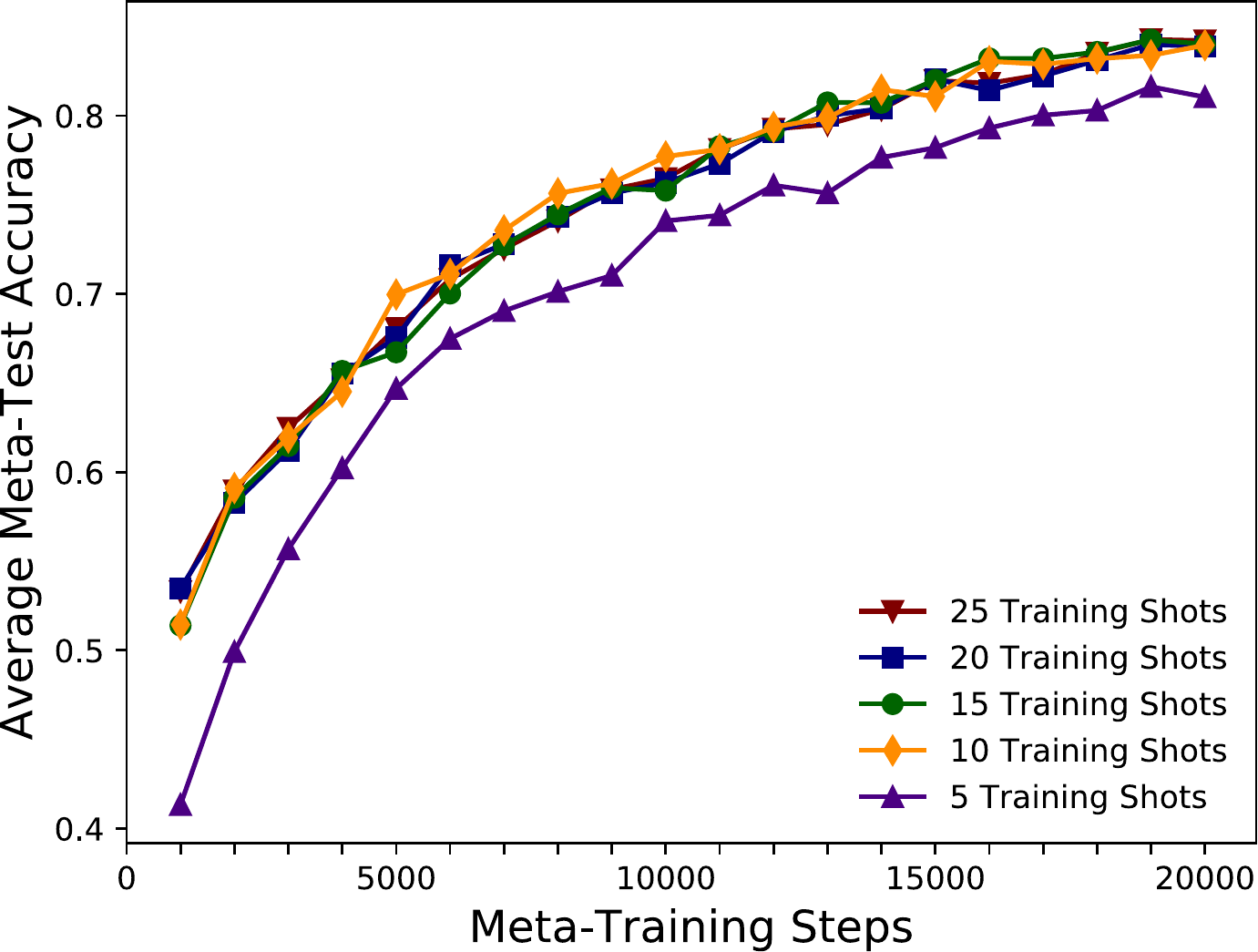}
	\includegraphics[width=0.245\linewidth]{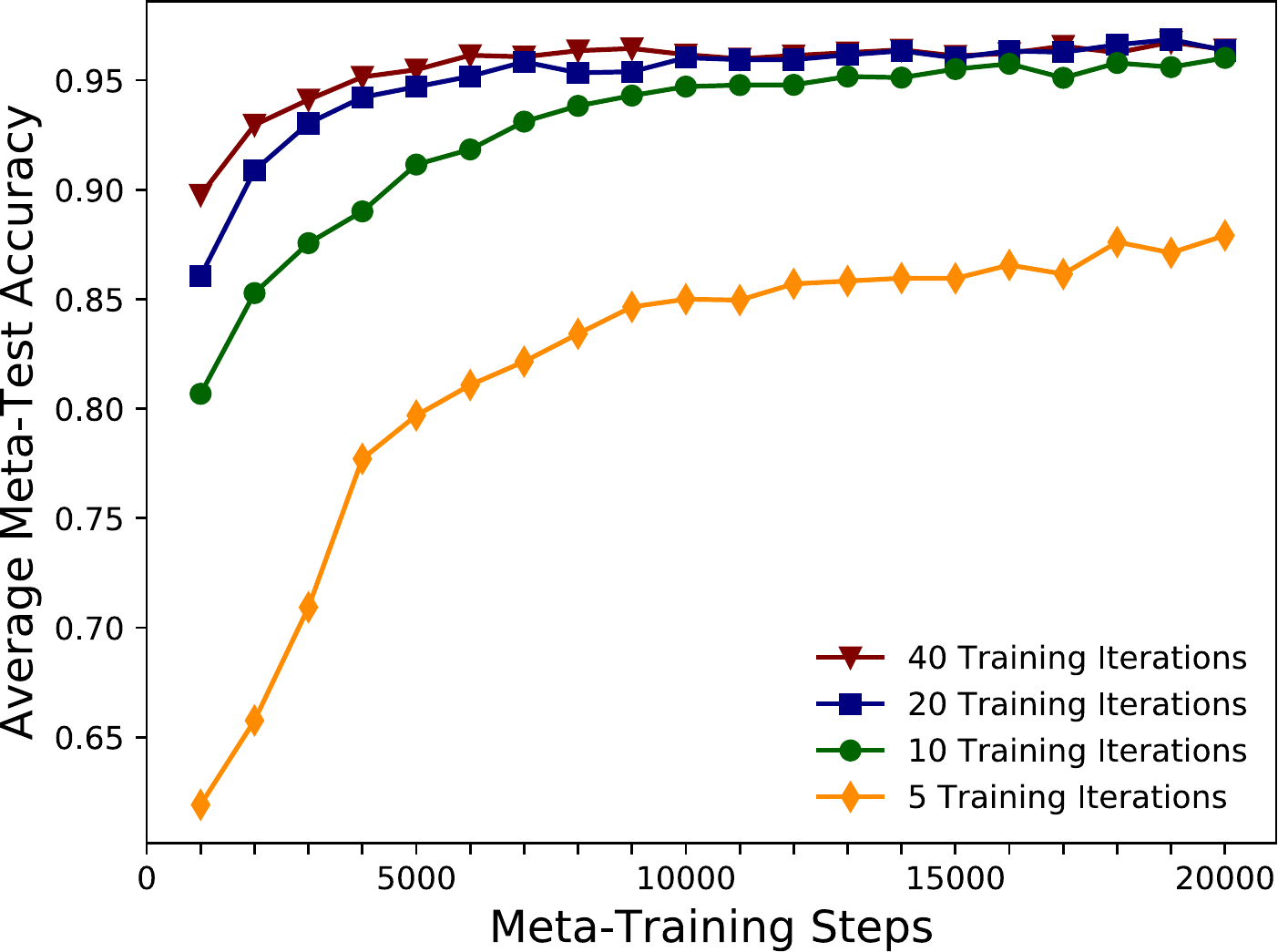}
	\includegraphics[width=0.245\linewidth]{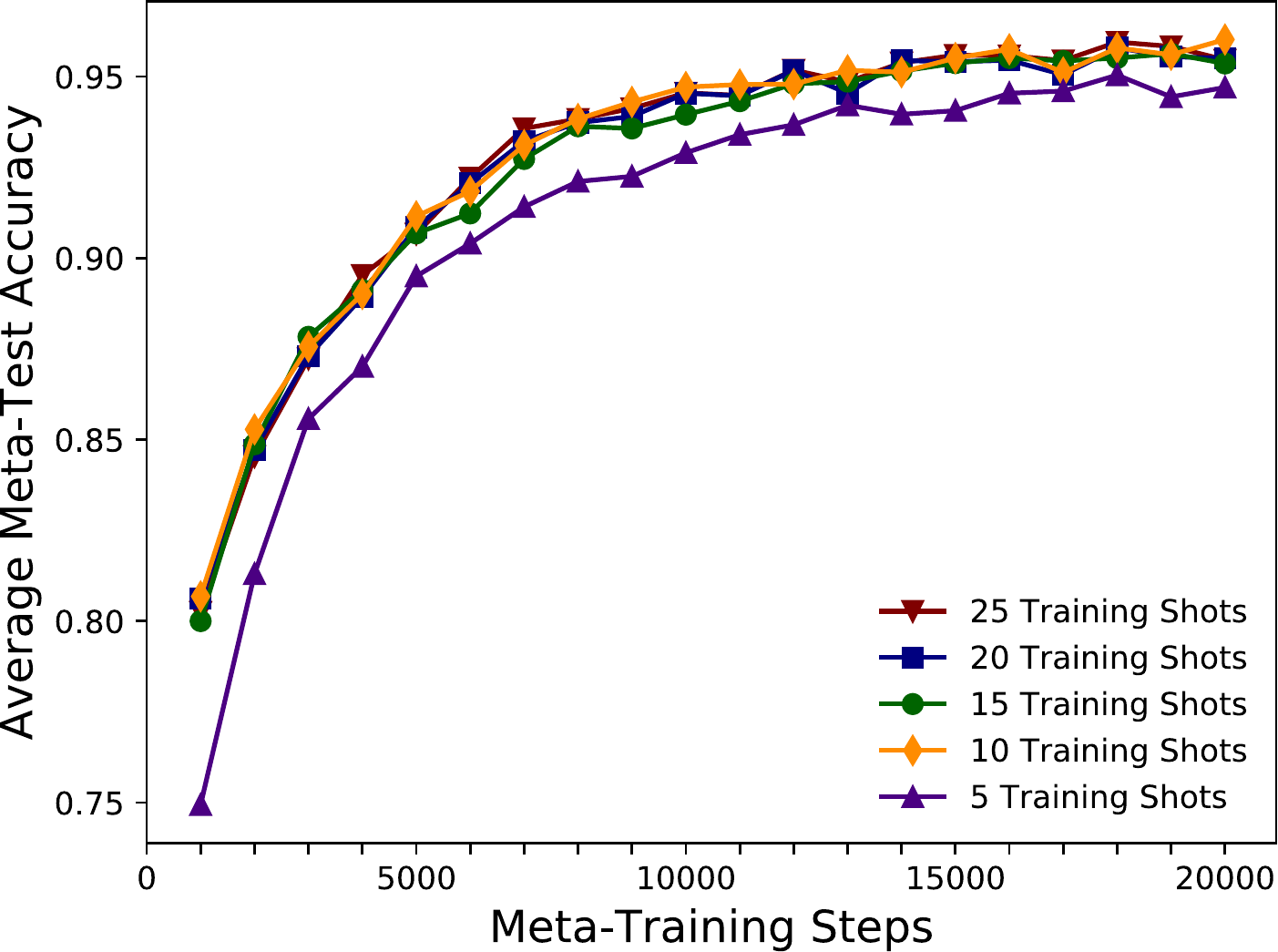}
	\caption{\label{fig:omni20plots}
		Performance of the FLI variant of \Eph with OGD within-task (Reptile) on 20-way Omniglot when varying the number of task samples and the number of iterations per training task.
		In the left-hand plots we use 1-shot at meta-test time; in the right-hand plots we use 5-shots.
		50 iterations are used at meta-test time in both cases.
	}
\end{figure}

\end{document}